\documentclass[accepted]{uai2026} 
                        
\usepackage[american]{babel}

\usepackage{natbib} 
\usepackage{mathtools} 
\usepackage{booktabs} 
\usepackage{tikz} 
\usepackage{abstract}
\usepackage{microtype}
\usepackage{subcaption}
\usepackage{booktabs} 
\usepackage{graphicx} 
\usepackage{titlesec}
\usepackage{mathtools}
\usepackage{adjustbox}
\usepackage{derivative}
\usepackage{multirow}
\usepackage{amsmath, amssymb, amsfonts, dsfont, amsthm}
\usepackage{extpfeil}
\usepackage{float}
\PassOptionsToPackage{hyphens}{url}\usepackage{hyperref}

\usepackage[dont-mess-around]{fnpct}
\usepackage{balance} 

\usepackage{algorithm, algorithmic}

\theoremstyle{plain}
\newtheorem{theorem}{Theorem}[section]
\newtheorem{proposition}[theorem]{Proposition}
\newtheorem{lemma}[theorem]{Lemma}

\theoremstyle{definition}
\newtheorem{definition}[theorem]{Definition}

\theoremstyle{remark}
\newtheorem{remark}[theorem]{Remark}

\title{Optimal Model Activation Policies for Inference Networks \\ of Large Language Models}

\author[1]{\href{mailto:<phoebuschar@aueb.gr>?Subject=On Optimal Model Activation Policies for Inference Networks of Large Language Models}{Foivos Charalampakos}{}}
\author[2]{Md Ibrahim Ibne Alam}
\author[1]{Iordanis Koutsopoulos}
\author[3]{Koushik Kar}
\affil[1]{%
    Department of Informatics\\
    Athens University of Economics and Business
}
\affil[2]{%
    Department of Electrical and Computer Engineering\\
    Yale University
}
\affil[3]{%
    Department of Electrical, Computer, and Systems Engineering\\
    Rensselaer Polytechnic Institute
  }
  
\begin{document}
\maketitle

\begin{abstract}
  Recent advances in large language models (LLMs) have rendered them necessary for Natural Language Processing (NLP) tasks, and their high inference cost motivates the study of cost–performance trade-offs to carry out these tasks. In practice, several expert LLMs are used in synergy for inference, either in an ensemble mode or in series, yet without a principled approach on how to best use the available models. An adaptive approach can route simple queries to cheaper, less reliable LLMs and complex ones to more capable, costly models. However, because this strategy focuses on query difficulty rather than model expertise, a clear understanding on how to best leverage available expert models is missing.
  We introduce and lay the groundwork for \emph{inference networks}, a graph-based framework, where nodes denote alternative LLMs, and links denote conditional model activations. The inference network design problem is to determine the best topology, namely the best way to use (some or all of) the models that best addresses the cost-performance trade-off. We start from the basic topology of a series of LLM experts, each of which has a different cost and a different {level of expertise}, which is captured via model confidence.
  We formulate the problem of optimal activation of these models so as to minimize the expected inference cost subject to a target performance constraint. For this special class of inference networks, we prove that the optimal activation policy has a threshold structure: query the lowest-cost LLM first, and invoke the more expensive LLM only if the confidence falls below a defined threshold.
  For discriminative tasks (e.g., classification), the optimal policy consists of a set of thresholds, one threshold for each class, while for generative tasks (e.g., question-answering), it consists of a single threshold. We provide a structured method to compute the thresholds (i.e., via a low-dimensional search or a closed-form solution when applicable), and practical confidence estimation mechanisms for both task types.
  Experiments with off-the-shelf open-source LLMs on text classification and generation benchmarks show substantial cost reductions while meeting the specified performance budget.
\end{abstract}

\section{Introduction}

Recent developments in generative language modeling, particularly Transformer-based auto-regressive large language models (LLMs), have significantly enhanced performance across complex natural language processing (NLP) tasks. The exceptional (often zero-shot) capabilities of these foundation models have led practitioners and researchers to increasingly rely on them for building applications. However, given the high computational burden required to train and use such models (typically consisting of tens of billions of parameters), they are often faced with the decision of which LLM service to utilize, given their available budget, quality expectations and significant cost (which may capture monetary cost, compute cost, or energy consumption). 
Consequently, this need has led to the development of numerous approaches with different goals. Methods like model compression and pruning techniques \citep{lora, puzzle_nas} aim to reduce the cost of using such big models while on the other hand, model ensembles \citep{arora, llm-blender} exploit the expressive power of multiple models (with an increased cost).  \newline
There are also methods in between, that aim to select the most suitable model, while trying to mitigate the computational demands. Examples of this class of methods are the alternative formulations of auto-regressive decoding mechanisms \citep{decoding_leviathan} that use inexpensive LLMs for most generated tokens and switch to costly ones when necessary (e.g., the cheap model’s estimates are not confident enough or poorly aligned to the target’s distribution). Another approach that shows great potential for balancing cost and output quality is \textit{adaptive inference}, which dynamically adjusts the computational resources spent for a task based on the task complexity. By selectively switching between different models, adaptive inference methods can significantly lower the inference cost (e.g., computation FLOPs or API money) without severely compromising the accuracy of the results.
In this work, we move beyond LLM inference schemes that only rely on query characteristics, and instead introduce and study inference networks, a framework of LLM experts that prioritizes model-specific expertise and theoretical model activations for various NLP tasks. We define inference networks as directed graphs where the nodes are available LLMs, and a link $(i, j)$ exists when model $j$ is activated, possibly based on the result of model $i$ which is already executed. 
The inference network design problem refers to the problem of finding the best topology, namely the best way to use (some or all of) the models that best addresses the cost-performance trade-off.
Special cases of these inference networks are the ensemble topology (parallel nodes), where the available models are used, and their output is aggregated, and the chain topology (serial nodes).
In this work, we fix attention to a simple, yet fundamental case of inference networks, that of models in series. 
Thus, for two LLMs, the lower-cost model is activated first, and its output confidence determines whether a higher-cost model should be invoked, thereby grounding the routing decision in model-specific expertise.
\newline
A critical component of this system 
is the rule for deferring the input queries to the larger model(s) (i.e., model activation rule). Designing such rules for LLMs that perform generative tasks is not trivial and introduces challenges \citep{gupta_uncertainty}. Prior works like FrugalGPT \citep{frugal} or Hybrid LLM \citep{hybridlllm}, provide empirical-only strategies that investigate how to dispatch queries across LLMs, using learned auxiliary routing models. 
In contrast, we focus {on deriving activation policies with theoretically characterized optimal structure based on LLMs' self-confidence}, without utilizing an external routing model that would require additional training and data collection. In particular, under (mild) assumptions, we prove that the optimal activation policy is threshold-based: we first use the lowest-cost LLM, evaluate the confidence of its response, and escalate to a higher-cost LLM only when that confidence is below a certain threshold.
For the discriminative tasks such as classification ones, the optimal policy can be expressed as a set of class-specific thresholds. 
Additionally, we empirically demonstrate the superiority of these policies against other state of the art strategies \citep{frugal, hybridlllm}
in real-world scenarios, using off-the-shelf open-source LLMs on various text classification and text generation datasets. We manage to achieve significant cost reduction (up to 97\% in some instances) while performance stays within a defined budget constraint. This showcases that our policies can be used in practical scenarios, delivering a better cost-quality trade-off.

\section{Background}
\label{sec:background}
\paragraph{\textbf{LLM inference.}} We focus on answering natural language queries. A range of practical NLP generative or discriminative tasks, such as question answering or text classification, can be naturally expressed as query-answer tuples. To generate responses for input queries, we rely on a
$\text{LLM}(\cdot): \mathcal{Q} \to \mathcal{A}$ that, given a query $\boldsymbol{q} \in \mathcal{Q}$, produces an answer $\boldsymbol{\hat{a}} \in \mathcal{A}$. 
Most recent advances in the field have also involved approaches like Reinforcement Learning-based training using rewards from human feedback \citep{rlhf,dpo}, providing state-of-the-art results and expertise on a variety of tasks beyond simple word completion, like conversation, code generation, instruction following, and commonsense reasoning \citep{gemma3, llama3, toolformer, deepseek, gpt4, qwen3, mistral, gpt3}. 
During inference, given any query $\boldsymbol{q}$, LLMs can handle each task as text generation. For discriminative tasks, $\mathcal{A}$ can be a set of predefined classes $\mathcal{A} = \{a_i\}_{i\in [\mathcal{A}]}$ so one can restrict the model to output only tokens relevant to this set. In tasks that require explicit text generation, the (predicted) posterior distribution over the whole vocabulary 
\citep{DBLP:conf/iclr/HoltzmanBDFC20, fan-etal-2018-hierarchical, ficler-goldberg-2017-controlling} can be sampled to produce a  response.
\paragraph{\textbf{Confidence estimation.}} In addition to generating responses, we also want to estimate how confident the LLM is in its outputs. Confidence estimation aims to provide a score that reflects the reliability of a model's answer, and in practice, calculating this score can be challenging. This challenge can be addressed either using self-assessed, 
entropy-based metrics tied to uncertainty estimation \citep{lm-polygraph, gupta_uncertainty, DBLP:conf/iclr/KuhnGF23}, or through alternative approaches involving external scoring models \citep{frugal}. In this work, we assess the model's internal certainty through a simple function based on its predicted probability distributions, yielding a confidence score in the range $[0,1]$. The type of task (discriminative or generative) dictates the function's definition: generative tasks require a score for the entire text sequence, while discriminative tasks only require a score for the predicted class.
We leave more sophisticated confidence estimation approaches \citep{Shrivastava, Kadavath, lm-polygraph} for future work.
\section{System Model}
The core idea is to perform LLM inference through a serial network of models, each progressively more complex (i.e., with a larger number of parameters and higher inference cost), applied sequentially. The serial network is a directed graph ${G} = (V, E)$, where $V = {1, \dots, N}$ and $E \subseteq V \times V$, such that for every edge $(i,j) \in E$, $i < j$. This directed graph represents a serial chain, meaning each vertex has exactly one outgoing edge to a larger-indexed vertex, ensuring a linear structure.
In this work, we focus on the case of two models ($|V|=2$) where the lower-cost model first processes each query, and a policy determines whether its output should be used; if not, the input query is escalated to the second model.
We use the confidence of the first LLM to decide whether to pass a query to the next node. The confidence reflects the model's ability to accurately handle a given query, and it is derived from its generated answer. This approach shifts the focus from estimating query difficulty (something that's often hard to define) using external scorers \citep{frugal, hybridlllm} to using confidence as a readily computable indicator of model performance. Confidence is influenced by both the difficulty of the query and the expertise of the model, which is defined by how well its scope (e.g., domain of training data) correlates with the query. Between two models of similar expertise, the one with higher confidence becomes the more reliable choice for handling the query. \\
Consider two models: a `student' model (sLLM) which incurs less cost to generate an answer per query and a costlier `master' model (mLLM). 
For a given query with true answer $\boldsymbol{a}$, sLLM is used first, and a deferral policy $\pi$ determines whether to accept its output answer $\boldsymbol{\hat{a}}^s$ or defer the input query to mLLM to obtain answer $\boldsymbol{\hat{a}}^m$. 
We condition the deferral policy on the confidence sLLM outputs for its answer, a scalar in $[0,1]$, denoted by $\beta$. 
Thus, we use a measurable threshold policy $\pi : [0,1] \rightarrow \{0,1\}$, parameterized by $\theta$:
\begin{equation}
    \pi(\beta, \theta) = \begin{cases}
0, & \beta \geq \theta\, \text{(accept sLLM at $\beta$)}\\
1, & \beta < \theta\,\text{(defer to mLLM at $\beta$)}
\end{cases}
\end{equation}
where 
$\beta = \text{conf}(\boldsymbol{\hat{a}}^s)$
and $\text{conf}(\cdot): \mathcal{A} \to [0,1]$ is the function that measures $\text{sLLM}$'s confidence regarding the generated answer.
We aim to characterize the optimal policy (i.e., the optimal threshold) 
which emerges out of the optimization problem of minimizing the expected inference
cost subject to a target performance constraint. 
Assuming the distributions of variables such as $\beta$ are known, we can solve this problem either analytically or numerically. Specifically, if we have access to the joint probability distribution
$f(\boldsymbol{a}, \boldsymbol{\hat{a}}^s, \boldsymbol{\hat{a}}^m, \beta)$, we can derive closed-form solutions for the optimization problem. In practical settings, where exact knowledge of the distribution may be unavailable, we provide empirical solutions using a finite sample of data.

\subsection{Generative tasks - Single-threshold policy}
\label{subsec:single-threshold-policy}

First, we focus on generative tasks (e.g., question answering tasks). The output space $\mathcal{A}$ is an open-ended set of sequences (e.g., words, sentences, or structured responses). 
For a given input query $\boldsymbol{q}$, sLLM generates the text sequence $\text{sLLM}(\boldsymbol{q}) = \boldsymbol{\hat{a}}^s \in \mathcal{A}$. Moreover, the associated confidence score $\beta \in [0, 1]$ reflects its certainty about the whole generated text sequence. 
According to the adopted policy, the prediction from sLLM is only accepted when 
$\beta \geq \theta$.
Regarding the cost of processing $\boldsymbol{q}$, we always `pay' $c_s$ and `pay' $c_m (> c_s)$ only when we defer to mLLM, i.e., when $\beta < \theta$. Hence, we can define the cost for the {$i$-th query} as:
\begin{equation}
    c(\beta_i, \theta) = c_s + \pi(\beta_i, \theta) \cdot c_m = c_s + \mathds{1}[\beta_i < \theta] \cdot c_m\,. 
\end{equation}
We can compute the \textit{expected} cost, which depends only on how often $\beta$ falls below the threshold, as:
\begin{equation}
\begin{aligned}
    \text{Cost}_\text{ST}(\theta) &= c_s + c_m \mathbb{E}_\beta\big[\mathds{1}[\beta<\theta]\big]\\ &= c_s + c_m \Pr(\beta < \theta) = c_s + c_m \int_0^{\theta} f(\beta)\,d\beta\,.
\end{aligned}
\end{equation}
Correspondingly, we can define the error formulations based on two separate settings. In the first one, we assume a `teacher-student' paradigm where mLLM is considered as an expert model, an oracle (the `teacher') so error can be only incurred from the small model when $\beta \geq \theta$. In the second setting, both models are allowed to make mistakes. \newline
Specifically, beginning with the oracle setting (where we have that $\boldsymbol{{a}} = \boldsymbol{\hat{a}}^m$), the query-based error becomes: 
\begin{equation}
\begin{aligned}
    e(\boldsymbol{\hat{a}}^m_i, \boldsymbol{\hat{a}}^s_i, \beta_i, \theta) &= \left(1-\pi(\beta_i, \theta )\right)\mathds{1}[\boldsymbol{\hat{a}}^m_i \neq \boldsymbol{\hat{a}}^s_i] \\ &= \mathds{1}[\beta_i \geq \theta]\mathds{1}[\boldsymbol{\hat{a}}^m_i \neq \boldsymbol{\hat{a}}^s_i] \,,
\end{aligned}
\end{equation}
where the binary error indicator $\mathds{1}[\boldsymbol{\hat{a}}^m_i \neq \boldsymbol{\hat{a}}^s_i]$ can be defined accordingly for generation tasks (e.g., by using cosine similarity between the embeddings of the true and predicted answers and a threshold $\delta$ to indicate a mismatch when the similarity score $\cos(\mathrm{emb}(\boldsymbol{\hat{a}}^m_i), \mathrm{emb}(\boldsymbol{\hat{a}}^s_i))$ falls below $\delta$). We can also compute the \textit{expected} error\footnote{`ST-O': Single-Threshold Oracle} as:
\begin{multline}
\label{eq:exp_error_1}
    \text{Error}_\text{ST-O}(\theta) = \mathbb{E}_{\boldsymbol{a}, \boldsymbol{\hat{a}}^m, \boldsymbol{\hat{a}}^s,\beta}\big[\mathds{1}[\beta \geq \theta]\mathds{1}[\boldsymbol{\hat{a}}^m \neq \boldsymbol{\hat{a}}^s]\big] =  \\ \int_0^1 \sum_{\boldsymbol{a}, \boldsymbol{\hat{a}}^m, \boldsymbol{\hat{a}}^s} \mathds{1}[\beta \geq \theta]\mathds{1}[\boldsymbol{\hat{a}}^m \neq \boldsymbol{\hat{a}}^s] f(\boldsymbol{a}, \boldsymbol{\hat{a}}^m, \boldsymbol{\hat{a}}^s, \beta)\,d\beta = \\ \int_0^1 \sum_{\boldsymbol{\hat{a}}^m, \boldsymbol{\hat{a}}^s} \mathds{1}[\beta \geq \theta]\mathds{1}[\boldsymbol{\hat{a}}^m \neq \boldsymbol{\hat{a}}^s] \Big( \sum_{\boldsymbol{a}} f(\boldsymbol{a}, \boldsymbol{\hat{a}}^m, \boldsymbol{\hat{a}}^s, \beta) \Big) d\beta = \\  \int_0^1 \sum_{\boldsymbol{\hat{a}}^m, \boldsymbol{\hat{a}}^s} \mathds{1}[\beta \geq \theta]\mathds{1}[\boldsymbol{\hat{a}}^m \neq \boldsymbol{\hat{a}}^s] f(\boldsymbol{\hat{a}}^m, \boldsymbol{\hat{a}}^s, \beta) \,d\beta\,.
\end{multline}

Using the chain rule of probability, we can split the joint distribution as $f(\boldsymbol{\hat{a}}^m, \boldsymbol{\hat{a}}^s, \beta) = f(\beta)\cdot f(\boldsymbol{\hat{a}}^m, \boldsymbol{\hat{a}}^s \mid \beta)$ and thus write (\ref{eq:exp_error_1}) as:
\begin{multline}
    \text{Error}_\text{ST-O}(\theta) = \\ \int_0^1 \mathds{1}[\beta \geq \theta] f(\beta) \Big[  \sum_{\boldsymbol{\hat{a}}^m, \boldsymbol{\hat{a}}^s} \mathds{1}[\boldsymbol{\hat{a}}^m \neq \boldsymbol{\hat{a}}^s] f(\boldsymbol{\hat{a}}^m, \boldsymbol{\hat{a}}^s \mid \beta)\Big]d\beta\,.
\end{multline}
To simplify this expression, we can use the definition of conditional probability, the expectation of indicator functions (which is equal to the probability of the event)\footnote{$\Pr(A) = \mathbb{E}[\mathds{1}(A)]$} and the law of total expectation\footnote{$\mathbb{E}[Z] = \mathbb{E}[\mathbb{E}[Z \mid \beta]]$ with $Z = \mathds{1}[\beta \geq \theta] \mathds{1}[\boldsymbol{a} \neq \boldsymbol{\hat{a}}]$} to acquire:
\begin{equation}
    \text{Error}_\text{ST-O}(\theta) = \int_\theta^1 \varepsilon_s(\beta)f(\beta)\,d\beta\,,
\end{equation}
where $\varepsilon_s(\beta) \coloneq \Pr(\boldsymbol{\hat{a}}^m \neq \boldsymbol{\hat{a}}^s \mid \beta) = \sum_{\boldsymbol{\hat{a}}^m} \sum_{\boldsymbol{\hat{a}}^s} \mathds{1}[\boldsymbol{\hat{a}}^m \neq \boldsymbol{\hat{a}}^s] f(\boldsymbol{\hat{a}}^m, \boldsymbol{\hat{a}}^s \mid \beta)$, which is the computation of the expectation of the indicator under the joint conditional distribution. So, the optimization objective that has to be solved reduces to:
\begin{equation}
\label{eq:oracle_obj}
\begin{aligned}
\min_{\theta \in [0,1]}  \quad & \text{Cost}_\text{ST}(\theta)\\
\textrm{s.t.} \quad & \text{Error}_\text{ST-O}(\theta) \leq 1 - \xi \,,\\
\end{aligned}
\end{equation}
where 
$\xi \in [0,1]$ 
is a desired value for a task-specific metric that measures performance. 

Since $\text{Cost}_\text{ST}(\theta)$ 
is non-decreasing in $\theta$, the optimal value $\theta^*$ is the \textit{smallest} that satisfies the constraint:
\begin{equation}
\label{eq:opt_thr}
    \theta^* = \inf \left\{ \theta \in [0,1]\,\, \bigg| \int_\theta^1 \varepsilon_s(\beta)f(\beta) \leq 1 - \xi \right\} \,.
\end{equation}
Under the (mild) assumptions that $\varepsilon_s(\beta) \in [0,1]$ and $f(\beta) \geq 0$, (\ref{eq:opt_thr}) obtains the optimal threshold. We present the main theoretical results below. The complete proofs, together with an illustrative example admitting closed-form expressions, are provided in Appendix~\ref{app:st-o-proofs}.\footnote{Here, optimality is defined with respect to a threshold-based policy. Under mild assumptions, any $\beta$-based rule can be reduced to a threshold policy; see Appendix \ref{app:single_general_policies} for details.
}
\begin{lemma}[Monotonicity of cost and error in the oracle case] \label{lemma:monotonicity} Let $f:[0,1]\to[0,\infty)$ be integrable with CDF $F(\theta)=\int_{0}^{\theta} f(\beta)\,d\beta$. Let $\varepsilon:[0,1]\to[0,1]$ be any measurable function. For $\theta\in[0,1]$, \begin{equation} C(\theta)=c_s+c_m\,F(\theta),\qquad g(\theta)=\int_{\theta}^{1}\varepsilon(\beta)\,f(\beta)\,d\beta, \end{equation} with constants $c_s>0$, $c_m>0$. Then $C(\cdot)$ is non-decreasing and $g(\cdot)$ is non-increasing in $[0,1]$. \end{lemma}
\begin{proof}
See Appendix~\ref{app:st-o-proofs}
\end{proof}
\begin{theorem}[Optimal threshold in the oracle case]\label{thm:st-o}
Fix an error budget $b\in[0,1]$ and consider
\[
\min_{\theta\in[0,1]} C(\theta)\quad\text{s.t.}\quad g(\theta)\le b,
\]
with $C(\cdot)$ and $g(\cdot)$ as in Lemma~\ref{lemma:monotonicity}. 
Let $\theta^* = \inf\{\theta\in[0,1]: g(\theta)\le b\}$.
Then, if the problem is feasible, every optimal solution is attained at the leftmost feasible threshold, i.e., $\theta^{\mathrm{opt}} = \theta^*$. Moreover, if $C$ is strictly increasing, then the optimizer is unique. If, in addition, $g$ is strictly decreasing, then the constraint is tight at optimum and $\theta^*$ is the unique solution of $g(\theta)=b$.
\end{theorem}
\begin{proof}
See Appendix~\ref{app:st-o-proofs}.
\end{proof}
\paragraph{Intuition of Theorem~\ref{thm:st-o}.}
Increasing $\theta$ means deferring more often to mLLM. This always increases cost, but it can only decrease error (Lemma~\ref{lemma:monotonicity}). Therefore, the best strategy is simple: increase $\theta$ only as much as needed to satisfy the target error budget, and stop at the \emph{smallest feasible threshold}. 

We now drop the oracle assumption and allow mLLM to make mistakes. The mLLM model has its own probability of error $\varepsilon_m(\gamma)$. 
To simplify our analysis, we assume that the error rates of the two models are independent so we can `compress' the error curve of mLLM into a constant $\bar{\varepsilon}_m$\footnote{$\bar{\varepsilon}_m = \mathbb{E} \big[  \varepsilon_m(\gamma) \mid \beta  \big] = \mathbb{E}_{\gamma} \big[  \varepsilon_m(\gamma)  \big] = \int_0^1 \varepsilon_m(\gamma)f_m(\gamma)\,d\gamma$.}. 
While the cost remains the same as in the oracle case, the query-based error now becomes: 
\begin{multline}
    e(\boldsymbol{a}_i, \boldsymbol{\hat{a}}^s_i, \boldsymbol{\hat{a}}^m_i, \beta_i, \theta) = \\ 
    \begin{gathered}
        \left( 1 - \pi(\beta_i, \theta) \right)\mathds{1}[\boldsymbol{a}_i \neq \boldsymbol{\hat{a}}^s_i] + \pi(\beta_i, \theta) \mathds{1}[\boldsymbol{a}_i \neq \boldsymbol{\hat{a}}^m_i] = \\ \mathds{1}[\beta_i \geq \theta]\mathds{1}[\boldsymbol{a}_i \neq \boldsymbol{\hat{a}}^s_i] + \mathds{1}[\beta_i < \theta]\mathds{1}[\boldsymbol{a}_i \neq \boldsymbol{\hat{a}}^m_i]\,,
    \end{gathered}
\end{multline}
while, following the same steps as in~\eqref{eq:exp_error_1}, the expected error becomes:
\begin{equation} \label{eq:exp_error_2}
    \text{Error}_\text{ST-NO}(\theta) = \int_\theta^1 \varepsilon_s(\beta)f(\beta)\,d\beta + \bar{\varepsilon}_m \int_0^\theta f(\beta)\,d\beta\,.
\end{equation}
Note that in the non-oracle case the {slope} of the error curve can be positive or negative, depending on whether sLLM is locally better or worse than mLLM at score level $\theta$.\footnote{The error function is continuous (proof in Appendix \ref{app:continuity}) and the derivative is $\odv{}{\theta} \text{Error}_\text{NO}(\theta) = -(\varepsilon_s(\theta) - \bar{\varepsilon}_m)f(\theta)$.} Thus, unlike in the oracle case, $\text{Error}_\text{ST-NO}(\theta)$ need not be monotone in $\theta$.

The constraint of the optimization problem in (\ref{eq:oracle_obj}) is replaced by (\ref{eq:exp_error_2}) and the solution of (\ref{eq:opt_thr}) still provides the optimal threshold. As previously, we present the theorem's results here and leave the proof in Appendix~\ref{app:st-o-proofs}

\begin{theorem}[Optimal threshold in the non-oracle case]
\label{thm:st-no}
Fix an error budget $b\in[0,1]$ and consider 
\[
\min_{\theta\in[0,1]} C(\theta)
\quad\text{s.t.}\quad
g_{\mathrm{NO}}(\theta)\le b,
\]
where $C(\cdot)$ as in Lemma~\ref{lemma:monotonicity} is non-decreasing and
$g_{\mathrm{NO}}(\cdot) = \text{Error}_\text{ST-NO}(\cdot)$. Define $\mathcal{F}_b=\{\theta\in[0,1]: g_{\mathrm{NO}}(\theta)\le b\}$ and $\theta^*=\inf \mathcal{F}_b$.
Assume $\mathcal{F}_b\neq\emptyset$. Then:
(i) $\theta^*$ is an optimal solution, even if $\mathcal{F}_b$ is a disjoint union of intervals;
(ii) if $g_{\mathrm{NO}}(0)\le b$, then $\theta^*=0$; 
(iii) if $g_{\mathrm{NO}}(0)>b$, then $\theta^*>0$ and the constraint is tight:
$g_{\mathrm{NO}}(\theta^*)=b$.
\end{theorem}
\begin{proof}
See Appendix~\ref{app:st-no-proofs}.
\end{proof}
\paragraph{Intuition of Theorem~\ref{thm:st-no}.}The error function is not necessarily monotone in $\theta$, and $\mathcal{F}_b$ can be a union of disjoint intervals. Theorem~\ref{thm:st-no}
shows that, within the single-threshold family, the optimal threshold is still the smallest feasible one, purely because $C(\cdot)$ is non-decreasing in $\theta$.

\subsection{Discriminative tasks - Multi-thresholds policy}\label{subsec:multi_class_clf}
We now focus on discriminative tasks (e.g., classification) where LLMs can be used to provide answers regarding specific categories. We consider $K$-way classification tasks where $\mathcal{A} = \{1, \ldots, K\}$ and now, for a sample $\boldsymbol{q}$, sLLM outputs the predicted label $\text{sLLM}(\boldsymbol{q}) = \boldsymbol{\hat{a}}^s \in \mathcal{A}$ and a \textit{vector} of confidence scores \(\boldsymbol{\beta} \in \Delta^{K-1} = \left\{ \boldsymbol{\beta} \in [0,1]^K \mid \sum_{j=1}^K \beta_{j} = 1 \right\}\), where $\beta_{ j}$ is the confidence score that sLLM assigns to the $j$-th class. Consequently, the standard approach is to use the most confident class as the predicted one, i.e., $\boldsymbol{\hat{a}} = \arg\max_{j = 1, \dots, K} \beta_{ j}$ \citep{chow1970}.
In realistic classification settings (especially with class imbalance or asymmetric class difficulty), performance over different classes can vary and models can display a different behavior per class. As a result, and to further generalize our framework, we extend the decision policy to a multi-threshold one that is more flexible and can adjust to class-specific calibration and difficulty, e.g., one might want to trust easier-to-classify classes more 
and be more cautious with harder classes.
We now consider a threshold vector $\boldsymbol{\theta} = (\theta_1, \ldots, \theta_K)$, where $\theta_k$ is the threshold associated to the $k$-th class. Therefore, the prediction from sLLM is now only accepted when $\beta_{k} \geq \theta_k$. Thus, we can define the cost for the {$i$-th query} as:
\begin{equation}
    c(\beta_{i,k}, \theta) = c_s + \pi(\beta_{i,k}, \theta_k) \cdot c_m = c_s + \mathds{1}[\beta_{i,k} < \theta_k] \cdot c_m\,,
\end{equation}
and the \textit{expected} cost is expressed as:
\begin{equation}
\label{eq:mt-cost}
    \text{Cost}_\text{MT}(\boldsymbol{\theta}) = c_s + c_m \sum_{k=1}^K p_k \int_0^{\theta_k} f_k(\beta)\,d\beta\,,
\end{equation}
where $\boldsymbol{p} \in \Delta^{K-1} = \left\{ \boldsymbol{p} \in [0,1]^K \mid \sum_{j=1}^K p_j = 1 \right\}$ are (predicted) class priors. Following the two different settings, we will again provide different formulations for the error. In the oracle setting, the error for the {$i$-th query} can be defined as:
\begin{equation}
\begin{aligned}
    e(\boldsymbol{\hat{a}}^m_i, \boldsymbol{\hat{a}}^s_i, \beta_{i, k}, \theta_k) &= \left(1-\pi(\beta_{i,k}, \theta_k)\right)\mathds{1}[\boldsymbol{\hat{a}}^m_i \neq \boldsymbol{\hat{a}}^s_i] \\ &= \mathds{1}[\beta_{i,k} \geq \theta_k]\mathds{1}[\boldsymbol{\hat{a}}^m_i \neq \boldsymbol{\hat{a}}^s_i] \,.
\end{aligned}
\end{equation}
Furthermore, we define $\varepsilon_{s,k}(\beta) = \Pr(\boldsymbol{\hat{a}}^m \neq k \mid \boldsymbol{\hat{a}}^s = k, \beta_k = \beta), \, \forall\, k =1,\ldots, K$ and the expected error becomes:
\begin{equation}
    \text{Error}_\text{MT-O}(\boldsymbol{\theta}) = \sum_{k=1}^K p_k \int_{\theta_k}^1 \varepsilon_{s,k}(\beta)f_k(\beta)\,d\beta\,.
\end{equation}
Consequently, the optimization objective now becomes:
\begin{equation}
\label{eq:multi_class_obj}
\begin{aligned}
\min_{\boldsymbol{\theta} \in [0,1]^K}  \quad & \text{Cost}_\text{MT}(\boldsymbol{\theta})\\
\textrm{s.t.} \quad & \text{Error}_\text{MT-O}(\boldsymbol{\theta}) \leq 1 - \xi \,.
\end{aligned}
\end{equation}
The following theorem shows that for some $t \in \mathbb{R}$, any optimal threshold vector must satisfy $\varepsilon_{s,1}(\theta^*_1) = \ldots = \varepsilon_{s, K} (\theta^*_K) = t$ under the assumptions that $\varepsilon_{s,k}(\beta) \in [0,1]$ and $f_k(\beta) \geq 0,\, \forall\, k=1,\ldots,K$. The proof can be found in Appendix~\ref{app:mt-o-proof}.
\begin{lemma}[Coordinate-wise monotonicity in the oracle case] \label{lemma:coordinate_monotonicity}
For each class $k\in\{1,\dots,K\}$ let $f_k:[0,1]\to[0,\infty)$ be integrable with CDF $F_k(\theta)=\int_0^\theta f_k(\beta)\,d\beta$, and let $\varepsilon_k:[0,1]\to[0,1]$ be any measurable function. Let also $p_k\ge 0$ with $\sum_k p_k=1$, and constants $c_s,c_m\ge 0$. For a threshold vector $\boldsymbol{\theta}=(\theta_1,\dots,\theta_K)\in[0,1]^K$ define:
\begin{equation*}
\begin{aligned}
C(\boldsymbol{\theta})&=c_s+c_m\sum_{k=1}^K p_k\,F_k(\theta_k)\,, \\
g(\boldsymbol{\theta})&=\sum_{k=1}^K p_k\int_{\theta_k}^{1}\varepsilon_k(\beta)\,f_k(\beta)\,d\beta\,.
\end{aligned}
\end{equation*}
Then, for each coordinate $k$, $C(\cdot)$ is non-decreasing in $\theta_k$ and $g(\cdot)$ is non-increasing in $\theta_k$.
Consequently, for any error budget $b\in[0,1]$, the feasible set
$
\mathcal{F}_b=\{\boldsymbol{\theta}\in[0,1]^K:~g(\boldsymbol{\theta})\le b\}
$
is an \emph{upper-orthant}: if $\boldsymbol{\theta}\in\mathcal{F}_b$ and $\boldsymbol{\theta}'\ge \boldsymbol{\theta}$ coordinate-wise, then $\boldsymbol{\theta}'\in\mathcal{F}_b$.
\end{lemma}

\begin{proof}
See Appendix~\ref{app:mt-o-proof}
\end{proof}
\begin{theorem}[Optimal threshold vector in the oracle case]\label{thm:mt-o}
Consider the constrained problem
\begin{equation}
\min_{\boldsymbol{\theta}\in[0,1]^K}~ C(\boldsymbol{\theta})\qquad\text{s.t.}\qquad g(\boldsymbol{\theta})\le b,
\end{equation}
with $C(\cdot),g(\cdot)$ as in Lemma~\ref{lemma:coordinate_monotonicity} and $b\in[0,1]$.

Then:
(i) There exists an optimal solution $\boldsymbol{\theta}^*$ with {tight} constraint $g(\boldsymbol{\theta}^*)=b$, unless $\boldsymbol{\theta}=\boldsymbol{0}$ is already feasible (in which case $\boldsymbol{\theta}^*=\boldsymbol{0}$), 
(ii) At any optimal solution with coordinates in the interior of $\mathcal{F}_b$ (i.e., not forced to $0$ or $1$), the {boundary errors} are equalized:
$\varepsilon_{s,1}(\theta_1^*)=\ldots=\varepsilon_{s,K}(\theta_K^*)=t$ and (iii) for each interior class $k$, 
we can take $ \theta_k^* \in \varepsilon_{s,k}^{-1}(t)$, with $t$ determined by
$g(\boldsymbol{\theta}^*)= b
$.
\end{theorem}
\begin{proof}
See Appendix~\ref{app:mt-o-proof}.
\end{proof}
\paragraph{Intuition of Theorem~\ref{thm:mt-o}.}
If one class is being deferred at a boundary point where the sLLM is much worse than another class at its boundary, then we can reallocate deferrals across classes and reduce cost without violating the error budget. Thus, the optimum equalizes the `boundary errors' across all interior classes, while easy classes saturate at always-accept and difficult classes may saturate at always-defer.
\footnote{We again focus our proof on the threshold-based policy; detailed proof for more general cases are provided in Appendix \ref{app:multi_general_policies}.
}

We now move to the non-oracle setting, where mLLM has a non-zero error probability. 
The cost again remains the same as in the oracle case~(\ref{eq:mt-cost}) while the sample-based error now becomes: 
\begin{multline}
    e(\boldsymbol{a}_i, \boldsymbol{\hat{a}}^s_i, \boldsymbol{\hat{a}}^m_i, \beta_{i,k}, \theta_k) = \\
    \begin{gathered}
    \left( 1 - \pi(\beta_{i,k}, \theta_k) \right)\mathds{1}[\boldsymbol{a}_i \neq \boldsymbol{\hat{a}}^s_i] + \pi(\beta_{i,k}, \theta_k) \mathds{1}[\boldsymbol{a}_i \neq \boldsymbol{\hat{a}}^m_i] = \\ \mathds{1}[\beta_{i,k} \geq \theta_k]\mathds{1}[\boldsymbol{a}_i \neq \boldsymbol{\hat{a}}^s_i] + \mathds{1}[\beta_{i,k} < \theta_k]\mathds{1}[\boldsymbol{a}_i \neq \boldsymbol{\hat{a}}^m_i]\,.
    \end{gathered}
\end{multline}
We again assume the mLLM’s error rate is independent of the sLLM’s outputs, 
so the error curve can be summarized by a constant $\bar{\varepsilon}_{m,k}\in [0,1]$ and 
the expected error becomes:
\begin{multline}
\label{eq:exp_error_4}
    \text{Error}_\text{MT-NO}(\theta) = 
    \mathbb{E} \big[  e(\boldsymbol{a}_i, \boldsymbol{\hat{a}}^s_i, \boldsymbol{\hat{a}}^m_i, \beta_{i,k}, \theta_k)  \big] = \\ 
    \begin{gathered}
    \sum_{k=1}^K p_k
\left(
    \int_{\theta_k}^1 \varepsilon_{s,k}(\beta)\,f_k(\beta)\,d\beta
    \;+\;
    \bar{\varepsilon}_{m,k} \int_0^{\theta_k} f_k(\beta)\,d\beta
\right)\,.
\end{gathered}
\end{multline}
The constraint of the optimization problem in (\ref{eq:multi_class_obj}) is replaced by (\ref{eq:exp_error_4}). In contrast to the oracle case, the error in the non-oracle setting is not guaranteed to be coordinate-wise non-increasing in each $\theta_k$, because increasing a threshold both removes
sLLM errors and can add mLLM errors. Consequently, the feasible set
$\mathcal{F}_b$ need not be an upper-orthant; it can be a more complicated (even disjoint) subset of $[0,1]^K$. Nonetheless, as in the single-threshold case, the cost remains monotone in each coordinate, and this alone suffices to characterize the structure of optimal thresholds below.

\begin{theorem}[Optimal threshold vector in the non-oracle case]
\label{thm:mt-no}
Consider the constrained problem
\begin{equation}
\min_{\boldsymbol{\theta}\in[0,1]^K}~ C(\boldsymbol{\theta})
\qquad\text{s.t.}\qquad g_{\mathrm{NO}}(\boldsymbol{\theta})\le b,
\label{eq:no_constrained}
\end{equation}
with $C(\cdot)$ remaining the same as in the oracle case
and the non-oracle error as in (\ref{eq:exp_error_4}) (i.e., $g_{\mathrm{NO}(\cdot)} = \text{Error}_\text{MT-NO}(\cdot)$).
Define the \emph{excess error} curves
$
\Delta_k(\beta):=\varepsilon_{s,k}(\beta)-\bar{\varepsilon}_{m,k}, \,\,\text{for}\,\, k=1,\dots,K\,.
$

Then:
(i) there exists an optimal solution $\boldsymbol{\theta}^*$ with \emph{tight} constraint $g_{\mathrm{NO}}(\boldsymbol{\theta}^*)=b$, unless $\boldsymbol{\theta}=\boldsymbol{0}$ is already feasible (in which case $\boldsymbol{\theta}^*=\boldsymbol{0}$), (ii) 
For any optimal solution $\boldsymbol{\theta}^*$ and two classes $i,j$ with coordinates interior in $\mathcal{F}_b$ that satisfy
$
0<\theta_i^*,\theta_j^*<1
\,\,\text{and}\,\,
\Delta_i(\theta_i^*)>0,\ \Delta_j(\theta_j^*)>0
$, their excess boundary errors must be equal:
$
\Delta_i(\theta_i^*)=\Delta_j(\theta_j^*)=t\in\mathbb{R}
$ ($t$ is determined
by the tightness equation $g_{\mathrm{NO}}(\boldsymbol{\theta}^*)=b$). Consequently, for every class $k$ with
$0<\theta_k^*<1$ and $\Delta_k(\theta_k^*)>0$\,,
$
\Delta_k(\theta_k^*) = t
\label{eq:excess_equal_all}
$, and (iii) 
an optimal solution $\boldsymbol{\theta}^*$ can be chosen to satisfy one of the following three regimes, for each class $k$: (a) 
if $\Delta_k(\beta)\le t$ for all $\beta$, then $\theta_k^*=0$, (b) if $\Delta_k(\cdot)$ crosses level $t$, then $\theta_k^*\in \Delta_k^{-1}(t)$ and $0<\theta_k^*<1$, and (c)
if $\Delta_k(\beta)> t$ for all $\beta$, then $\theta_k^*=1$. \newline
\end{theorem}

\begin{proof}
    See Appendix~\ref{app:mt-no-proof}.
\end{proof}

\paragraph{Intuition of Theorem~\ref{thm:mt-no}.} The relevant quantity is no longer the sLLM error alone, but the {excess error} $\Delta_k(\beta)=\varepsilon_{s,k}(\beta)-\bar{\varepsilon}_{m,k}$, which measures how much worse (or better) the sLLM is than mLLM for class $k$ at confidence level $\beta$. The optimal policy equalizes this excess boundary error across all interior classes. Hence, the oracle equalization rule extends to the non-oracle case by replacing $\varepsilon_{s,k}$ by $\Delta_k$. If one class had higher excess boundary error than another, reallocating deferrals would reduce cost without violating the error constraint, contradicting optimality. Classes always below this level are handled by the sLLM (always accept), while those always above it are always deferred to the mLLM.

\subsection{Summary of Theoretical Results}
\label{subsec:theory_summary}

We summarize the main theoretical conclusions for the four settings considered this Section:

\textbf{1) Single-threshold, oracle setting.}
Error decreases and cost increases monotonically with the threshold $\theta$. 
The optimal policy is the smallest threshold that meets the error budget: defer only as much as necessary.\\
\textbf{2) Single-threshold, non-oracle setting.}
Although error is no longer monotone in $\theta$, cost remains monotone. The optimal policy is still the leftmost feasible threshold (the least costly one satisfying the error constraint).\\
\textbf{3) Multi-threshold, oracle setting.}
With class-specific thresholds, the optimum equalizes a common boundary error level across all classes in the interior of the feasible set.\\ 
\textbf{4) Multi-threshold, non-oracle setting.} 
The optimum equalizes excess boundary error across interior classes, given that the relevant quantity is the excess error (sLLM minus mLLM error).\\
\textbf{Overall implication.}
Across all cases, cost monotonicity selects the least expensive feasible policy, and in the multi-threshold setting optimality is characterized by an equalization principle (boundary error or excess boundary error).

\section{Experimental results}
\label{sec:simulation}
In this section, we evaluate our policies in both discriminative and generative settings. We use text classification as a representative discriminative task, and question-answering (QA) and machine-translation (MT) as generative benchmarks, demonstrating effectiveness across diverse problem settings. Given a finite sample of data, we approximate expected cost and error via efficient Monte-Carlo (MC) sampling algorithms (more details about the approximation and the implemented algorithms can be found in Appendices \ref{app:mc_sampling_details} and \ref{app:algo_details}, respectively)..
\textbf{Models.} We experiment with three open-source LLM families of various sizes, namely, different versions of the \textsc{gemma3} model \citep{gemma3} (i.e., $1$B up to $27$B parameters), \textsc{qwen3} \citep{qwen3} with $4$B and \textsc{ministral3} \citep{liu2026ministral3} with $3$B parameters. 
We considered both homogeneous pairings, where both models belong to the same family, and heterogeneous pairings, combining models from different families. \\
\textbf{Datasets.} We test our framework on several classification datasets. For binary classification, we use \textsc{sst-2} \citep{socher-etal-2013-recursive} and \textsc{fakenews} \citep{fakenews}; while for multi-class classification, we used \textsc{agnews} \citep{agnews} and \textsc{emotion} \citep{saravia-etal-2018-carer}. For the QA task, we select the \textsc{SQuAD} dataset \citep{rajpurkar-etal-2016-squad} and for the MT task, the \textsc{WMT} english-to-german translate dataset \citep{bojar-EtAl:2014:W14-33}. We use a 80$\%$ split for the optimization process and the rest 20$\%$ for testing.\footnote{More details about the datasets in Appendix \ref{app:datasets}. Due to computational restrictions, we limit our experiments to a subset of 8,000 samples from each dataset.}\\
\textbf{Evaluation metrics.} We use standard, task-appropriate metrics. For classification, we report \emph{Accuracy}, and the $F_1$ score. For text generation, we measure semantic similarity via cosine similarity between predicted and reference embeddings \citep{DBLP:conf/emnlp/ReimersG19}, and lexical overlap using \textsc{rouge-l} \citep{lin-2004-rouge}. Inference cost is measured in TFLOPs using the PyTorch profiler \citep{DBLP:conf/nips/PaszkeGMLBCKLGA19} for per-query estimates; then, model-specific cost constants $c_s$ and $c_m$ are computed by summarizing these estimated (e.g., with statistics like the mean or the median). We report efficiency using a `cost savings' metric, defined as the percentage reduction in total cost relative to using only the mLLM for all queries.\\
\textbf{Baselines.} We compare our policies against: \emph{(i) a no-deferral strategy}, where all queries are handled by sLLM; \emph{(ii) a full-deferral strategy}, where all queries are sent to mLLM; \emph{(iii) a random-deferral strategy}, deferring each query independently with probability 0.5; \emph{(iv) FrugalGPT} \citep{frugal}, a cascade framework that uses an external scorer and optimizes a single deferral threshold; and \emph{(v) HybridLLM} \citep{hybridlllm}, which trains an external router to assign each query to one of the two models.

We make our code available in a temporary repository: \href{https://github.com/anonym-conf/inf-nets}{https://github.com/anonym-conf/inf-nets}.

\subsection{Generative Tasks}
We utilize the single-threshold policy in the free-form generation tasks. In each case, the model generates an output sequence conditioned on the given context and question for QA, and a sentence to be translated for MT.
For the confidence value, we use the average log probability of the predicted tokens \citep{lm-polygraph}. 
Additionally, to measure the expected probability of error, we defined the following notion of correctness: 
{At first, the embeddings of both the predicted and true answers are computed} using Sentence Transformers \citep{DBLP:conf/emnlp/ReimersG19}. Secondly, we calculate the cosine similarity between the embeddings of the true and predicted answers. Then, a threshold $\tau$ is applied to the similarity score to determine correctness, i.e., $\mathds{1}[\hat{\boldsymbol{a}} = \boldsymbol{a}] = \mathds{1}[\mathrm{sim}(\hat{\boldsymbol{a}}, \boldsymbol{a}) \geq \tau]$.
We present results for both datasets in Table~\ref{tab:single-thres_results} for the \textsc{gemma3-1b} (as sLLM) and \textsc{gemma3-4b} (as mLLM) models pair in the non-oracle setting (results with additional model pairs can be found in Appendix \ref{app:additional_results}). For the optimization, we set the bound $\xi$ as the performance of mLLM, 
{while} for FrugalGPT, we use our policy's achieved cost as the budget. HybridLLM solves a slightly different objective that minimizes cost savings (deferred samples) but subject to keeping the performance drop (reduction in BART score \citep{bartscore}) less than a defined value (typically 1\%).
The results show that our method can deliver a strong performance–cost trade-off. On \textsc{SQuAD}, 
{our method} maintains 
{mLLM} performance with a 26.66\% reduction in inference cost, and on WMT (where mLLM was worse than sLLM), it maintains the best possible performance with maximal cost savings. These results demonstrate that our policy adapts to task difficulty, allocating compute only when necessary and avoiding the inefficiencies of both static and random strategies. During our experiments, we also observed that the performance of FrugalGPT and HybridLLM was strongly contingent on the effectiveness of the routing model and the data used to train it, something that represents a drawback for these methods. Furthermore, compared to FrugalGPT and HybridLLM, we also save up the cost of training the additional routing model. 
The error and cost trade-offs for the \textsc{SQuAD} dataset are also illustrated in Fig.~\ref{fig:tradeoff}.
\begin{table}[h!]
  \caption{{Performance comparison of our policy with other policies }
  on generation tasks (\textsc{gemma3-1b} as sLLM and \textsc{gemma3-4b} as mLLM).}
  \label{tab:single-thres_results}
  \centering
    \begin{small}
      \begin{adjustbox}{max width=\columnwidth}
        \begin{tabular}{cccc}
          \toprule
          \bf Dataset & \bf Policy & \bf Performance (\textsc{rouge-l} / \textsc{cosine}) & \bf Cost Saved ($\%$) \\
          \midrule 
           \multirow{4}{*}{{\sc SQuAD}} & None Deferred & 0.39 / 0.49 & 99.09$\%$ \\
           &  All Deferred & 0.57 / 0.65 & 0$\%$\\
           &  Random & 0.48 / 0.57 & 48.28$\%$\\
           &  FrugalGPT & 0.57 / 0.64 & 26.22$\%$\\
           &  HybridLLM & 0.57 / 0.65 & 0.66$\%$\\
           &  Ours & 0.57 / 0.64 & 26.66$\%$\\
           \midrule
           \multirow{4}{*}{{\sc WMT}} & None Deferred & 0.37 / 0.62 & 97$\%$ \\
           &  All Deferred & 0.19 / 0.32 & 0$\%$\\
           &  Random & 0.28 / 0.46 & 46.18$\%$\\
           &  FrugalGPT & 0.37 / 0.62 & 97$\%$\\
           &  HybridLLM & 0.37 / 0.62 & 97$\%$\\
           &  Ours & 0.38 / 0.62 & 97$\%$\\
          \bottomrule
        \end{tabular}
        \end{adjustbox}
    \end{small}
  \vskip -0.1in
\end{table}
\begin{figure}[!htb]
    \centering
\subfloat{\includegraphics[width=\columnwidth, ]{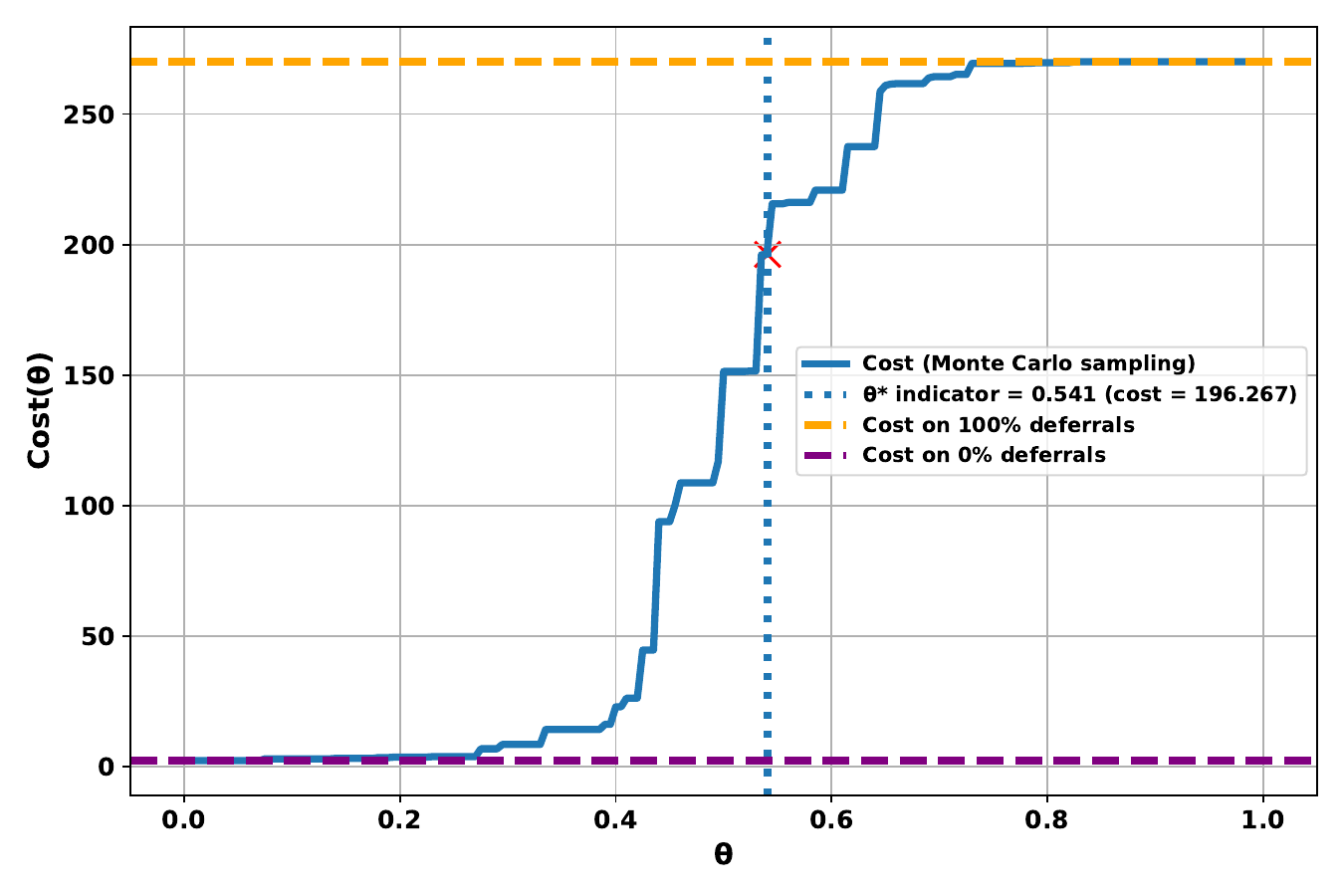}}\\[1.5mm]
\subfloat{\includegraphics[width=\columnwidth, ]{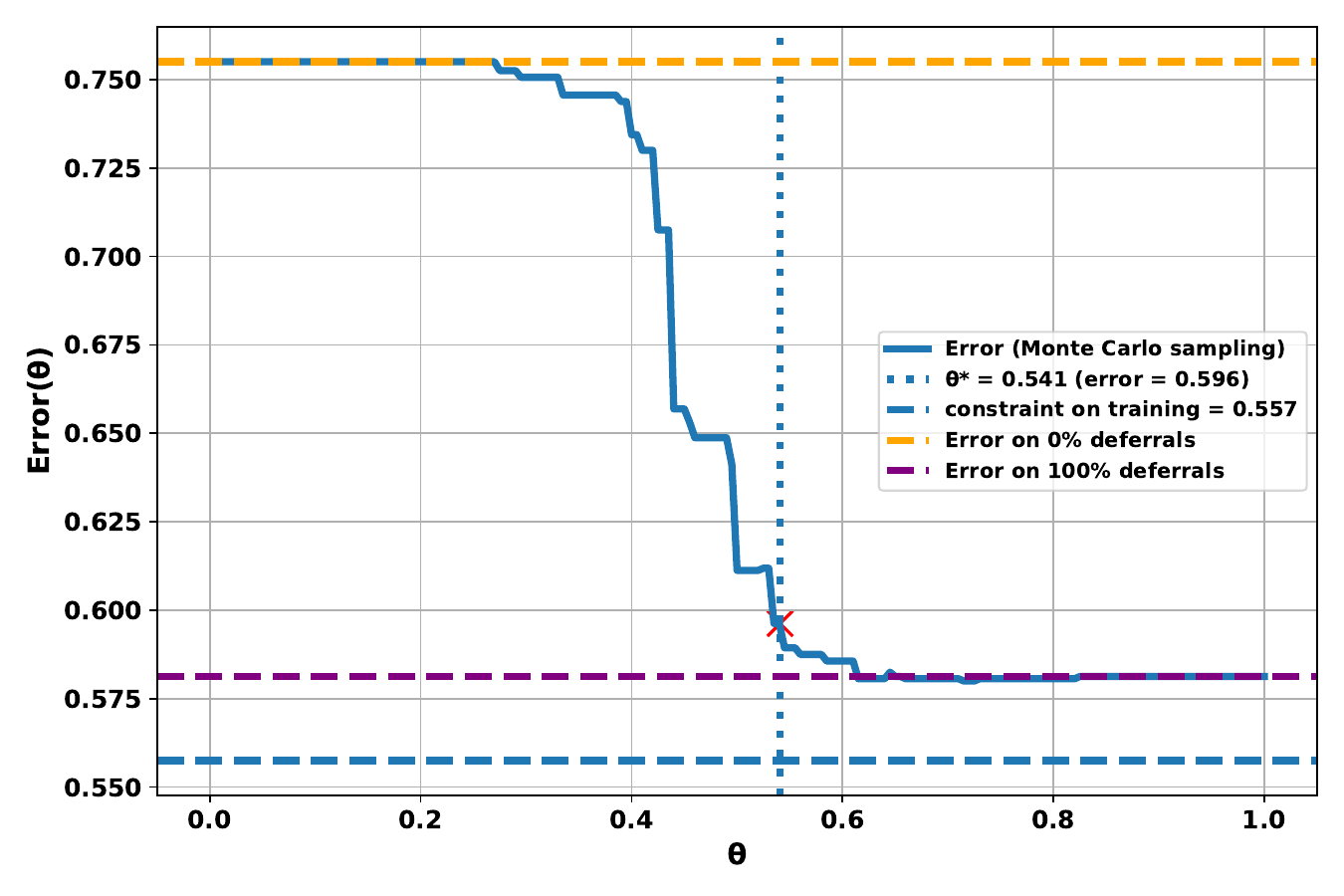}}
    \caption{Expected cost (top) and error (bottom) as functions of the deferral threshold $\theta$ (single-threshold policy) on the \textsc{SQuAD} test set. The deferral rate increases monotonically with the threshold. Horizontal lines indicate the cases of 0\% and 100\% deferral rate. The vertical line denotes the optimal threshold $\theta^*$ obtained by our policy.}
    \label{fig:tradeoff} 
\end{figure}

\subsection{Discriminative Tasks}
As described in Section~\ref{subsec:multi_class_clf}, we utilize the multi-threshold policy in zero-shot classification settings where a separate threshold is optimized for each category. For confidence estimation, we utilize the maximum softmax probability \citep{chow1970, committees} provided by the model. To simulate classification, we constrain the LLM to output only a single token corresponding to the predicted class label. Table~\ref{tab:multi-thres_results} shows the results of the experiments on the test set of each classification dataset for the \textsc{gemma3-1b} (as sLLM) and \textsc{gemma3-4b} (as mLLM) models pair in the non-oracle setting (results with additional model pairs can be found in Appendix \ref{app:additional_results}). 
Across all datasets, our policy consistently maintains near–`optimal' performance (i.e., comparable to deferring all samples to the mLLM) while reducing inference cost {significantly}. On \textsc{sst-2}, it matches full-deferral accuracy with $\sim$82\% cost savings, and on \textsc{agnews} and \textsc{emotion}, it achieves full-deferral performance with $\sim$16\% and $\sim$68\% savings, respectively. Even on the more challenging \textsc{fakenews} dataset, 
the performance {degrades only $\sim$2\%} from full-deferral while cutting cost by $\sim$45\%.

\begin{table}[h!]
  \caption{Performance comparison of our policy with other policies
  on classification tasks using the \textsc{gemma3-1b} (sLLM) and \textsc{gemma3-4b} (mLLM) models pair.}
  \label{tab:multi-thres_results}
  \centering
    \begin{small}
      \begin{adjustbox}{max width=\columnwidth}
        \begin{tabular}{cccc}
          \toprule
          \bf Dataset & \bf Policy & \bf Performance ($\bf Accuracy$ / $\bf F_1$) & \bf Cost Saved ($\bf \%$) \\
          \midrule 
           \multirow{4}{*}{{\sc sst-2}} & None Deferred & 0.84 / 0.84 & 98.34$\%$ \\
           &  All Deferred & 0.90 / 0.90 & 0$\%$\\
           &  Random & 0.87 / 0.87 & 47.53$\%$\\
           &  FrugalGPT & 0.86 / 0.86 & 87.16$\%$\\
           &  HybridLLM & 0.86 / 0.86 & 61.66$\%$\\
           &  Ours & 0.90 / 0.90 & 82.09$\%$\\
           \midrule
           \multirow{4}{*}{{\sc agnews}} & None Deferred & 0.56 / 0.49 & 98.16$\%$ \\
           &  All Deferred & 0.78 / 0.74 & 0$\%$\\
           &  Random & 0.68 / 0.63 & 47.35$\%$\\
           &  FrugalGPT & 0.70 / 0.70 & 12.72$\%$\\
           &  HybridLLM & 0.77 / 0.74 & 8.03$\%$\\
           &  Ours & 0.78 / 0.75 & 16.10$\%$\\
           \midrule
           \multirow{4}{*}{{\sc emotion}} & None Deferred & 0.46 / 0.36 & 98.85$\%$ \\
           &  All Deferred & 0.56 / 0.46 & 0$\%$\\
           &  Random & 0.52 / 0.43 & 48.0$\%$\\
           &  FrugalGPT & 0.49 / 0.42 & 65.16$\%$\\
           &  HybridLLM & 0.53 / 0.42 & 42.91$\%$\\
           &  Ours & 0.56 / 0.48 & 67.85$\%$\\
           \midrule
           \multirow{4}{*}{{\sc fakenews}} & None Deferred & 0.52 / 0.40 & 96.54$\%$ \\
           &  All Deferred & 0.90 / 0.90 & 0$\%$\\
           &  Random & 0.72 / 0.70 & 45.92$\%$\\
           &  FrugalGPT & 0.90 / 0.89 & 44.84$\%$\\
           &  HybridLLM & 0.85 / 0.84 & 55.29$\%$\\
           &  Ours & 0.88 / 0.87 & 45.53$\%$\\
          \bottomrule
        \end{tabular}
        \end{adjustbox}
    \end{small}
  \vskip -0.1in
\end{table}
\section{Conclusion} 
We introduced \emph{inference networks} of LLM experts, modeled as graphs, where nodes represent alternative models and edges capture conditional model activations. We formalized the design of optimal model activation as minimizing the expected inference cost subject to a target performance constraint and we characterized the structure of the optimal activation rule in a two-model chain, proving that it admits a \emph{threshold form} on a scalar confidence score: for discriminative settings the policy decomposes into class-specific thresholds, while for generative settings it reduces to a single threshold. Beyond establishing this structural result, we provided approaches to compute the thresholds through efficient low-dimensional search (and closed-form solutions in special cases), making the proposed policy directly implementable. To bridge theory and practice, we applied practical confidence estimation methods to both discriminative and generative outputs and evaluated the framework using open-source LLMs on text classification and generation benchmarks. The results demonstrated that the proposed policies reduce inference costs while meeting a specified performance budget, improving the cost–quality trade-off in realistic deployments. \\
Several directions remain open. Firstly, the sequential decision-making nature of the problem makes it possible to apply reinforcement learning (RL) methods for
learning the optimal threshold parameter(s).
Secondly, while we view inference networks as general directed graphs, this work fixes the topology to be serial; extending the analysis to richer topologies
and jointly optimizing routing and topology is a natural next step. Finally, improving confidence estimation for generative tasks can further reduce the discrepancy between predicted confidence and actual quality, enhancing both performance and reliability in LLM-based inference.
\bibliography{uai2026-template}

@article{friedman2000greedy,
  author = {Friedman, Jerome H.},
  journal = {Annals of Statistics},
  pages = {1189--1232},
  title = {Greedy {F}unction {A}pproximation: {A} {G}radient {B}oosting {M}achine},
  volume = 29,
  year = 2000
}

@inproceedings{tabi,
author = {Wang, Yiding and Chen, Kai and Tan, Haisheng and Guo, Kun},
title = {Tabi: {A}n {E}fficient {M}ulti-{L}evel {I}nference {S}ystem for {L}arge {L}anguage {M}odels},
year = {2023},
isbn = {9781450394871},
publisher = {Association for Computing Machinery},
address = {New York, NY, USA},
url = {https://doi.org/10.1145/3552326.3587438},
doi = {10.1145/3552326.3587438},
booktitle = {Proceedings of the Eighteenth European Conference on Computer Systems},
pages = {233–248},
numpages = {16},
location = {Rome, Italy},
series = {EuroSys '23}
}

@inproceedings{lin-2004-rouge,
    title = "{ROUGE}: A {P}ackage for {A}utomatic {E}valuation of {S}ummaries",
    author = "Lin, Chin-Yew",
    booktitle = "Text Summarization Branches Out",
    month = jul,
    year = "2004",
    address = "Barcelona, Spain",
    publisher = "Association for Computational Linguistics",
    url = "https://aclanthology.org/W04-1013/",
    pages = "74--81"
}

@inproceedings{DBLP:conf/nips/PaszkeGMLBCKLGA19,
  author       = {Adam Paszke and
                  Sam Gross and
                  Francisco Massa and
                  Adam Lerer and
                  James Bradbury and
                  Gregory Chanan and
                  Trevor Killeen and
                  Zeming Lin and
                  Natalia Gimelshein and
                  Luca Antiga and
                  Alban Desmaison and
                  Andreas K{\"{o}}pf and
                  Edward Z. Yang and
                  Zachary DeVito and
                  Martin Raison and
                  Alykhan Tejani and
                  Sasank Chilamkurthy and
                  Benoit Steiner and
                  Lu Fang and
                  Junjie Bai and
                  Soumith Chintala},
  editor       = {Hanna M. Wallach and
                  Hugo Larochelle and
                  Alina Beygelzimer and
                  Florence d'Alch{\'{e}}{-}Buc and
                  Emily B. Fox and
                  Roman Garnett},
  title        = {Py{T}orch: {A}n {I}mperative {S}tyle, {H}igh-{P}erformance {D}eep {L}earning {L}ibrary},
  booktitle    = {Advances in Neural Information Processing Systems 32: Annual Conference
                  on Neural Information Processing Systems 2019, NeurIPS 2019, December
                  8-14, 2019, Vancouver, BC, Canada},
  pages        = {8024--8035},
  year         = {2019},
  url          = {https://proceedings.neurips.cc/paper/2019/hash/bdbca288fee7f92f2bfa9f7012727740-Abstract.html},
  bibsource    = {dblp computer science bibliography, https://dblp.org}
}

@article{kde_est,
 ISSN = {00034851, 21688990},
 URL = {http://www.jstor.org/stable/2237390},
 author = {Murray Rosenblatt},
 journal = {The Annals of Mathematical Statistics},
 number = {3},
 pages = {832--837},
 publisher = {Institute of Mathematical Statistics},
 title = {Remarks on {S}ome {N}onparametric {E}stimates of a {D}ensity {F}unction},
 urldate = {2026-01-26},
 volume = {27},
 year = {1956}
}

@article{BLUM1973448,
title = {Time bounds for selection},
journal = {Journal of Computer and System Sciences},
volume = {7},
number = {4},
pages = {448-461},
year = {1973},
issn = {0022-0000},
doi = {https://doi.org/10.1016/S0022-0000(73)80033-9},
url = {https://www.sciencedirect.com/science/article/pii/S0022000073800339},
author = {Manuel Blum and Robert W. Floyd and Vaughan Pratt and Ronald L. Rivest and Robert E. Tarjan}
}

@inbook{Bierens_1994, place={Cambridge}, title={The {N}adaraya–{W}atson kernel regression function estimator}, booktitle={Topics in Advanced Econometrics: Estimation, Testing, and Specification of Cross-Section and Time Series Models}, publisher={Cambridge University Press}, author={Bierens, Herman J.}, year={1994}, pages={212–247}}

@article{frugal,
  author       = {Lingjiao Chen and
                  Matei Zaharia and
                  James Zou},
  title        = {Frugal{GPT}: {H}ow to {U}se {L}arge {L}anguage {M}odels {W}hile {R}educing {C}ost and
                  {I}mproving {P}erformance},
  journal      = {Trans. Mach. Learn. Res.},
  volume       = {2024},
  year         = {2024},
  url          = {https://openreview.net/forum?id=cSimKw5p6R},
  bibsource    = {dblp computer science bibliography, https://dblp.org}
}

@InProceedings{bojar-EtAl:2014:W14-33,
  author    = {Bojar, Ondrej  and  Buck, Christian  and  Federmann, Christian  and  Haddow, Barry  and  Koehn, Philipp  and  Leveling, Johannes  and  Monz, Christof  and  Pecina, Pavel  and  Post, Matt  and  Saint-Amand, Herve  and  Soricut, Radu  and  Specia, Lucia  and  Tamchyna, Ale
{s}},
  title     = {Findings of the 2014 {W}orkshop on {S}tatistical {M}achine {T}ranslation},
  booktitle = {Proceedings of the Ninth Workshop on Statistical Machine Translation},
  month     = {June},
  year      = {2014},
  address   = {Baltimore, Maryland, USA},
  publisher = {Association for Computational Linguistics},
  pages     = {12--58},
  url       = {http://www.aclweb.org/anthology/W/W14/W14-3302}
}

@inproceedings{committees,
  author       = {Xiaofang Wang and
                  Dan Kondratyuk and
                  Eric Christiansen and
                  Kris M. Kitani and
                  Yair Movshovitz{-}Attias and
                  Elad Eban},
  title        = {Wisdom of {C}ommittees: {A}n {O}verlooked {A}pproach {T}o {F}aster and {M}ore {A}ccurate {M}odels},
  booktitle    = {The Tenth International Conference on Learning Representations, {ICLR}
                  2022, Virtual Event, April 25-29, 2022},
  publisher    = {OpenReview.net},
  year         = {2022},
  url          = {https://openreview.net/forum?id=MvO2t0vbs4-},
  bibsource    = {dblp computer science bibliography, https://dblp.org}
}

@article{Anderson01121954,
author = {T. W. Anderson and D. A. Darling},
title = {A {T}est of {G}oodness of {F}it},
journal = {Journal of the American Statistical Association},
volume = {49},
number = {268},
pages = {765--769},
year = {1954},
publisher = {Taylor \& Francis},
doi = {10.1080/01621459.1954.10501232},


URL = { 
    
    
        https://www.tandfonline.com/doi/abs/10.1080/01621459.1954.10501232
    

},

}

@article{ks_test,
 ISSN = {01621459, 1537274X},
 URL = {http://www.jstor.org/stable/2280095},
 author = {Frank J. Massey},
 journal = {Journal of the American Statistical Association},
 number = {253},
 pages = {68--78},
 publisher = {[American Statistical Association, Taylor & Francis, Ltd.]},
 title = {The {K}olmogorov-{S}mirnov {T}est for {G}oodness of {F}it},
 urldate = {2026-01-26},
 volume = {46},
 year = {1951}
}

@inproceedings{DBLP:conf/nips/JitkrittumGMNRK23,
  author       = {Wittawat Jitkrittum and
                  Neha Gupta and
                  Aditya Krishna Menon and
                  Harikrishna Narasimhan and
                  Ankit Singh Rawat and
                  Sanjiv Kumar},
  editor       = {Alice Oh and
                  Tristan Naumann and
                  Amir Globerson and
                  Kate Saenko and
                  Moritz Hardt and
                  Sergey Levine},
  title        = {When {D}oes {C}onfidence-{B}ased {C}ascade {D}eferral {S}uffice?},
  booktitle    = {Advances in Neural Information Processing Systems 36: Annual Conference
                  on Neural Information Processing Systems 2023, NeurIPS 2023, New Orleans,
                  LA, USA, December 10 - 16, 2023},
  year         = {2023},
  url          = {http://papers.nips.cc/paper\_files/paper/2023/hash/1f09e1ee5035a4c3fe38a5681cae5815-Abstract-Conference.html},
  bibsource    = {dblp computer science bibliography, https://dblp.org}
}

@inproceedings{lora,
  author       = {Edward J. Hu and
                  Yelong Shen and
                  Phillip Wallis and
                  Zeyuan Allen{-}Zhu and
                  Yuanzhi Li and
                  Shean Wang and
                  Lu Wang and
                  Weizhu Chen},
  title        = {Lo{RA}: {L}ow-{R}ank {A}daptation of {L}arge {L}anguage {M}odels},
  booktitle    = {The Tenth International Conference on Learning Representations, {ICLR}
                  2022, Virtual Event, April 25-29, 2022},
  year         = {2022},
  url          = {https://openreview.net/forum?id=nZeVKeeFYf9},
}

@misc{liu2026ministral3,
      title={Ministral 3}, 
      author={Alexander H. Liu and Kartik Khandelwal and Sandeep Subramanian and Victor Jouault and Abhinav Rastogi and Adrien Sadé and Alan Jeffares and Albert Jiang and Alexandre Cahill and Alexandre Gavaudan and Alexandre Sablayrolles and Amélie Héliou and Amos You and Andy Ehrenberg and Andy Lo and Anton Eliseev and Antonia Calvi and Avinash Sooriyarachchi and Baptiste Bout and Baptiste Rozière and Baudouin De Monicault and Clémence Lanfranchi and Corentin Barreau and Cyprien Courtot and Daniele Grattarola and Darius Dabert and Diego de las Casas and Elliot Chane-Sane and Faruk Ahmed and Gabrielle Berrada and Gaëtan Ecrepont and Gauthier Guinet and Georgii Novikov and Guillaume Kunsch and Guillaume Lample and Guillaume Martin and Gunshi Gupta and Jan Ludziejewski and Jason Rute and Joachim Studnia and Jonas Amar and Joséphine Delas and Josselin Somerville Roberts and Karmesh Yadav and Khyathi Chandu and Kush Jain and Laurence Aitchison and Laurent Fainsin and Léonard Blier and Lingxiao Zhao and Louis Martin and Lucile Saulnier and Luyu Gao and Maarten Buyl and Margaret Jennings and Marie Pellat and Mark Prins and Mathieu Poirée and Mathilde Guillaumin and Matthieu Dinot and Matthieu Futeral and Maxime Darrin and Maximilian Augustin and Mia Chiquier and Michel Schimpf and Nathan Grinsztajn and Neha Gupta and Nikhil Raghuraman and Olivier Bousquet and Olivier Duchenne and Patricia Wang and Patrick von Platen and Paul Jacob and Paul Wambergue and Paula Kurylowicz and Pavankumar Reddy Muddireddy and Philomène Chagniot and Pierre Stock and Pravesh Agrawal and Quentin Torroba and Romain Sauvestre and Roman Soletskyi and Rupert Menneer and Sagar Vaze and Samuel Barry and Sanchit Gandhi and Siddhant Waghjale and Siddharth Gandhi and Soham Ghosh and Srijan Mishra and Sumukh Aithal and Szymon Antoniak and Teven Le Scao and Théo Cachet and Theo Simon Sorg and Thibaut Lavril and Thiziri Nait Saada and Thomas Chabal and Thomas Foubert and Thomas Robert and Thomas Wang and Tim Lawson and Tom Bewley and Tom Bewley and Tom Edwards and Umar Jamil and Umberto Tomasini and Valeriia Nemychnikova and Van Phung and Vincent Maladière and Virgile Richard and Wassim Bouaziz and Wen-Ding Li and William Marshall and Xinghui Li and Xinyu Yang and Yassine El Ouahidi and Yihan Wang and Yunhao Tang and Zaccharie Ramzi},
      year={2026},
      eprint={2601.08584},
      archivePrefix={arXiv},
      primaryClass={cs.CL},
      url={https://arxiv.org/abs/2601.08584}, 
}

@inproceedings{DBLP:conf/emnlp/ReimersG19,
  author       = {Nils Reimers and
                  Iryna Gurevych},
  editor       = {Kentaro Inui and
                  Jing Jiang and
                  Vincent Ng and
                  Xiaojun Wan},
  title        = {Sentence-{BERT}: {S}entence {E}mbeddings using {S}iamese {B}ERT-{N}etworks},
  booktitle    = {Proceedings of the 2019 Conference on Empirical Methods in Natural
                  Language Processing and the 9th International Joint Conference on
                  Natural Language Processing, {EMNLP-IJCNLP} 2019, Hong Kong, China,
                  November 3-7, 2019},
  pages        = {3980--3990},
  publisher    = {Association for Computational Linguistics},
  year         = {2019},
  url          = {https://doi.org/10.18653/v1/D19-1410},
  doi          = {10.18653/V1/D19-1410},
  bibsource    = {dblp computer science bibliography, https://dblp.org}
}

@inproceedings{puzzle_nas,
  author       = {Akhiad Bercovich and
                  Tomer Ronen and
                  Talor Abramovich and
                  Nir Ailon and
                  Nave Assaf and
                  Mohammad Dabbah and
                  Ido Galil and
                  Amnon Geifman and
                  Yonatan Geifman and
                  Izhak Golan and
                  Netanel Haber and
                  Ehud Karpas and
                  Roi Koren and
                  Itay Levy and
                  Pavlo Molchanov and
                  Shahar Mor and
                  Zach Moshe and
                  Najeeb Nabwani and
                  Omri Puny and
                  Ran Rubin and
                  Itamar Schen and
                  Ido Shahaf and
                  Oren Tropp and
                  Omer Ullman Argov and
                  Ran Zilberstein and
                  Ran El{-}Yaniv},
  title        = {Puzzle: {D}istillation-{B}ased {NAS} for {I}nference-{O}ptimized {LLMs}},
  booktitle    = {Forty-second International Conference on Machine Learning, {ICML}
                  2025, Vancouver, BC, Canada, July 13-19, 2025},
  year         = {2025},
  url          = {https://openreview.net/forum?id=RY5MMBHRqo},
}

@inproceedings{decoding_leviathan,
  author       = {Yaniv Leviathan and
                  Matan Kalman and
                  Yossi Matias},
  editor       = {Andreas Krause and
                  Emma Brunskill and
                  Kyunghyun Cho and
                  Barbara Engelhardt and
                  Sivan Sabato and
                  Jonathan Scarlett},
  title        = {Fast {I}nference from {T}ransformers via {S}peculative {D}ecoding},
  booktitle    = {International Conference on Machine Learning, {ICML} 2023, 23-29 July
                  2023, Honolulu, Hawaii, {USA}},
  series       = {Proceedings of Machine Learning Research},
  volume       = {202},
  pages        = {19274--19286},
  publisher    = {{PMLR}},
  year         = {2023},
  url          = {https://proceedings.mlr.press/v202/leviathan23a.html},
  bibsource    = {dblp computer science bibliography, https://dblp.org}
}

@inproceedings{llm-blender,
    title = "{LLM}-{B}lender: {E}nsembling {L}arge {L}anguage {M}odels with {P}airwise {R}anking and {G}enerative Fusion",
    author = "Jiang, Dongfu  and
      Ren, Xiang  and
      Lin, Bill Yuchen",
    editor = "Rogers, Anna  and
      Boyd-Graber, Jordan  and
      Okazaki, Naoaki",
    booktitle = "Proceedings of the 61st Annual Meeting of the Association for Computational Linguistics (Volume 1: Long Papers)",
    month = jul,
    year = "2023",
    address = "Toronto, Canada",
    publisher = "Association for Computational Linguistics",
    url = "https://aclanthology.org/2023.acl-long.792/",
    doi = "10.18653/v1/2023.acl-long.792",
    pages = "14165--14178",
}

@inproceedings{arora,
  author       = {Simran Arora and
                  Avanika Narayan and
                  Mayee F. Chen and
                  Laurel J. Orr and
                  Neel Guha and
                  Kush Bhatia and
                  Ines Chami and
                  Christopher R{\'{e}}},
  title        = {Ask {M}e {A}nything: {A} simple strategy for prompting language models},
  booktitle    = {The Eleventh International Conference on Learning Representations,
                  {ICLR} 2023, Kigali, Rwanda, May 1-5, 2023},
  publisher    = {OpenReview.net},
  year         = {2023},
  url          = {https://openreview.net/forum?id=bhUPJnS2g0X},
  bibsource    = {dblp computer science bibliography, https://dblp.org}
}

@inproceedings{gupta_uncertainty,
  author       = {Neha Gupta and
                  Harikrishna Narasimhan and
                  Wittawat Jitkrittum and
                  Ankit Singh Rawat and
                  Aditya Krishna Menon and
                  Sanjiv Kumar},
  title        = {Language {M}odel {C}ascades: {T}oken-{L}evel {U}ncertainty {A}nd {B}eyond},
  booktitle    = {The Twelfth International Conference on Learning Representations,
                  {ICLR} 2024, Vienna, Austria, May 7-11, 2024},
  publisher    = {OpenReview.net},
  year         = {2024},
  url          = {https://openreview.net/forum?id=KgaBScZ4VI},
  bibsource    = {dblp computer science bibliography, https://dblp.org}
}

@article{DBLP:journals/tmlr/FerrerR25,
  author       = {Luciana Ferrer and
                  Daniel Ramos},
  title        = {Evaluating {P}osterior {P}robabilities: {D}ecision {T}heory, {P}roper {S}coring
                  {R}ules, and {C}alibration},
  journal      = {Trans. Mach. Learn. Res.},
  volume       = {2025},
  year         = {2025},
  url          = {https://openreview.net/forum?id=qbrE0LR7fF},
  bibsource    = {dblp computer science bibliography, https://dblp.org}
}

@inproceedings{toolformer,
  author       = {Timo Schick and
                  Jane Dwivedi{-}Yu and
                  Roberto Dess{\`{\i}} and
                  Roberta Raileanu and
                  Maria Lomeli and
                  Eric Hambro and
                  Luke Zettlemoyer and
                  Nicola Cancedda and
                  Thomas Scialom},
  title        = {Toolformer: {L}anguage {M}odels {C}an {T}each {T}hemselves to {U}se {T}ools},
  booktitle    = {Advances in Neural Information Processing Systems 36: Annual Conference
                  on Neural Information Processing Systems 2023, NeurIPS 2023, New Orleans,
                  LA, USA, December 10 - 16, 2023},
  year         = {2023},
  url          = {http://papers.nips.cc/paper\_files/paper/2023/hash/d842425e4bf79ba039352da0f658a906-Abstract-Conference.html},
  bibsource    = {dblp computer science bibliography, https://dblp.org}
}

@article{mistral,
  author       = {Albert Q. Jiang and
                  Alexandre Sablayrolles and
                  Arthur Mensch and
                  Chris Bamford and
                  Devendra Singh Chaplot and
                  Diego de Las Casas and
                  Florian Bressand and
                  Gianna Lengyel and
                  Guillaume Lample and
                  Lucile Saulnier and
                  L{\'{e}}lio Renard Lavaud and
                  Marie{-}Anne Lachaux and
                  Pierre Stock and
                  Teven Le Scao and
                  Thibaut Lavril and
                  Thomas Wang and
                  Timoth{\'{e}}e Lacroix and
                  William El Sayed},
  title        = {Mistral {7B}},
  journal      = {CoRR},
  volume       = {abs/2310.06825},
  year         = {2023},
  url          = {https://doi.org/10.48550/arXiv.2310.06825},
  doi          = {10.48550/ARXIV.2310.06825},
  eprinttype    = {arXiv},
  eprint       = {2310.06825},
  bibsource    = {dblp computer science bibliography, https://dblp.org}
}

@article{qwen3,
  author       = {An Yang and
                  Anfeng Li and
                  Baosong Yang and
                  Beichen Zhang and
                  Binyuan Hui and
                  Bo Zheng and
                  Bowen Yu and
                  Chang Gao and
                  Chengen Huang and
                  Chenxu Lv and
                  Chujie Zheng and
                  Dayiheng Liu and
                  Fan Zhou and
                  Fei Huang and
                  Feng Hu and
                  Hao Ge and
                  Haoran Wei and
                  Huan Lin and
                  Jialong Tang and
                  Jian Yang and
                  Jianhong Tu and
                  Jianwei Zhang and
                  Jian Yang and
                  Jiaxi Yang and
                  Jingren Zhou and
                  Junyang Lin and
                  Kai Dang and
                  Keqin Bao and
                  Kexin Yang and
                  Le Yu and
                  Lianghao Deng and
                  Mei Li and
                  Mingfeng Xue and
                  Mingze Li and
                  Pei Zhang and
                  Peng Wang and
                  Qin Zhu and
                  Rui Men and
                  Ruize Gao and
                  Shixuan Liu and
                  Shuang Luo and
                  Tianhao Li and
                  Tianyi Tang and
                  Wenbiao Yin and
                  Xingzhang Ren and
                  Xinyu Wang and
                  Xinyu Zhang and
                  Xuancheng Ren and
                  Yang Fan and
                  Yang Su and
                  Yichang Zhang and
                  Yinger Zhang and
                  Yu Wan and
                  Yuqiong Liu and
                  Zekun Wang and
                  Zeyu Cui and
                  Zhenru Zhang and
                  Zhipeng Zhou and
                  Zihan Qiu},
  title        = {Qwen3 {T}echnical {R}eport},
  journal      = {CoRR},
  volume       = {abs/2505.09388},
  year         = {2025},
  url          = {https://doi.org/10.48550/arXiv.2505.09388},
  doi          = {10.48550/ARXIV.2505.09388},
  eprinttype    = {arXiv},
  eprint       = {2505.09388},
  bibsource    = {dblp computer science bibliography, https://dblp.org}
}

@article{deepseek,
  author       = {DeepSeek{-}AI},
  title        = {DeepSeek-{R1}: {I}ncentivizing {R}easoning {C}apability in {LLM}s via {R}einforcement {L}earning},
  journal      = {CoRR},
  volume       = {abs/2501.12948},
  year         = {2025},
  url          = {https://doi.org/10.48550/arXiv.2501.12948},
  doi          = {10.48550/ARXIV.2501.12948},
  eprinttype    = {arXiv},
  eprint       = {2501.12948},
  bibsource    = {dblp computer science bibliography, https://dblp.org}
}

@article{gpt4,
  author       = {OpenAI},
  title        = {{GPT-4} {T}echnical {R}eport},
  journal      = {CoRR},
  volume       = {abs/2303.08774},
  year         = {2023},
  url          = {https://doi.org/10.48550/arXiv.2303.08774},
  doi          = {10.48550/ARXIV.2303.08774},
  eprinttype    = {arXiv},
  eprint       = {2303.08774},
  bibsource    = {dblp computer science bibliography, https://dblp.org}
}

@inproceedings{ficler-goldberg-2017-controlling,
    title = "Controlling {L}inguistic {S}tyle {A}spects in {N}eural {L}anguage {G}eneration",
    author = "Ficler, Jessica  and
      Goldberg, Yoav",
    editor = "Brooke, Julian  and
      Solorio, Thamar  and
      Koppel, Moshe",
    booktitle = "Proceedings of the Workshop on Stylistic Variation",
    month = sep,
    year = "2017",
    address = "Copenhagen, Denmark",
    publisher = "Association for Computational Linguistics",
    url = "https://aclanthology.org/W17-4912/",
    doi = "10.18653/v1/W17-4912",
    pages = "94--104",
}

@inproceedings{fan-etal-2018-hierarchical,
    title = "Hierarchical {N}eural {S}tory {G}eneration",
    author = "Fan, Angela  and
      Lewis, Mike  and
      Dauphin, Yann",
    editor = "Gurevych, Iryna  and
      Miyao, Yusuke",
    booktitle = "Proceedings of the 56th Annual Meeting of the Association for Computational Linguistics (Volume 1: Long Papers)",
    month = jul,
    year = "2018",
    address = "Melbourne, Australia",
    publisher = "Association for Computational Linguistics",
    url = "https://aclanthology.org/P18-1082/",
    doi = "10.18653/v1/P18-1082",
    pages = "889--898",
}

@inproceedings{socher-etal-2013-recursive,
    title = "Recursive {D}eep {M}odels for {S}emantic {C}ompositionality {O}ver a {S}entiment {T}reebank",
    author = "Socher, Richard  and
      Perelygin, Alex  and
      Wu, Jean  and
      Chuang, Jason  and
      Manning, Christopher D.  and
      Ng, Andrew  and
      Potts, Christopher",
    booktitle = "Proceedings of the 2013 Conference on Empirical Methods in Natural Language Processing",
    month = oct,
    year = "2013",
    address = "Seattle, Washington, USA",
    publisher = "Association for Computational Linguistics",
    url = "https://www.aclweb.org/anthology/D13-1170",
    pages = "1631--1642",
}

@inproceedings{agnews,
author = {Zhang, Xiang and Zhao, Junbo and LeCun, Yann},
title = {Character-level {C}onvolutional {N}etworks for {T}ext
{C}lassification},
year = {2015},
publisher = {MIT Press},
address = {Cambridge, MA, USA},
booktitle = {Proceedings of the 29th International Conference on Neural Information Processing Systems - Volume 1},
pages = {649–657},
numpages = {9},
location = {Montreal, Canada},
series = {NIPS'15}
}

@misc{fakenews,
  author = {Pulkit Chauhan},
  title = {Fake{N}ews {D}ataset},
  howpublished = {\url{https://huggingface.co/Pulk17/Fake-News-Detection}},
  note = {Accessed: 2026-01-26},
  year={2025}
}

@inproceedings{DBLP:conf/iclr/HoltzmanBDFC20,
  author       = {Ari Holtzman and
                  Jan Buys and
                  Li Du and
                  Maxwell Forbes and
                  Yejin Choi},
  title        = {The {C}urious {C}ase of {N}eural {T}ext {D}egeneration},
  booktitle    = {8th International Conference on Learning Representations, {ICLR} 2020,
                  Addis Ababa, Ethiopia, April 26-30, 2020},
  publisher    = {OpenReview.net},
  year         = {2020},
  url          = {https://openreview.net/forum?id=rygGQyrFvH},
  bibsource    = {dblp computer science bibliography, https://dblp.org}
}

@article{llama3,
  author       = {Llama Team},
  title        = {The {L}lama 3 {H}erd of {M}odels},
  journal      = {CoRR},
  volume       = {abs/2407.21783},
  year         = {2024},
  url          = {https://doi.org/10.48550/arXiv.2407.21783},
  doi          = {10.48550/ARXIV.2407.21783},
  eprinttype    = {arXiv},
  eprint       = {2407.21783},
  bibsource    = {dblp computer science bibliography, https://dblp.org}
}

@inproceedings{saravia-etal-2018-carer,
    title = "{CARER}: {C}ontextualized {Af}fect {R}epresentations for {E}motion {R}ecognition",
    author = "Saravia, Elvis  and
      Liu, Hsien-Chi Toby  and
      Huang, Yen-Hao  and
      Wu, Junlin  and
      Chen, Yi-Shin",
    booktitle = "Proceedings of the 2018 Conference on Empirical Methods in Natural Language Processing",
    month = oct # "-" # nov,
    year = "2018",
    address = "Brussels, Belgium",
    publisher = "Association for Computational Linguistics",
    url = "https://www.aclweb.org/anthology/D18-1404",
    doi = "10.18653/v1/D18-1404",
    pages = "3687--3697",
}

@inproceedings{rajpurkar-etal-2016-squad,
    title = "{SQ}u{AD}: 100,000+ {Q}uestions for {M}achine {C}omprehension of {T}ext",
    author = "Rajpurkar, Pranav  and
      Zhang, Jian  and
      Lopyrev, Konstantin  and
      Liang, Percy",
    editor = "Su, Jian  and
      Duh, Kevin  and
      Carreras, Xavier",
    booktitle = "Proceedings of the 2016 Conference on Empirical Methods in Natural Language Processing",
    month = nov,
    year = "2016",
    address = "Austin, Texas",
    publisher = "Association for Computational Linguistics",
    url = "https://aclanthology.org/D16-1264",
    doi = "10.18653/v1/D16-1264",
    pages = "2383--2392",
    eprint={1606.05250},
    archivePrefix={arXiv},
    primaryClass={cs.CL},
}

@article{gemma3,
  author       = {Gemma Team},
  title        = {Gemma 3 {T}echnical {R}eport},
  journal      = {CoRR},
  volume       = {abs/2503.19786},
  year         = {2025},
  url          = {https://doi.org/10.48550/arXiv.2503.19786},
  doi          = {10.48550/ARXIV.2503.19786},
  eprinttype    = {arXiv},
  eprint       = {2503.19786},
  bibsource    = {dblp computer science bibliography, https://dblp.org}
}

@inproceedings{DBLP:conf/wsdm/SakotaP024,
  author       = {Marija Sakota and
                  Maxime Peyrard and
                  Robert West},
  editor       = {Luz Angelica Caudillo{-}Mata and
                  Silvio Lattanzi and
                  Andr{\'{e}}s Mu{\~{n}}oz Medina and
                  Leman Akoglu and
                  Aristides Gionis and
                  Sergei Vassilvitskii},
  title        = {Fly-{S}wat or {C}annon? {C}ost-{E}ffective {L}anguage {M}odel {C}hoice via {M}eta-{M}odeling},
  booktitle    = {Proceedings of the 17th {ACM} International Conference on Web Search
                  and Data Mining, {WSDM} 2024, Merida, Mexico, March 4-8, 2024},
  pages        = {606--615},
  publisher    = {{ACM}},
  year         = {2024},
  url          = {https://doi.org/10.1145/3616855.3635825},
  doi          = {10.1145/3616855.3635825},
  bibsource    = {dblp computer science bibliography, https://dblp.org}
}

@inproceedings{DBLP:conf/icml/ChenZ022,
  author       = {Lingjiao Chen and
                  Matei Zaharia and
                  James Zou},
  editor       = {Kamalika Chaudhuri and
                  Stefanie Jegelka and
                  Le Song and
                  Csaba Szepesv{\'{a}}ri and
                  Gang Niu and
                  Sivan Sabato},
  title        = {Efficient {O}nline {ML} {API} {S}election for {M}ulti-{L}abel {C}lassification
                  Tasks},
  booktitle    = {International Conference on Machine Learning, {ICML} 2022, 17-23 July
                  2022, Baltimore, Maryland, {USA}},
  series       = {Proceedings of Machine Learning Research},
  volume       = {162},
  pages        = {3716--3746},
  publisher    = {{PMLR}},
  year         = {2022},
  url          = {https://proceedings.mlr.press/v162/chen22ad.html},
  bibsource    = {dblp computer science bibliography, https://dblp.org}
}

@article{Shnitzer,
  author       = {Tal Shnitzer and
                  Anthony Ou and
                  M{\'{\i}}rian Silva and
                  Kate Soule and
                  Yuekai Sun and
                  Justin Solomon and
                  Neil Thompson and
                  Mikhail Yurochkin},
  title        = {Large {L}anguage {M}odel {R}outing with {B}enchmark {D}atasets},
  journal      = {CoRR},
  volume       = {abs/2309.15789},
  year         = {2023},
  url          = {https://doi.org/10.48550/arXiv.2309.15789},
  doi          = {10.48550/ARXIV.2309.15789},
  eprinttype    = {arXiv},
  eprint       = {2309.15789},
  bibsource    = {dblp computer science bibliography, https://dblp.org}
}

@inproceedings{bert,
  author       = {Jacob Devlin and
                  Ming{-}Wei Chang and
                  Kenton Lee and
                  Kristina Toutanova},
  editor       = {Jill Burstein and
                  Christy Doran and
                  Thamar Solorio},
  title        = {{BERT:} {P}re-training of {D}eep {B}idirectional {T}ransformers for {L}anguage
                  {U}nderstanding},
  booktitle    = {Proceedings of the 2019 Conference of the North American Chapter of
                  the Association for Computational Linguistics: Human Language Technologies,
                  {NAACL-HLT} 2019, Minneapolis, MN, USA, June 2-7, 2019, Volume 1 (Long
                  and Short Papers)},
  pages        = {4171--4186},
  publisher    = {Association for Computational Linguistics},
  year         = {2019},
  url          = {https://doi.org/10.18653/v1/n19-1423},
  doi          = {10.18653/V1/N19-1423},
  bibsource    = {dblp computer science bibliography, https://dblp.org}
}

@inproceedings{frugal_ML,
  author       = {Lingjiao Chen and
                  Matei Zaharia and
                  James Y. Zou},
  editor       = {Hugo Larochelle and
                  Marc'Aurelio Ranzato and
                  Raia Hadsell and
                  Maria{-}Florina Balcan and
                  Hsuan{-}Tien Lin},
  title        = {Frugal{ML}: {H}ow to use {ML} {P}rediction {API}s more accurately and cheaply},
  booktitle    = {Advances in Neural Information Processing Systems 33: Annual Conference
                  on Neural Information Processing Systems 2020, NeurIPS 2020, December
                  6-12, 2020, virtual},
  year         = {2020},
  url          = {https://proceedings.neurips.cc/paper/2020/hash/789ba2ae4d335e8a2ad283a3f7effced-Abstract.html},
  bibsource    = {dblp computer science bibliography, https://dblp.org}
}

@inproceedings{edge_query,
  author       = {Anil Kag and
                  Igor Fedorov and
                  Aditya Gangrade and
                  Paul N. Whatmough and
                  Venkatesh Saligrama},
  title        = {Efficient {E}dge {I}nference by {S}elective {Q}uery},
  booktitle    = {The Eleventh International Conference on Learning Representations,
                  {ICLR} 2023, Kigali, Rwanda, May 1-5, 2023},
  publisher    = {OpenReview.net},
  year         = {2023},
  url          = {https://openreview.net/forum?id=jpR98ZdIm2q},
  bibsource    = {dblp computer science bibliography, https://dblp.org}
}

@inproceedings{huang-etal-2025-routereval,
    title = "{R}outer{E}val: A {C}omprehensive {B}enchmark for Routing {LLM}s to {E}xplore {M}odel-level {S}caling {U}p in {LLM}s",
    author = "Huang, Zhongzhan  and
      Ling, Guoming  and
      Lin, Yupei  and
      Chen, Yandong  and
      Zhong, Shanshan  and
      Wu, Hefeng  and
      Lin, Liang",
    editor = "Christodoulopoulos, Christos  and
      Chakraborty, Tanmoy  and
      Rose, Carolyn  and
      Peng, Violet",
    booktitle = "Findings of the Association for Computational Linguistics: EMNLP 2025",
    month = nov,
    year = "2025",
    address = "Suzhou, China",
    publisher = "Association for Computational Linguistics",
    url = "https://aclanthology.org/2025.findings-emnlp.208/",
    doi = "10.18653/v1/2025.findings-emnlp.208",
    pages = "3860--3887",
    ISBN = "979-8-89176-335-7"
}

@inproceedings{efficient_classif,
  author       = {Gao Huang and
                  Danlu Chen and
                  Tianhong Li and
                  Felix Wu and
                  Laurens van der Maaten and
                  Kilian Q. Weinberger},
  title        = {Multi-{S}cale {D}ense {N}etworks for {R}esource {E}fficient {I}mage {C}lassification},
  booktitle    = {6th International Conference on Learning Representations, {ICLR} 2018,
                  Vancouver, BC, Canada, April 30 - May 3, 2018, Conference Track Proceedings},
  publisher    = {OpenReview.net},
  year         = {2018},
  url          = {https://openreview.net/forum?id=Hk2aImxAb},
  bibsource    = {dblp computer science bibliography, https://dblp.org}
}

@inproceedings{zooter,
  author       = {Keming Lu and
                  Hongyi Yuan and
                  Runji Lin and
                  Junyang Lin and
                  Zheng Yuan and
                  Chang Zhou and
                  Jingren Zhou},
  editor       = {Kevin Duh and
                  Helena G{\'{o}}mez{-}Adorno and
                  Steven Bethard},
  title        = {Routing to the {E}xpert: {E}fficient {R}eward-guided {E}nsemble of {L}arge {L}anguage {M}odels},
  booktitle    = {Proceedings of the 2024 Conference of the North American Chapter of
                  the Association for Computational Linguistics: Human Language Technologies
                  (Volume 1: Long Papers), {NAACL} 2024, Mexico City, Mexico, June 16-21,
                  2024},
  pages        = {1964--1974},
  publisher    = {Association for Computational Linguistics},
  year         = {2024},
  url          = {https://doi.org/10.18653/v1/2024.naacl-long.109},
  doi          = {10.18653/V1/2024.NAACL-LONG.109},
  bibsource    = {dblp computer science bibliography, https://dblp.org}
}

@ARTICLE{chow1970,
  author={Chow, C.},
  journal={IEEE Transactions on Information Theory}, 
  title={On optimum recognition error and reject tradeoff}, 
  year={1970},
  volume={16},
  number={1},
  pages={41-46},
  doi={10.1109/TIT.1970.1054406}}

@article{astrom,
  author       = {{Åström, Karl Johan}},
  issn         = {{0022-247X}},
  language     = {{eng}},
  pages        = {{174--205}},
  publisher    = {{Academic Press}},
  series       = {{Journal of Mathematical Analysis and Applications}},
  title        = {Optimal {C}ontrol of {M}arkov {P}rocesses with {I}ncomplete {S}tate {I}nformation {I}},
  url          = {{https://lup.lub.lu.se/search/files/5323668/8867085.pdf}},
  doi          = {{10.1016/0022-247X(65)90154-X}},
  volume       = {{10}},
  year         = {{1965}},
}

@ARTICLE{9560049,
  author={Han, Yizeng and Huang, Gao and Song, Shiji and Yang, Le and Wang, Honghui and Wang, Yulin},
  journal={IEEE Transactions on Pattern Analysis and Machine Intelligence}, 
  title={Dynamic {N}eural {N}etworks: {A} {S}urvey}, 
  year={2022},
  volume={44},
  number={11},
  pages={7436-7456},
  doi={10.1109/TPAMI.2021.3117837}}

@article{Shrivastava,
  author       = {Vaishnavi Shrivastava and
                  Percy Liang and
                  Ananya Kumar},
  title        = {{L}lamas {K}now {W}hat {GPT}s {D}on't {S}how: {S}urrogate {M}odels for {C}onfidence
                  {E}stimation},
  journal      = {CoRR},
  volume       = {abs/2311.08877},
  year         = {2023},
  url          = {https://doi.org/10.48550/arXiv.2311.08877},
  doi          = {10.48550/ARXIV.2311.08877},
  eprinttype    = {arXiv},
  eprint       = {2311.08877},
}

@inproceedings{hybridlllm,
  author       = {Dujian Ding and
                  Ankur Mallick and
                  Chi Wang and
                  Robert Sim and
                  Subhabrata Mukherjee and
                  Victor R{\"{u}}hle and
                  Laks V. S. Lakshmanan and
                  Ahmed Hassan Awadallah},
  title        = {Hybrid {LLM:} {C}ost-{E}fficient and {Q}uality-{A}ware {Q}uery {R}outing},
  booktitle    = {The Twelfth International Conference on Learning Representations,
                  {ICLR} 2024, Vienna, Austria, May 7-11, 2024},
  publisher    = {OpenReview.net},
  year         = {2024},
  url          = {https://openreview.net/forum?id=02f3mUtqnM},
  bibsource    = {dblp computer science bibliography, https://dblp.org}
}

@inproceedings{dpo,
  author       = {Rafael Rafailov and
                  Archit Sharma and
                  Eric Mitchell and
                  Christopher D. Manning and
                  Stefano Ermon and
                  Chelsea Finn},
  editor       = {Alice Oh and
                  Tristan Naumann and
                  Amir Globerson and
                  Kate Saenko and
                  Moritz Hardt and
                  Sergey Levine},
  title        = {Direct {P}reference {O}ptimization: {Y}our {L}anguage {M}odel is {S}ecretly a {R}eward {M}odel},
  booktitle    = {Advances in Neural Information Processing Systems 36: Annual Conference
                  on Neural Information Processing Systems 2023, NeurIPS 2023, New Orleans,
                  LA, USA, December 10 - 16, 2023},
  year         = {2023},
  url          = {http://papers.nips.cc/paper\_files/paper/2023/hash/a85b405ed65c6477a4fe8302b5e06ce7-Abstract-Conference.html},
  bibsource    = {dblp computer science bibliography, https://dblp.org}
}

@inproceedings{bartscore,
  author       = {Weizhe Yuan and
                  Graham Neubig and
                  Pengfei Liu},
  editor       = {Marc'Aurelio Ranzato and
                  Alina Beygelzimer and
                  Yann N. Dauphin and
                  Percy Liang and
                  Jennifer Wortman Vaughan},
  title        = {{BARTScore}: {Evaluating} {Generated} {Text} as {Text} {Generation}},
  booktitle    = {Advances in Neural Information Processing Systems 34: Annual Conference
                  on Neural Information Processing Systems 2021, NeurIPS 2021, December
                  6-14, 2021, virtual},
  pages        = {27263--27277},
  year         = {2021},
  url          = {https://proceedings.neurips.cc/paper/2021/hash/e4d2b6e6fdeca3e60e0f1a62fee3d9dd-Abstract.html},
  bibsource    = {dblp computer science bibliography, https://dblp.org}
}

@inproceedings{automix,
author = {Aggarwal, Pranjal and Madaan, Aman and Anand, Ankit and Potharaju, Srividya Pranavi and Mishra, Swaroop and Zhou, Pei and Gupta, Aditya and Rajagopal, Dheeraj and Kappaganthu, Karthik and Yang, Yiming and Upadhyay, Shyam and Faruqui, Manaal and Mausam},
title = {Auto{M}ix: {A}utomatically {M}ixing {L}anguage {M}odels},
year = {2024},
isbn = {9798331314385},
publisher = {Curran Associates Inc.},
address = {Red Hook, NY, USA},
articleno = {4164},
numpages = {35},
location = {Vancouver, BC, Canada},
series = {NIPS '24},
booktitle = {Proceedings of the 38th International Conference on Neural Information Processing Systems},
articleno = {4164},
}

@article{Kadavath,
  author       = {Saurav Kadavath and
                  Tom Conerly and
                  Amanda Askell and
                  Tom Henighan and
                  Dawn Drain and
                  Ethan Perez and
                  Nicholas Schiefer and
                  Zac Hatfield{-}Dodds and
                  Nova DasSarma and
                  Eli Tran{-}Johnson and
                  Scott Johnston and
                  Sheer El Showk and
                  Andy Jones and
                  Nelson Elhage and
                  Tristan Hume and
                  Anna Chen and
                  Yuntao Bai and
                  Sam Bowman and
                  Stanislav Fort and
                  Deep Ganguli and
                  Danny Hernandez and
                  Josh Jacobson and
                  Jackson Kernion and
                  Shauna Kravec and
                  Liane Lovitt and
                  Kamal Ndousse and
                  Catherine Olsson and
                  Sam Ringer and
                  Dario Amodei and
                  Tom Brown and
                  Jack Clark and
                  Nicholas Joseph and
                  Ben Mann and
                  Sam McCandlish and
                  Chris Olah and
                  Jared Kaplan},
  title        = {Language {M}odels ({M}ostly) {K}now {W}hat {T}hey {K}now},
  journal      = {CoRR},
  volume       = {abs/2207.05221},
  year         = {2022},
  url          = {https://doi.org/10.48550/arXiv.2207.05221},
  eprinttype    = {arXiv},
  eprint       = {2207.05221},
}

@inproceedings{DBLP:conf/iclr/KuhnGF23,
  author       = {Lorenz Kuhn and
                  Yarin Gal and
                  Sebastian Farquhar},
  title        = {Semantic {U}ncertainty: {L}inguistic {I}nvariances for {U}ncertainty {E}stimation in {N}atural {L}anguage {G}eneration},
  booktitle    = {The Eleventh International Conference on Learning Representations,
                  {ICLR} 2023, Kigali, Rwanda, May 1-5, 2023},
  publisher    = {OpenReview.net},
  year         = {2023},
  url          = {https://openreview.net/forum?id=VD-AYtP0dve},
  bibsource    = {dblp computer science bibliography, https://dblp.org}
}

@inproceedings{lm-polygraph,
    title = "{LM}-{P}olygraph: {U}ncertainty {E}stimation for {L}anguage {M}odels",
    author = "Fadeeva, Ekaterina  and
      Vashurin, Roman  and
      Tsvigun, Akim  and
      Vazhentsev, Artem  and
      Petrakov, Sergey  and
      Fedyanin, Kirill  and
      Vasilev, Daniil  and
      Goncharova, Elizaveta  and
      Panchenko, Alexander  and
      Panov, Maxim  and
      Baldwin, Timothy  and
      Shelmanov, Artem",
    editor = "Feng, Yansong  and
      Lefever, Els",
    booktitle = "Proceedings of the 2023 Conference on Empirical Methods in Natural Language Processing: System Demonstrations",
    month = dec,
    year = "2023",
    address = "Singapore",
    publisher = "Association for Computational Linguistics",
    url = "https://aclanthology.org/2023.emnlp-demo.41/",
    doi = "10.18653/v1/2023.emnlp-demo.41",
    pages = "446--461"
}

@inproceedings{rlhf,
  author       = {Long Ouyang and
                  Jeffrey Wu and
                  Xu Jiang and
                  Diogo Almeida and
                  Carroll L. Wainwright and
                  Pamela Mishkin and
                  Chong Zhang and
                  Sandhini Agarwal and
                  Katarina Slama and
                  Alex Ray and
                  John Schulman and
                  Jacob Hilton and
                  Fraser Kelton and
                  Luke Miller and
                  Maddie Simens and
                  Amanda Askell and
                  Peter Welinder and
                  Paul F. Christiano and
                  Jan Leike and
                  Ryan Lowe},
  editor       = {Sanmi Koyejo and
                  S. Mohamed and
                  A. Agarwal and
                  Danielle Belgrave and
                  K. Cho and
                  A. Oh},
  title        = {Training language models to follow instructions with human feedback},
  booktitle    = {Advances in Neural Information Processing Systems 35: Annual Conference
                  on Neural Information Processing Systems 2022, NeurIPS 2022, New Orleans,
                  LA, USA, November 28 - December 9, 2022},
  year         = {2022},
  url          = {http://papers.nips.cc/paper\_files/paper/2022/hash/b1efde53be364a73914f58805a001731-Abstract-Conference.html},
  bibsource    = {dblp computer science bibliography, https://dblp.org}
}

@article{DBLP:journals/cj/Brent71,
  author       = {Richard P. Brent},
  title        = {An lgorithm with Guaranteed Convergence for Finding a Zero of a Function},
  journal      = {Comput. J.},
  volume       = {14},
  number       = {4},
  pages        = {422--425},
  year         = {1971},
  url          = {https://doi.org/10.1093/comjnl/14.4.422},
  doi          = {10.1093/COMJNL/14.4.422},
  bibsource    = {dblp computer science bibliography, https://dblp.org}
}

@inproceedings{murong,
  author       = {Murong Yue and
                  Jie Zhao and
                  Min Zhang and
                  Liang Du and
                  Ziyu Yao},
  title        = {Large {L}anguage {M}odel {C}ascades with {M}ixture of {T}hought {R}epresentations
                  for Cost-Efficient Reasoning},
  booktitle    = {The Twelfth International Conference on Learning Representations,
                  {ICLR} 2024, Vienna, Austria, May 7-11, 2024},
  publisher    = {OpenReview.net},
  year         = {2024},
  url          = {https://openreview.net/forum?id=6okaSfANzh},
  bibsource    = {dblp computer science bibliography, https://dblp.org}
}

@inproceedings{gpt3,
  author       = {Tom B. Brown and
                  Benjamin Mann and
                  Nick Ryder and
                  Melanie Subbiah and
                  Jared Kaplan and
                  Prafulla Dhariwal and
                  Arvind Neelakantan and
                  Pranav Shyam and
                  Girish Sastry and
                  Amanda Askell and
                  Sandhini Agarwal and
                  Ariel Herbert{-}Voss and
                  Gretchen Krueger and
                  Tom Henighan and
                  Rewon Child and
                  Aditya Ramesh and
                  Daniel M. Ziegler and
                  Jeffrey Wu and
                  Clemens Winter and
                  Christopher Hesse and
                  Mark Chen and
                  Eric Sigler and
                  Mateusz Litwin and
                  Scott Gray and
                  Benjamin Chess and
                  Jack Clark and
                  Christopher Berner and
                  Sam McCandlish and
                  Alec Radford and
                  Ilya Sutskever and
                  Dario Amodei},
  editor       = {Hugo Larochelle and
                  Marc'Aurelio Ranzato and
                  Raia Hadsell and
                  Maria{-}Florina Balcan and
                  Hsuan{-}Tien Lin},
  title        = {Language {M}odels are {F}ew-{S}hot {L}earners},
  booktitle    = {Advances in Neural Information Processing Systems 33: Annual Conference
                  on Neural Information Processing Systems 2020, NeurIPS 2020, December
                  6-12, 2020, virtual},
  year         = {2020},
  url          = {https://proceedings.neurips.cc/paper/2020/hash/1457c0d6bfcb4967418bfb8ac142f64a-Abstract.html},
  bibsource    = {dblp computer science bibliography, https://dblp.org}
}

\newpage

\onecolumn

\title{Optimal Model Activation Policies for Inference Networks of Large Language Models\\(Supplementary Material)}
\maketitle

\appendix

\section{Additional Related Work}
\paragraph{\textbf{Model chaining.}} To address the high computational cost of large language models (LLMs), recent works have explored frameworks that adaptively allocate computational resources based on input difficulty. \citep{tabi} utilize a multi-level inference engine for \textit{discriminative models} (e.g., BERT \citep{bert} etc.), using confidence scores to decide whether to return predictions from an (initial) small model or escalate queries to larger ones. It also leverages attention-based token pruning to reduce latency by removing unnecessary works and weighted ensembling to maintain accuracy. FrugalGPT \citep{frugal} extends this idea to generative modeling by designing an algorithmic framework that selects among multiple LLMs based on a learned output quality judge. It invokes $K$-out-of-$N$ models in a sequence until a satisfactory answer quality threshold is reached, achieving cost savings without significantly compromising accuracy. In an orthogonal line of work, \citep{gupta_uncertainty} investigate measurements for developing deferral rules and
develop a mechanism based on quantiles of token probabilities 
Furthermore, AutoMix \citep{automix} presents a strategy using self-verification of the small model (framed as an entailment task) and a partially observable Markov decision process (POMDP)-based \citep{astrom} router to decide whether to escalate a query to a larger model. \\
Finally, apart from generative models, there are works that implement deferral rules for predictive ML tasks \citep{committees, DBLP:conf/nips/JitkrittumGMNRK23, frugal_ML, DBLP:conf/icml/ChenZ022, 9560049, efficient_classif}. For instance, \citep{frugal_ML} develops a system with an application in classification tasks. The nature of such tasks allow the estimation of the answer quality without querying the next model each time (e.g., via classification labels), something that is more challenging when it comes to generative tasks and models.
\paragraph{\textbf{Routing-based approaches}.} While model chaining frameworks focus on 
sequentially escalating queries through models of increasing capacity, model routing approaches aim to make an in-advance assignment of each query to the most suitable model in a one-shot decision \citep{huang-etal-2025-routereval}. Recent work on model routing explores various techniques to leverage performance diversity among LLMs. HybridLLM \citep{hybridlllm} proposes a trainable quality-aware router that dynamically adaptively routes queries to either a small or large model based on predicted difficulty. \citep{zooter} frames the routing problem as reward-guided model selection, distilling supervision from off-the-shelf reward models to learn which expert LLM best handles each query. 
Similarly, \citep{DBLP:conf/wsdm/SakotaP024} proposes a meta-modeling strategy that predicts both performance and cost of candidate LLMs to optimize query assignments under budget constraints, matching the accuracy of the strongest LLM while reducing costs. Finally, \citep{Shnitzer} introduce a benchmark-driven routing formulation, using binary classifiers trained on LLMs' benchmark results to predict which model is most likely to succeed on an incoming query, thus offering a low-cost method for effective model selection without repeated inference. Overall, these routing approaches enable efficient inference by invoking only a single model per query through informed, pre-inference model selection. There are are works that consider hybrid inference in classification settings. \citep{edge_query} propose a training framework where two models (a cheap one and a more expensive one) are trained from scratch along with the router model. This setting, albeit feasible for (most) widely used classifiers, is prohibitive for LLM use-cases due to the high cost of training and fine-tuning such big foundation models.

\section{Proofs}\label{app:proofs}

We provide extended versions of the main text's proofs.

\subsection{Single-Threshold Oracle case}\label{app:st-o-proofs}


\begin{proof}[Proof of Lemma~\ref{lemma:monotonicity}]
For $0\le \theta_1<\theta_2\le 1$,
\begin{equation}
C(\theta_2)-C(\theta_1)=c_m\!\int_{\theta_1}^{\theta_2}\! f(\beta)\,d\beta \;\ge\; 0,
\end{equation}
since $f\ge 0$. Thus, $C(\cdot)$ is non-decreasing.
Similarly,
\begin{equation}
g(\theta_2)-g(\theta_1)
= -\int_{\theta_1}^{\theta_2}\varepsilon(\beta)\,f(\beta)\,d\beta \;\le\; 0,
\end{equation}
because $\varepsilon(\beta)f(\beta)\ge 0$. Hence, $g(\cdot)$ is non-increasing.
\end{proof}


\begin{proof}[Proof of Theorem~\ref{thm:st-o}]
By Lemma~\ref{lemma:monotonicity}, $g(\cdot)$ is non-increasing, hence the feasible set $\mathcal F_b:=\{\theta: g(\theta)\le b\}$ is a right ray $[\theta^*,1]$ (or empty), where $\theta^*:=\inf\{\theta\in[0,1]: g(\theta)\le b\}$.
Since $C(\cdot)$ is non-decreasing, its minimum over $\mathcal F_b$ is attained at its left endpoint, i.e., at $\theta^*$.
Furthermore, $g(\cdot)$ is continuous so the constraint is tight, i.e., $g(\theta)=b$ and if it is strictly decreasing (e.g., $\varepsilon(\theta^*)f(\theta^*)>0$), $\theta^*$ is unique with $g(\theta)=b$.
\end{proof}



\begin{remark}[When ties can happen]
Let $\theta_{\min}=\inf\{\theta\in[0,1]: g(\theta)\le b\}$ denote the smallest feasible threshold.


\textbf{Non-uniqueness (ties).}
Non-uniqueness can occur only if $g$ is \emph{flat} at the budget level on an interval starting at $\theta_{\min}$:
there exists $[a,c]\subseteq[0,1]$ with $a=\theta_{\min}$ such that
\begin{equation}
g(\theta)=b \quad \text{for all }\ \theta\in[a,c],
\end{equation}
equivalently $\varepsilon(\theta)\,f(\theta)=0$ for almost every $\theta\in[a,c]$.
In this situation:
\begin{itemize}
\item If $f(\theta)>0$ on a set of positive measure inside $[a,c]$ (e.g., $g$ is flat because $\varepsilon(\theta)=0$ on that segment),
then $C(\theta)$ is strictly increasing on $[a,c]$, and the unique cost-minimizer is the \emph{leftmost} feasible point:
\(
\theta^* \;=\; a \;=\; \theta_{\min}.
\)
\end{itemize}
\textbf{Edge case.}
If $g(0)\le b$, then $\theta^*=0$. Moreover, if there exists $\delta>0$ with $f(\theta)=0$ for all $\theta\in[0,\delta]$, then
$C(\theta)$ is constant on $[0,\delta]$ and any $\theta\in[0,\delta]$ is also optimal.
\end{remark}

\paragraph{\textbf{Optimal threshold calculation example.}}The optimal threshold $\theta^*$ can be analytically computed if $f(\beta)$ belongs into a distribution family with closed-form PDF and $\varepsilon_s(\beta)$ is known. To make things concrete, we now present how to calculate $\theta^*$ (in the oracle scenario) under the assumption that $\beta \sim \text{Uniform}(0,1)$ and that the error rate of sLLM decreases linearly with confidence, e.g., $\varepsilon_s(\beta) = 1 - \beta$\footnote{This can be generally achieved and (asymptotically) ensured through calibration with proper scoring rules, under no distributional shift \citep{DBLP:journals/tmlr/FerrerR25}. 
}
We thus have that $f(\beta) = 1$ and 
\begin{equation}
    \text{Error}_\text{ST-O}(\theta) = \int_\theta^1 (1-\beta)\,d\beta = \left[  \beta - \frac{1}{2}\beta^2  \right]_\theta^1 = \frac{1}{2} - \theta + \frac{1}{2}\theta^2 \leq 1 - \xi\,.
\end{equation}

Solving the quadratic equality (the optimal $\theta^*$ occurs exactly where the error constraint is tight) and taking the smallest solution, we get:
\begin{equation}
\label{eq:single_thres_sol}
    \theta^* = 1 \pm \sqrt{2(1 - \xi)} \xRightarrow{\theta \in [0,1]} \theta^* = 1 - \sqrt{2(1 - \xi)} \,.
\end{equation}
This result can also be generalized if $\beta$ lies in an arbitrary interval in the set $ \mathcal{I} = \left\{ [x,y] \subseteq [0,1] \mid x \leq y \right\}$ and for any distribution family that provides a closed-form PDF expression as well as different forms for $\varepsilon_s(\beta)$. 

\subsection{Single-Threshold Non-Oracle case}\label{app:st-no-proofs}

\begin{proof}[Proof of Theorem~\ref{thm:st-no}]
(i) Since $\mathcal{F}_b$ is closed (by continuity of $g_{\mathrm{NO}}$) and non-empty, $\theta^*=\min\mathcal{F}_b$ exists and satisfies $\theta\ge \theta^*$ for all $\theta\in\mathcal{F}_b$.
Because $C(\cdot)$ is non-decreasing, $C(\theta)\ge C(\theta^*)$ for all feasible $\theta$, hence $\theta^*$ is optimal. \newline
(ii) If $g_{\mathrm{NO}}(0)\le b$, then $\theta=0$ is feasible and minimizes $C$, so $\theta^*=0$. \\
(iii) If $g_{\mathrm{NO}}(0)>b$ and $g_{\mathrm{NO}}(\theta^*)<b$, continuity of $g_{\mathrm{NO}}$ implies there exists $\delta>0$ such that $g_{\mathrm{NO}}(\theta)<b$ for all $\theta\in(\theta^*-\delta,\theta^*+\delta)$, yielding a feasible $\theta<\theta^*$, contradicting minimality of $\theta^*$. Therefore $g_{\mathrm{NO}}(\theta^*)=b$.
\end{proof}

\subsection{Multi-Thresholds Oracle case}\label{app:mt-o-proof}

\begin{proof}[Proof of Lemma \ref{lemma:coordinate_monotonicity}]
For $0\le \theta_k<\theta_k' \le 1$,
\begin{equation}
C(\theta_1,\!\dots,\theta_k',\!\dots,\theta_K)-C(\boldsymbol{\theta})
=c_m p_k\!\int_{\theta_k}^{\theta_k'}\! f_k(\beta)\,d\beta \ge 0,
\end{equation}
and
\begin{equation}
g(\theta_1,\!\dots,\theta_k',\!\dots,\theta_K)-g(\boldsymbol{\theta}) = -\,p_k\!\int_{\theta_k}^{\theta_k'} \varepsilon_k(\beta)\,f_k(\beta)\,d\beta \le 0,
\end{equation}
since $f_k\ge 0$ and $\varepsilon_k\in[0,1]$. The upper-orthant property follows immediately.
\end{proof}

\begin{proof}[Proof of Theorem~\ref{thm:mt-o}]
(i) By Lemma~\ref{lemma:coordinate_monotonicity}, decreasing any coordinate strictly decreases $C(\cdot)$ and increases $g(\cdot)$.
If an optimal $\hat{\boldsymbol{\theta}}$ had slack $g(\hat{\boldsymbol{\theta}})<b$, decreasing some coordinate slightly would reduce $C(\cdot)$ while maintaining feasibility-contradiction. Hence the constraint is tight at optimum unless $\boldsymbol{\theta}=\boldsymbol{0}$ is already feasible.

(ii) Suppose $\boldsymbol{\theta}$ is feasible with $g(\boldsymbol{\theta})=b$, and there exist two interior classes $i,j$ with $\varepsilon_{s,i}(\theta_i)>\varepsilon_{s,j}(\theta_j)$. Consider small moves that keep total error unchanged: increase $\theta_i$ by $d>0$ and decrease $\theta_j$ by $s>0$ chosen so that
$
p_i\varepsilon_{s,i}(\theta_i) f_i(\theta_i)\,d \;=\; p_j\varepsilon_{s,j}(\theta_j) f_j(\theta_j)\,s
$.
The cost change is:
\begin{equation}
\Delta C = c_m\big[p_i f_i(\theta_i) d - p_j f_j(\theta_j) s\big]
= c_m p_i f_i(\theta_i) d \Big(1-\frac{\varepsilon_{s,i}(\theta_i)}{\varepsilon_{s,j}(\theta_j)}\Big) < 0\,,
\end{equation}
a contradiction to optimality. Therefore all interior coordinates must satisfy $\varepsilon_{s,k}(\theta_k^*)=t$ for a common $t$.

(iii) With equalized boundary error $t$, pick any $\theta_k^*\in \varepsilon_{s,k}^{-1}(t)$ for each active class. The tightness condition $g(\boldsymbol{\theta}^*)=b$ uniquely fixes $t$ (unless some coordinates saturate at $0$ or $1$, as described).
\end{proof}

\begin{remark}[When ties can happen]
Let $\boldsymbol{\theta}^*$ be an optimal solution obtained by the equalized-boundary-error rule
$\varepsilon_{s,k}(\theta_k^*)=t^*$ on all active classes, with $g(\boldsymbol{\theta}^*)=b$.
The optimal threshold vector need not be unique under any of the following conditions:

\noindent\textbf{Zero-density plateaus (per-class indifference).}
If for some class $k$ there exists an interval $I\subset[0,1]$ with $\theta_k^*\in I$ and
$f_k(u)=0$ for almost every $u\in I$, then for any $\theta'_k\in I$, the vector
$\boldsymbol{\theta}^\prime=(\theta_1^*,\dots, \allowbreak \theta_{k-1}^*,\theta'_k,\theta_{k+1}^*,\dots,\theta_K^*)$ satisfies
$C(\boldsymbol{\theta}')=C(\boldsymbol{\theta}^*)$ and $g(\boldsymbol{\theta}')=g(\boldsymbol{\theta}^*)=b$. 

\noindent\textbf{Cross-class flat segments at the optimal level $t^*$.}
Suppose there exist two (or more) classes $i\neq j$ with intervals $I_i,I_j\subset[0,1]$
such that $\theta_i^*\in I_i$, $\theta_j^*\in I_j$, and
$\varepsilon_{s,i}(u)=t^*$, $\varepsilon_{s,j}(u)=t^*$ for a.e.\ $u\in I_i$ and $u\in I_j$, with
$f_i,f_j>0$ on sets of positive measure in those intervals.
Then for sufficiently small $\delta_i,\delta_j>0$ one can construct nontrivial perturbations
$\tilde\theta$ by decreasing $\theta_i^*$ within $I_i$ and increasing $\theta_j^*$ within $I_j$ so that
\begin{equation}
p_i\!\int_{\theta_i^*-\delta_i}^{\theta_i^*} f_i(u)\,du
\;=\;
p_j\!\int_{\theta_j^*}^{\theta_j^*+\delta_j} f_j(u)\,du,
\end{equation}
which preserves total error (because $\varepsilon_{s,i}=\varepsilon_{s,j}=t^*$ on the moved mass) and yields
\begin{equation}
\Delta C = 0 \Rightarrow c_m\Big[-p_i\int_{\theta_i^*-\delta_i}^{\theta_i^*} f_i(u)\,du + p_j\int_{\theta_j^*}^{\theta_j^*+\delta_j}f_j(u)\,du\Big] = 0\,.
\end{equation}
\noindent\textbf{Flat budget map around $t^*$ with cost indifference.}
Let $\theta_k(t):=\inf\{\beta\in[0,1]:\varepsilon_{s,k}(\beta)\le t\}$ and
$g(t):=\sum_k p_k\int_{\theta_k(t)}^{1}\varepsilon_k f_k$.
If there exists an interval $[t_-,t_+]$ with $t_-<t_+$ such that $g(t)=b$ for all $t\in[t_-,t_+]$
and, in addition, for every $t\in[t_-,t_+]$ the deferral mass
$\sum_k p_k F_k(\theta_k(t))$ is constant, 
then all $\theta(t)$ with $t\in[t_-,t_+]$ are optimal. 

\end{remark}

\begin{remark}[Computation]
Numerically (or analytically under certain assumptions that allow closed-form solutions), we solve the single equation $g(t)=b$ (e.g., by bisection/Brent). 
Equivalently, we can form the Lagrangian that yields the stationarity rule $\varepsilon_k(\theta_k^*)=c_m/\lambda^*$ for all active $k$, so $t^*=c_m/\lambda^*$ and the same construction follows. 
\end{remark}

\paragraph{\textbf{Optimal threshold vector calculation example.}} Following the example we used in the single-threshold policy,
we will show how to analytically compute the optimal vector of thresholds $\boldsymbol{\theta}^*$ (in the oracle case) under the assumption that $\beta_k \sim \text{Uniform}(0,1),\, \forall \, k = 1,\ldots,K$ (thus having that each $f_k(\beta) = 1$ and $\int_0^{\theta_k} f_k(\beta)\,d\beta=\theta_k$) and that we use $K$ linear error functions $\varepsilon_k(\beta) = 1 - \delta_k\beta,\, \delta_k > 0$. Since all the variables of the threshold vector are coupled through the shared constraint in (\ref{eq:multi_class_obj}), we will employ the method of \textit{Lagrange multipliers} which allows the handling of multiple variables simultaneously. 

Using the Uniform distribution assumption, we can define the (relaxed) Lagrangian objective as:
\begin{equation}
    \mathcal{L}_{rel}(\theta_1, \ldots, \theta_K, \lambda) = c_s + c_m \sum_{k=1}^K p_k \theta_k + \lambda \left( \sum_{k=1}^K p_k \int_{\theta_k}^1 \varepsilon_k(\beta)\,d\beta - (1 - \xi)  \right)\,, 
\end{equation}
and by setting the derivative w.r.t $\theta_k$ to zero for optimality, we obtain the optimal thresholds:
\begin{equation}
    \frac{\partial \mathcal{L}_{rel}}{\partial \theta_k} = 0 \Rightarrow c_m p_k - \lambda p_k \varepsilon_k(\theta_k) = 0 \Rightarrow \varepsilon_k(\theta_k^*) = \frac{c_m}{\lambda} \Rightarrow \theta^*_k = \varepsilon_k^{-1}(\frac{c_m}{\lambda})\,. 
\end{equation}
Since the inverse of a linear function can be calculated, using the linear error functions assumption, we have that 
\begin{equation}
    \theta^*_k = \frac{1-\frac{c_m}{\lambda}}{\delta_k}\,.
\end{equation}
Furthermore, the constraint (i.e., the expected error) of (\ref{eq:multi_class_obj}) becomes:
\begin{equation}
    \sum_{k=1}^K p_k \int_{\theta^*_k}^1 (1-\delta_k\beta)\,d\beta = \sum_{k=1}^K \left[ 1 - \theta^*_k - \frac{1}{2}\delta_k(1-(\theta^*_k)^2) \right] = \sum_{k=1}^K \left[ 1 - \frac{1-\frac{c_m}{\lambda}}{\delta_k} - \frac{1}{2}\delta_k\left(1-\left(\frac{1-\frac{c_m}{\lambda}}{\delta_k}\right)^2\right) \right]\,,
\end{equation}
and ultimately, the constraint equation is:
\begin{equation}
    \sum_{k=1}^K \left[ 1 - \frac{1-\frac{c_m}{\lambda}}{\delta_k} - \frac{1}{2}\delta_k\left(1-\left(\frac{1-\frac{c_m}{\lambda}}{\delta_k}\right)^2\right) \right] = 1 - \xi\,.
\end{equation}
This is a non-linear equation in $\lambda$ and it cannot be solved analytically in the general case because each $\theta^*_k$ involves a rational function of $\lambda$ and the resulting equation is a sum of non-identical, non-linear rational-quadratic terms (which are not invertible in closed-form). Due to the complexity, numerical root-finding methods can be used (e.g., Brent’s method \citep{DBLP:journals/cj/Brent71}) to solve for $\lambda$. Once $\lambda^*$ is obtained, each optimal threshold is recovered using:
\begin{equation}
\label{eq:opt_thr_mult}
    \theta^*_k = \frac{1-\frac{c_m}{\lambda^*}}{\delta_k}\,.
\end{equation}
If we assume symmetry across all classes where $\delta_k = \delta$ and $p_k = \frac{1}{K},\, \forall k$ then all thresholds become equal: $\theta^*_k = \theta^*$. The constraint equation simplifies to:
\begin{equation}
    \delta (\theta^*)^2 - 2\theta^* + (-\delta + 2\xi) = 0 
\end{equation}
Solving the quadratic equation\footnote{Under $\xi \le 1$, the discriminant is non-negative, so real roots always exist.}\footnote{To ensure at least one solution in \([0,1]\), it suffices to impose
  $0 < \delta \le 2\xi$ ($\xi = 1$ is necessary for $\theta^*_+ \leq 1$). 
  If \(\xi = 1\), this reduces to \(1 \le \delta \le 2\), which ensures that \(\theta^*_-\in[0,1]\) and \(\theta^*_+ = 1\), so we can choose $\theta^*_-$.}, we get:
\begin{equation}
    \theta^* = \frac{1 \pm \sqrt{1 + \delta^2 - 2\delta\xi}}{\delta} \xRightarrow{\theta \in [0,1]} \frac{1 - \sqrt{1 + \delta^2 - 2\delta\xi}}{\delta}\,,
    \label{eq:opt_theta_multi_special}
\end{equation}
and then:
\begin{equation}
    \lambda^* = \frac{c_m}{1 - \delta\theta^*}\,.
\end{equation}
By setting $\delta = 1$ in (\ref{eq:opt_theta_multi_special}), the solution $\theta^*$ reduces to the single-threshold policy's solution in (\ref{eq:single_thres_sol}).

\subsection{Multi-Thresholds Non-Oracle case}\label{app:mt-no-proof}


\begin{proof}[Proof of Theorem \ref{thm:mt-no}]
    (i) By Lemma~\ref{lemma:coordinate_monotonicity}, decreasing any coordinate strictly decreases $C(\cdot)$ and increases $g(\cdot)$.
If an optimal $\hat{\boldsymbol{\theta}}$ had slack $g(\hat{\boldsymbol{\theta}})<b$, decreasing some coordinate slightly would reduce $C(\cdot)$ while maintaining feasibility, a contradiction. Hence the constraint is tight at optimum unless $\boldsymbol{\theta}=\boldsymbol{0}$ is already feasible. \newline
(ii) Fix two interior classes $i,j$ with $\Delta_i(\theta_i^*)>0$ and $\Delta_j(\theta_j^*)>0$.
Consider a perturbation that increases $\theta_i$ by $d>0$ and decreases $\theta_j$ by $s>0$.
Having that
\begin{equation}
\frac{\partial g_{\mathrm{NO}}}{\partial \theta^*_k}
= p_k f_k(\theta_k^*)\big(\bar{\varepsilon}_{m,k}-\varepsilon_{s,k}(\theta_k^*)\big)
= -p_k f_k(\theta_k^*)\Delta_k(\theta_k^*),
\end{equation}
we get the first-order error change:
\begin{equation}
\Delta g_{\mathrm{NO}}
= -p_i f_i(\theta_i^*)\Delta_i(\theta_i^*)\,d
\;+\;
p_j f_j(\theta_j^*)\Delta_j(\theta_j^*)\,s.
\end{equation}
Furthermore, we choose $s$ so that $\Delta g_{\mathrm{NO}}=0$, namely:
\begin{equation}
p_i f_i(\theta_i^*)\Delta_i(\theta_i^*)\,d
=
p_j f_j(\theta_j^*)\Delta_j(\theta_j^*)\,s.
\label{eq:error_balance_no}
\end{equation}
The corresponding first-order cost change is
\[
\Delta C
= c_m\big[p_i f_i(\theta_i^*)\,d - p_j f_j(\theta_j^*)\,s\big].
\]
Substituting from \eqref{eq:error_balance_no} yields
\[
\Delta C
= c_m\,p_i f_i(\theta_i^*)\,d\left(1-\frac{\Delta_i(\theta_i^*)}{\Delta_j(\theta_j^*)}\right).
\]
If $\Delta_i(\theta_i^*)>\Delta_j(\theta_j^*)$ then $\Delta C<0$, contradicting optimality of $\boldsymbol{\theta}^*$.
By symmetry, we also cannot have $\Delta_j(\theta_j^*)>\Delta_i(\theta_i^*)$, hence
$\Delta_i(\theta_i^*)=\Delta_j(\theta_j^*)$. 
This proves \eqref{eq:excess_equal_all}. \newline
(iii) By (ii), all interior classes with $\Delta_k(\theta_k^*)>0$ share a common level $t$.
Classes with $\Delta_k(\beta)\le t$ for all $\beta$ cannot satisfy $\Delta_k(\theta_k)=t$ in the interior, so they must saturate at the lowest-cost boundary $\theta_k^*=0$ (always accept). Likewise, classes with
$\Delta_k(\beta)>t$ for all $\beta$ saturate at $\theta_k^*=1$ (always defer).
All remaining classes satisfy $\theta_k^*\in \Delta_k^{-1}(t)$. 
\end{proof}

\subsection{Multi-thresholds policy justification}

As discussed in Section~\ref{subsec:multi_class_clf}, performance over different classes can vary and models can display a different behavior per class. We will briefly show when the multi-thresholds policy can be beneficial with an illustrative example. Consider the scenario in the Fig. \ref{fig:multi_justification} that describes a binary classification setting (two classes):

\begin{figure}[!ht]
  \vskip 0.2in
  \centering
    \centerline{\includegraphics[width=0.45\textwidth]{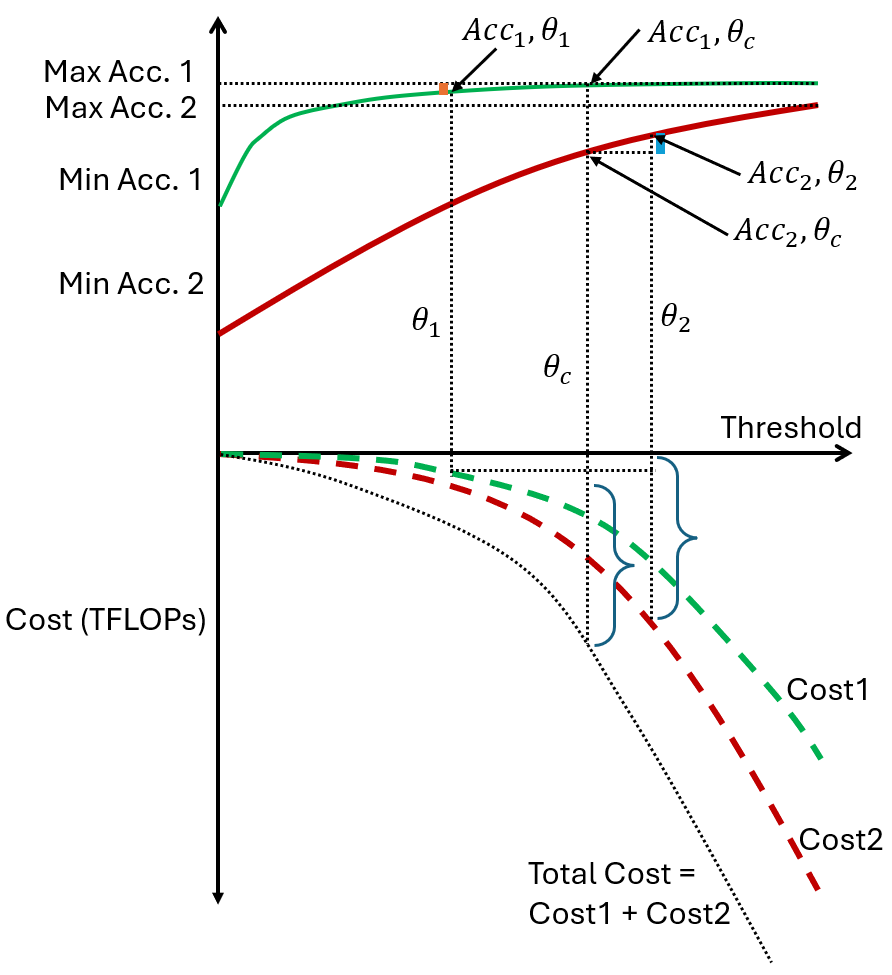}}
    \caption{
      Illustration of the accuracy–cost trade-off induced by a single threshold vs.\ multiple thresholds.
    }
    \label{fig:multi_justification}
\end{figure}

For simplicity, we assume both classes have same number of samples, so accuracy is just the average. 
The $x$-axis is the threshold value and the $y$-axis represents  accuracy (increasing upward), and cost (increasing downward). 
We assume \textsc{Max Acc. 1} to be the maximum accuracy achievable for class 1 with mLLM, while \textsc{Min Acc. 1} is the accuracy achieved with the sLLM. \textsc{Cost 1} is cost from class 1 classification with different threshold values. 

When we have one single threshold, $\theta_c$, then we achieve an accuracy of $\frac{1}{2}(\mathrm{Acc}_1, \theta_c + \mathrm{Acc}_2, \theta_c)$, but with two thresholds, we achieve higher accuracy, i.e., $\frac{1}{2}(\mathrm{Acc}_1, \theta_1 + \mathrm{Acc}_2, \theta_2)$, while keeping the same total cost.

\section{Optimality of multi-thresholds policy over generalized policies} \label{app:multi_general_policies}

We generalize the framework of the main text by allowing arbitrary (not specifically threshold-based) decision policies. We work in the same classification setting and retain the same cost and error models, and allow the mLLM to make mistakes (the oracle case follows naturally). We show that, under a mild monotone-risk assumption on the sLLM's confidence score, the Bayes-optimal \citep{} deferral policies are again \emph{per-class single-threshold policies} on a calibrated uncertainty score. We give an alternative proof based on Lagrangian relaxation and Bayes-optimal decision theory. 

\paragraph{Small model.}

For a given query, the small model outputs $(\hat{\boldsymbol{a}}^s, \beta_{\hat{\boldsymbol{a}}^s})$, where $\hat{\boldsymbol{a}}^s \in \{1,\dots,K\}$ is the predicted class and $\beta_{\hat{\boldsymbol{a}}^s} \in [0,1]$ is the scalar confidence for that class.
For each class $k$, let

\begin{equation}
p_k = \Pr(\hat{\boldsymbol{a}}^s = k), \qquad
f_k(\beta) = \text{density of }\beta_k \mid \hat{\boldsymbol{a}}^s=k\,,
\end{equation}
and define the \emph{class-wise error curve} of sLLM
\begin{equation}
\varepsilon_{s,k}(\beta) = \Pr({\boldsymbol{a}} \neq k \mid \hat{\boldsymbol{a}}^s = k, \beta_k = \beta)\,.
\end{equation}

\paragraph{Master model.}

Let $\hat{\boldsymbol{a}}^m$ denote the label predicted by mLLM when it is queried. We assume that the error of mLLM is independent of sLLM:
\begin{equation}
{\varepsilon}_{m,k}(\gamma) = \Pr({\boldsymbol{a}} \neq k \mid \hat{\boldsymbol{a}}^m =k, \gamma_k=\gamma)\,,
\end{equation}

and $\mathbb{E}[\varepsilon_{m,k}(\gamma)]$ is treated as a constant $\bar{\varepsilon}_{m,k}$. Note that we do \emph{not} assume $\bar{\varepsilon}_{m,k} \le \varepsilon_{s,k}(\beta)$ for all $\beta$; mLLM may be better or worse than sLLM depending on the region of the score space.

\paragraph{Deferral policies.}

A (possibly randomized) \emph{deferral policy} $\pi$ maps the observable
pair $(\hat{\boldsymbol{a}}^s,\beta_{\hat{\boldsymbol{a}}^s})$ to a deferral decision $D \in \{0,1\}$, where $D=0$ means ``accept sLLM'' and $D=1$ means ``defer to mLLM''.

- A deterministic policy is a measurable function
  \begin{equation}
  \pi : \{1,\dots,K\}\times[0,1] \to \{0,1\}\,.
  \end{equation}
  
- A randomized policy can be seen as a measurable function
  \begin{equation}
  \pi : \{1,\dots,K\}\times[0,1] \to [0,1]\,,
  \end{equation}
  where $\pi(k,\beta)$ is the probability of deferral when observing
  $(\hat{\boldsymbol{a}}^s=k,\beta_k=\beta)$.


\paragraph{Performance metrics.}

For a given policy $\pi$, define: \newline

- the \emph{final prediction} $\hat{\boldsymbol{a}}^\pi$ as
  \begin{equation}
  \hat{\boldsymbol{a}}^\pi =
  \begin{cases}
    \hat{\boldsymbol{a}}^s, & D=0, \\
    \hat{\boldsymbol{a}}^m, & D=1\,,
  \end{cases}
  \end{equation}
  
- the \emph{error indicator}
  \begin{equation}
  L = \mathds{1}[\hat{\boldsymbol{a}}^\pi \neq {\boldsymbol{a}}]\,,
  \end{equation}
  
- the \emph{per-sample cost}
  \begin{equation}
  C = c_s + c_m D\,.
  \end{equation}

Then the expected error and cost of policy $\pi$ are
\begin{equation}
\mathrm{Error}(\pi) = \mathbb{E}[L], \qquad
\mathrm{Cost}(\pi) = \mathbb{E}[C]\,.
\end{equation}

The optimization problem is:
\begin{equation}
  \label{eq:non-oracle-problem}
  \min_{\pi} \ \mathrm{Cost}(\pi)
  \quad\text{s.t.}\quad
  \mathrm{Error}(\pi) \le b,
\end{equation}

with $b\in[0,1]$ and the minimum is taken over all measurable policies $\pi$ based on $(\hat{\boldsymbol{a}}^s,\beta_{\hat{\boldsymbol{a}}^s})$. \newline

\noindent Before restricting attention to threshold policies, we analyze the more general
policy space using ideas from decision theory and convex analysis.

\subsection{The risk--cost region and convexity}

For any (possibly randomized) policy $\pi$, define the \emph{risk--cost vector}:
\begin{equation}
R(\pi) = \big( \mathrm{Error}(\pi), \; \mathrm{Cost}(\pi) \big) \in \mathbb{R}^2.
\end{equation}

Let $\mathcal{R}_{\mathrm{det}}$ be the set of all such pairs obtained from
deterministic policies, and let $\mathcal{R}$ be the set obtained when
randomization is allowed.

\begin{lemma}[Convexity via randomization]
  \label{lem:convexity}
  The set $\mathcal{R}$ is the convex hull of $\mathcal{R}_{\mathrm{det}}$ and
  is convex. More precisely, for any two policies $\pi_1,\pi_2$ and any
  $\alpha\in[0,1]$, define the randomized policy $\pi_\alpha$ that draws $U\sim
  \mathrm{Bernoulli}(\alpha)$ once and uses $\pi_1$ if $U=1$ and $\pi_2$ if
  $U=0$ for all samples. Then:
  $$
  R(\pi_\alpha) = \alpha R(\pi_1) + (1-\alpha) R(\pi_2).
  $$
\end{lemma}

\begin{proof}
  The equality follows from linearity of expectation:
  \begin{equation}
  \mathrm{Error}(\pi_\alpha)
  = \mathbb{E}[\mathds{1}[\hat{\boldsymbol{a}}^{\pi_\alpha}\neq {\boldsymbol{a}}]]
  = \alpha\,\mathrm{Error}(\pi_1) + (1-\alpha)\,\mathrm{Error}(\pi_2),
  \end{equation}
  and similarly for $\mathrm{Cost}(\pi_\alpha)$. Thus $\mathcal{R}$ is closed
  under convex combinations. Since deterministic policies are particular
  randomized policies, we have
  $\mathcal{R}=\mathrm{conv}(\mathcal{R}_{\mathrm{det}})$.
\end{proof}

Geometrically, $\mathcal{R}$ is a closed convex subset of $\mathbb{R}^2$. A
policy $\pi$ is (weakly) Pareto-optimal if its risk--cost pair $(e,c)$ lies on
the lower-left boundary of $\mathcal{R}$, i.e.\ if there is no other policy
$\pi'$ with $\mathrm{Err}(\pi')\le e$, $\mathrm{Cost}(\pi')\le c$, and at least
one strict inequality.

\subsection{Lagrangian formulation}


Introducing a Lagrange multiplier $\lambda\ge 0$, we define the Lagrangian of the constrained problem \eqref{eq:non-oracle-problem}:
\begin{equation}
\mathcal{L}(\pi,\lambda)
= \mathrm{Cost}(\pi) + \lambda\big(\mathrm{Error}(\pi)-b\big)\,.
\end{equation}

For fixed $\lambda$, minimizing $\mathcal{L}(\pi,\lambda)$ over $\pi$ is
equivalent to minimizing
\begin{equation}
J_\lambda(\pi) = \mathbb{E}\big[ c_s + c_m D + \lambda L \big]\,,
\end{equation}
since the term $-\lambda b$ does not depend on $\pi$. The functional $J_\lambda$ is the risk of policy $\pi$ under the loss function $\ell_\lambda$:
\begin{equation}
\ell_\lambda(D,{\boldsymbol{a}}, \hat{\boldsymbol{a}}^s, \hat{\boldsymbol{a}}^m)
=
\begin{cases}
c_s + \lambda \,\mathds{1}[\hat{\boldsymbol{a}}^s\neq {\boldsymbol{a}}] & \text{if } D=0
  \ (\text{accept}), \\
c_s + c_m + \lambda \,\mathds{1}[\hat{\boldsymbol{a}}^m\neq {\boldsymbol{a}}] & \text{if } D=1
  \ (\text{defer})\,.
\end{cases}
\end{equation}

\paragraph{Connection between the constrained and Lagrangian problems.}

By Lemma~\ref{lem:convexity}, $\mathcal{R}$ is convex. The objective in
\eqref{eq:non-oracle-problem} is to find, among all $(e,c)\in\mathcal{R}$ with
$e\le b$, a point with minimal $c$. Any such point on the lower-left boundary of $\mathcal{R}$ is Pareto-optimal.
The supporting hyperplane theorem implies that for every Pareto-optimal $(e^*,c^*)\in \mathcal{R}$ there exists some $\lambda\ge 0$ such that $(e^*,c^*)$ minimizes the linear functional
$
(e,c)\mapsto c + \lambda e
$
over $\mathcal{R}$. Translating back to policies, this means:

\begin{proposition}[Pareto points are Lagrangian minimizers]
  \label{prop:pareto-lagrange}
  Let $\pi^*$ be a policy such that $R(\pi^*)$ is Pareto-optimal in
  $\mathcal{R}$. Then there exists $\lambda\ge 0$ such that
  \[
  \pi^* \in
  \arg\min_{\pi} J_\lambda(\pi)
  = \arg\min_{\pi} \mathbb{E}[c_s + c_m D + \lambda L].
  \]
\end{proposition}

Thus, to characterize the solutions of the original constrained problem
\eqref{eq:non-oracle-problem}, it suffices to characterize the minimizers of
$J_\lambda$ for each $\lambda$.

\subsection{Bayes-optimal policy 
}

We now fix $\lambda>0$ and determine the policies that minimize $J_\lambda$. The \emph{conditional risks} of the two actions at a given pair $(\hat{\boldsymbol{a}}^s=k, \beta_k=\beta)$ are defined as:

- If we \textbf{accept} sLLM ($D=0$),
  \begin{equation}
  r_s(k,\beta) = \mathbb{E}[c_s + \lambda L \mid \hat{\boldsymbol{a}}^s=k,\beta_k=\beta, D=0]
  = c_s + \lambda \varepsilon_{s,k}(\beta)\,.
  \end{equation}

- If we \textbf{defer} to mLLM ($D=1$),
  \begin{equation}
  r_m = \mathbb{E}[c_s + c_m + \lambda L \mid  D=1]
  = c_s + c_m + \lambda \bar{\varepsilon}_{m,k}\,.
  \end{equation}

A policy $\pi$ induces a decision $D=\pi(\hat{\boldsymbol{a}}^s,\beta_{\hat{\boldsymbol{a}}^s})$.
The risk can be written as
\begin{equation}
J_\lambda(\pi)
= \mathbb{E} [ r_s(\hat{\boldsymbol{a}}^s,\beta_{\hat{\boldsymbol{a}}^s})\,\mathds{1}[D=0]
             + r_m\,\mathds{1}[D=1] ].
\end{equation}


\begin{proposition}[Bayes-optimal deferral rule]
  \label{prop:bayes-rule}
  For fixed $\lambda>0$, define the policy
  $\pi^*_\lambda$ by
  \begin{equation}
  \pi^*_\lambda(k,\beta) =
  \begin{cases}
    0 & \text{if } r_s(k,\beta)\le r_m,\\
    1 & \text{if } r_s(k,\beta) > r_m\,.
  \end{cases}
  \end{equation}
  Then $\pi^*_\lambda$ is a global minimizer of $J_\lambda$, i.e.,
  \begin{equation}
  J_\lambda(\pi^*_\lambda)
  \le J_\lambda(\pi)
  \quad \text{for all policies $\pi$}\,.
  \end{equation}
\end{proposition}


\begin{proof}
We show that $\pi^*_\lambda$ is a global minimizer of $J_\lambda$ over all policies. The argument is pointwise and uses only the conditional risks $r_s$ and $r_m$. Recall that for a given observation $(\hat{\boldsymbol{a}}^s,\beta_{\hat{\boldsymbol{a}}^s})=(k,\beta)$, the
conditional risks of the two actions are
\begin{equation}
r_s(k,\beta) = c_s + \lambda \varepsilon_{s,k}(\beta)
\quad\text{and}\quad
r_m      = c_s + c_m + \lambda \bar{\varepsilon}_{m,k}\,.
\end{equation}
Any (possibly randomized) policy $\pi$ induces a decision $D\in\{0,1\}$ at
$(k,\beta)$, where $D=0$ means ``accept'' and $D=1$ means ``defer''. The
associated conditional contribution to the Lagrangian loss is
\begin{equation}
R(k,\beta;D) = \mathbb{E}[c_s + c_m D + \lambda L \mid \hat{\boldsymbol{a}}^s=k,\beta_k=\beta, D]
= \begin{cases}
   r_s(k,\beta) & \text{if } D=0,\\
   r_m       & \text{if } D=1
  \end{cases}
\end{equation}

By definition, the Bayes policy $\pi^*_\lambda$ chooses, at each $(k,\beta)$,
\begin{equation}
D^* := \pi^*_\lambda(k,\beta)
=
\begin{cases}
0 & \text{if } r_s(k,\beta) \le r_m,\\
1 & \text{if } r_s(k,\beta) > r_m
\end{cases}
\end{equation}

Let $\pi'$ be any other policy, and let $D'=\pi(k,\beta)$ be its
decision at $(k,\beta)$. The excess conditional loss is defined as 

\begin{equation}
\Delta(k,\beta) = R(k,\beta;D'(k,\beta)) - R(k,\beta;D^*(k,\beta)).
\end{equation}

There are three exhaustive cases:

\begin{enumerate}
  \item \emph{Case 1: $\pi'$ and $\pi^*_\lambda$ choose the same action.} \\
    Then $D'(k,\beta) = D^*(k,\beta)$ and hence
    \[
    \Delta(k,\beta) = 0.
    \]

  \item \emph{Case 2: $\pi'$ accepts while $\pi^*_\lambda$ defers.} \\
    Here $D'(k,\beta)=0$ and $D^*(k,\beta)=1$. By the definition of $\pi^*_\lambda$,
    $D^*(k,\beta)=1$ implies $r_m \le r_s(k,\beta)$. Therefore, 
    \[
    \Delta(k,\beta)
    = R(k,\beta;0) - R(k,\beta;1)
    = r_S(k,\beta) - r_M(k) \;\ge 0,
    \]

  \item \emph{Case 3: $\pi'$ defers while $\pi^*_\lambda$ accepts.} \\
    Here $D'(k,\beta)=1$ and $D^*(k,\beta)=0$. By the definition of $\pi^*_\lambda$,
    $D^*(k,\beta)=0$ implies $r_s(k,\beta) \le r_m$. Therefore
    \[
    \Delta(k,\beta)
    = R(k,\beta;1) - R(k,\beta;0)
    = r_m - r_s(k,\beta) \;\ge 0,
    \]
\end{enumerate}

In all three cases we have $\Delta(k,\beta)\ge 0$, and $\Delta(k,\beta)>0$
whenever $\pi$ and $\pi^*_\lambda$ choose different actions at a point where
$r_s(k,\beta)\neq r_m$.

Integrating w.r.t.\ the joint distribution of $(\hat{\boldsymbol{a}}^s,\beta_{\hat{\boldsymbol{a}}^s})$ gives:
\begin{equation}
J_\lambda(\pi') - J_\lambda(\pi^*_\lambda)
= \mathbb{E}\big[ R(\hat{\boldsymbol{a}}^s,\beta_{\hat{\boldsymbol{a}}^s};D') - R(\hat{\boldsymbol{a}}^s,\beta_{\hat{\boldsymbol{a}}^s};D^*) \big]
= \mathbb{E}[\Delta(\hat{\boldsymbol{a}}^s,\beta_{\hat{\boldsymbol{a}}^s})] \;\ge 0.
\end{equation}

Thus no policy $\pi$ can have strictly smaller Lagrangian risk than
$\pi^*_\lambda$, and any policy with the same risk must agree with
$\pi^*_\lambda$ almost surely except possibly on the set $\{(k,\beta): r_s(k,\beta)=r_m\}$. This proves that $\pi^*_\lambda$ is a
(global) Bayes-optimal minimizer of $J_\lambda$.
\end{proof}

The decision rule of $\pi^*_\lambda$ can be written explicitly. The inequality
$r_s(k,\beta)\le r_m$ is
\begin{equation}
c_s + \lambda \varepsilon_{s,k}(\beta)
\;\le\;
c_s + c_m + \lambda \bar{\varepsilon}_{m,k}
\quad\Longleftrightarrow\quad
\varepsilon_{s,k}(\beta) - \bar{\varepsilon}_{m,k} \le \frac{c_m}{\lambda}.
\end{equation}

Thus,
\begin{equation}
\pi^*_\lambda(k,\beta) =
\begin{cases}
0 & \text{if } \varepsilon_{s,k}(\beta) - \bar{\varepsilon}_{m,k}
                \le \dfrac{c_m}{\lambda}, \\
1 & \text{otherwise}.
\end{cases}
\end{equation}

Equivalently:

\begin{equation}
  \label{eq:bayes-acceptance-region}
  \text{sLLM is accepted at $(k,\beta)$ iff}
  \quad
  \Delta_k(\beta) \le \frac{c_m}{\lambda},
\end{equation}
where $\Delta_k(\beta) = \varepsilon_{s,k}(\beta)-\bar{\varepsilon}_{m,k}$ is the
\emph{excess error} of sLLM over mLLM.

\paragraph{Interpretation.}
The Bayes-optimal policy accepts the small model at $(k,\beta)$ whenever its
excess error over the master is smaller than the cost-normalized margin
$c_m/\lambda$; it defers if sLLM is sufficiently worse than mLLM to justify
paying the additional cost $c_m$. 

\subsection{From Lagrangian minimizers back to the constrained problem}

Combining Proposition~\ref{prop:pareto-lagrange} and
Proposition~\ref{prop:bayes-rule}, we obtain:

\begin{theorem}[Lagrangian characterization of optimal policies]
  \label{thm:lagrange-characterization}
  Let $\pi^*$ be a policy whose risk--cost vector $R(\pi^*)$ is
  Pareto-optimal in $\mathcal{R}$. Then there exists $\lambda\ge 0$ such that
  $\pi^*$ coincides
  with the Bayes rule $\pi^*_\lambda$, i.e.,
  \[
  \pi^*(k,\beta) =
  \begin{cases}
   0 & \text{if } \varepsilon_{s,k}(\beta) - \bar{\varepsilon}_{m,k}
                    \le c_m/\lambda,\\
   1 & \text{otherwise}.
  \end{cases}
  \]
\end{theorem}

In particular, any optimal solution of the constrained problem
\eqref{eq:non-oracle-problem} that lies on the error--cost frontier must accept
sLLM exactly on the set
\[
A_\lambda = \{(k,\beta): \varepsilon_{s,k}(\beta) - \bar{\varepsilon}_{m,k}
                        \le c_m/\lambda\}
\]
for some $\lambda\ge 0$. 

At this stage, we have characterized the optimal acceptance region in terms
of a \emph{threshold in excess error} $\Delta_k(\beta)$. We now show how a
monotone-risk assumption allows us to represent $A_\lambda$ as a per-class
\emph{single threshold in $\beta$}.


\subsection{Monotone per-class error curves}

For each class $k$, consider the map
$
\beta \;\mapsto\; \varepsilon_{s,k}(\beta)
= \Pr(\boldsymbol{a} \neq k\mid \hat{\boldsymbol{a}}^s=k,\beta_k=\beta).
$

We say that the per-class score $\beta_k$ satisfies the \emph{monotone-risk
property} if, for each $k$, this map is (essentially) non-increasing on $[0,1]$:
\begin{equation}
  \label{eq:monotone-risk}
  \beta_1 \le \beta_2 \;\Rightarrow\;
  \varepsilon_k^s(\beta_1) \ge \varepsilon_k^s(\beta_2)
  \quad\text{for almost all $\beta_1,\beta_2$}.
\end{equation}

Equivalently, the correctness probability
\[
 \Pr(\boldsymbol{a}= k \mid \hat{\boldsymbol{a}}^s=k,\beta_k=\beta) = 1 - \varepsilon_{s,k}(\beta)
\]
is non-decreasing in $\beta$: higher $\beta$ means higher probability that the
prediction is correct, within class $k$.


\subsection{From error threshold to confidence-threshold}

The Bayes acceptance condition
\eqref{eq:bayes-acceptance-region} can be written as
\begin{equation}
\Delta_k(\beta) \le c_m/\lambda
\quad\Longleftrightarrow\quad
\varepsilon_{s,k}(\beta) \le \bar{\varepsilon}_{m,k} + \frac{c_m}{\lambda} = t_k(\lambda).
\end{equation}

For any $t\in\mathbb{R}$, define the sublevel set
\[
S_k(t) := \{\beta\in[0,1] : \varepsilon_{s,k}(\beta) \le t\}\,.
\]

Under Assumption~\eqref{eq:monotone-risk}, each $S_k(t)$ is either empty,
all of $[0,1]$, or an interval of the form $[\theta,1]$. More precisely:





\begin{lemma}[Structure of sublevel sets under monotone risk]\label{lem:sublevel-interval}
Suppose $\varepsilon_{s,k} : [0, 1] \to \mathbb{R}$ is non-increasing. For any $t \in \mathbb{R}$, the set $S_k(t)$ is either empty, $[0, 1]$, or an interval whose left endpoint is
$
\theta = \inf\, S_k(t)
$. Specifically,
$
S_k(t) = (\theta, 1]$ or $ S_k(t) = [\theta, 1]$.
\end{lemma}

\begin{proof}
If $S_k(t) = \varnothing$ or $S_k(t) = [0,1]$, the conclusion follows immediately.

Assume now that $S_k(t)$ is non-empty and not all of $[0,1]$. 
Define $\theta = \inf S_k(t)$. Since $S_k(t)$ is a non-empty subset of $[0,1]$, we have $\theta \in [0,1]$.

Let $\beta > \theta$. By definition of the infimum, there exists $\beta' \in S_k(t)$ such that $\theta \leq \beta' < \beta$. 
Because $\varepsilon_{s,k}$ is non-increasing, $\varepsilon_{s,k}(\beta) \leq \varepsilon_{s,k}(\beta') \leq t$. 
Hence $\beta \in S_k(t)$, so $(\theta, 1] \subseteq S_k(t)$.

If $\beta < \theta$, then $\beta$ is less than the infimum of $S_k(t)$, so $\beta \notin S_k(t)$.
Thus $S_k(t) \subseteq [\theta, 1]$.

Combining these inclusions, we have $S_k(t) = (\theta, 1]$ if $\theta \notin S_k(t)$, and $S_k(t) = [\theta, 1]$ if $\theta \in S_k(t)$.

If $\varepsilon_{s,k}(\cdot)$ is additionally right-continuous or continuous, then the existence of a sequence $\beta_n \downarrow \theta$ with $\varepsilon_k^s(\beta_n) \leq t$ implies $\varepsilon_{s,k}(\theta) \leq t$, so in that case $S_k(t) = [\theta, 1]$ always holds.
\end{proof}

Applying Lemma~\ref{lem:sublevel-interval} to $t_k(\lambda) = \bar{\varepsilon}_{m,k} + c_m /\lambda$ shows that, for each $k$, the acceptance set in $\beta$ is (up to null sets) an interval of the form $[\theta_k(\lambda), 1]$ or $(\theta_k(\lambda), 1]$. 
That is, there exists a scalar threshold $\theta_k(\lambda) \in [0, 1]$ such that:

\begin{equation}
\pi_\lambda^*(k, \beta) =
\begin{cases} 
0 & \beta > \theta_k(\lambda), \\
1 &  \beta < \theta_k(\lambda)
\end{cases}
\end{equation}

and the decision at \(\beta = \theta_k(\lambda)\) is either 0 or 1 without affecting the Bayes risk. 
If, in addition, \(\varepsilon_{s,k}\) is right-continuous or continuous at \(\theta_k(\lambda)\), then the acceptance set is exactly \([\theta_k(\lambda), 1]\).

Thus, under the monotone-risk assumption, the Bayes-optimal policies for the
Lagrangian problem are \emph{per-class single-threshold policies in $\beta$}.


Now that we know that optimal policies are of threshold form in $\beta$, we can
parameterize them by $\boldsymbol{\theta}\in[0,1]^K$ as in the oracle framework, and
compute the corresponding error and cost integrals.

Following the multi-thresholds policy definition in Section~\ref{subsec:multi_class_clf}  
, we obtain
\begin{align}
  \mathrm{Error}(\boldsymbol{\theta})
  &= \sum_{k=1}^K p_k\bigg[
      \int_{\theta_k}^1 \varepsilon_{s,k}(\beta) f_k(\beta)\,d\beta
    + \bar{\varepsilon}_{m,k} \int_0^{\theta_k} f_k(\beta)\,d\beta
    \bigg]\,, \label{eq:err-theta-final}\\
  \mathrm{Cost}(\boldsymbol{\theta})
  &= c_s + c_m \sum_{k=1}^K p_k \int_0^{\theta_k} f_k(\beta)\,d\beta\,.
   \label{eq:cost-theta-final}
\end{align}


\subsection{Lagrangian in terms of \texorpdfstring{$\boldsymbol{\theta}$}{}}

For $\lambda \geq 0$, define
\begin{equation}
L_\lambda(\boldsymbol{\theta})
:= \mathrm{Cost}(\boldsymbol{\theta}) +
   \lambda\big(\mathrm{Error}(\boldsymbol{\theta}) - b\big).
\end{equation}

Up to constants independent of $\theta$, 
we can write $L_\lambda(\theta)$ as
\begin{equation}
L_\lambda(\theta)
= c_s + \lambda\Big( \sum_{k=1}^K p_k \varepsilon_{M,k} - b \Big)
 + \sum_{k=1}^K p_k
    \Big[ c_m \int_0^{\theta_k} f_k(\beta)\,d\beta
         + \lambda \int_{\theta_k}^1 \Delta_k(\beta) f_k(\beta)\,d\beta
    \Big],
\end{equation}
where $\Delta_k(\beta) = \varepsilon_{s,k}(\beta)-\bar{\varepsilon}_{m,k}$.

Differentiating 
w.r.t.\ $\theta_k$ gives
\begin{align}
\frac{\partial}{\partial\theta_k}
\Big( c_m p_k\int_0^{\theta_k} f_k(\beta)\,d\beta \Big)
&= c_m p_k f_k(\theta_k)\,,\\
\frac{\partial}{\partial\theta_k}
\Big( \lambda p_k\int_{\theta_k}^1 \Delta_k(\beta) f_k(\beta)\,d\beta \Big)
&= -\lambda p_k \Delta_k(\theta_k) f_k(\theta_k)\,.
\end{align}

Assuming $p_k>0$ and $f_k(\theta_k)>0$, the stationarity
condition $\partial L_\lambda/\partial\theta_k = 0$ yields
\begin{equation}
  \label{eq:boundary-equalization-excess}
  \Delta_k(\theta_k^*)
  = \frac{c_m}{\lambda}
  \quad\text{for all } k \text{ with } 0<\theta_k^*<1\,.
\end{equation}

Thus, at any interior optimum, the \emph{excess error} of sLLM over mLLM at the
thresholds is equalized across all classes $k$:
\[
\varepsilon_{s,k}(\theta_k^*) - \bar{\varepsilon}_{m,k}
= \varepsilon_{s,k'}(\theta_{k'}^*) - \bar{\varepsilon}_{m,k'}
\quad\forall\,k,k' \text{ with } 0<\theta_k^*,\theta_{k'}^*<1.
\]

In the oracle case, $\bar{\varepsilon}_{m,k} = 0$, so
\eqref{eq:boundary-equalization-excess} reduces to equalization of the boundary
errors $\varepsilon_k(\theta_k^*)=c_m/\lambda$ as in the original framework. 

So, under the monotone-risk assumption, the per-class
threshold policy is optimal \emph{among all measurable policies that gate using $(\hat{\boldsymbol{a}}^s,\beta_{\hat{\boldsymbol{a}}^s})$}.

\subsection{Computation strategy via Lagrange}

The Bayes acceptance condition implies that for each class $k$,
the threshold $\theta_k(\lambda)$ (in the non-oracle case) is determined by
\begin{equation}
  \label{eq:theta-k-lambda-general}
  \varepsilon_{s,k}(\theta_k(\lambda)) - \bar{\varepsilon}_{m,k} = \frac{c_m}{\lambda}\,,
\end{equation}
whenever $0<\theta_k(\lambda)<1$ (with $\theta_k(\lambda)=0$ or $1$ if the
equation has no solution in $(0,1)$ and the Bayes region collapses to always
accept or always defer for that class).

If each $\varepsilon_{s,k}$ is invertible, and we can write
\begin{equation}
\theta_k(\lambda)
= (\varepsilon_{s,k})^{-1}\Big(\bar{\varepsilon}_{m,k} + \frac{c_m}{\lambda}\Big).
\end{equation}

Thus, for fixed $\lambda$, the thresholds $\theta_k(\lambda)$ are obtained by
\emph{per-class inversion} of the error curves at the common ``excess-error
level'' $\bar{\varepsilon}_{m,k}+c_m/\lambda$.

Plugging $\theta_k(\lambda)$ into \eqref{eq:err-theta-final} yields
\begin{equation}
\mathrm{Error}(\boldsymbol{\theta}(\lambda))
= \sum_{k=1}^K p_k\bigg[
    \int_{\theta_k(\lambda)}^1 \varepsilon_{s,k}(\beta) f_k(\beta)\,d\beta
  + \bar{\varepsilon}_{m,k} \int_0^{\theta_k(\lambda)} f_k(\beta)\,d\beta
  \bigg].
\end{equation}

Given an error budget $b\in[0,1]$, the target is to find $\lambda^*$ such that
\[
\mathrm{Error}(\boldsymbol{\theta}(\lambda^*)) = b.
\]

This is a scalar equation in $\lambda$. Under mild regularity conditions,
$\mathrm{Error}(\boldsymbol{\theta}(\lambda))$ is monotone in $\lambda$ (as $\lambda$
increases, the Bayes rule becomes more conservative and defers more often,
decreasing the use of sLLM in high-error regions), so this equation can be
solved by standard one-dimensional root-finding (bisection, Brent, etc.).

Once $\lambda^*$ is found, the optimal thresholds are
\[
\theta_k^* = \theta_k(\lambda^*)
= (\varepsilon_{s,k})^{-1}\Big(\bar{\varepsilon}_{m,k} + \frac{c_m}{\lambda^*}\Big).
\]

This is the direct analogue of the oracle case, where the target level is
$c_m/\lambda$ instead of $\bar{\varepsilon}_{m,k}+c_m/\lambda$.

\section{Optimality of single-threshold policy over generalized policies} \label{app:single_general_policies}

We now generalize the simpler single-threshold framework by allowing the existence arbitrary decision policies. We again work in the under the same setting and retain the same cost and error formulations as in the main text. We show that, under a mild monotone-risk assumption on the sLLM’s confidence score, the
Bayes-optimal deferral policy is a single-threshold one. The presentation will be more brief as it uses similar notions as Appendix \ref{app:multi_general_policies}

\paragraph{Small model.}

For a given query, the small model now outputs $(\hat{\boldsymbol{a}}^s, \beta$, where $\hat{\boldsymbol{a}}^s \in \mathcal{A}$ is the generated text and $\beta \in [0,1]$ is the scalar confidence for the `quality' of the generation. Let
$f(\beta)$ be the density of $\beta$ and the \emph{error curve} of sLLM:
$
\varepsilon_{s}(\beta) = \Pr({\boldsymbol{a}} \neq \hat{\boldsymbol{a}}^s \mid \beta)
$.

\paragraph{Master model.}

Let $\hat{\boldsymbol{a}}^m$ denote the generation of mLLM when it is queried. The error of mLLM is independent of sLLM:
\begin{equation}
{\varepsilon}_{m}(\gamma) = \Pr({\boldsymbol{a}} \neq \hat{\boldsymbol{a}}^m \mid \gamma)\,,
\end{equation}

and $\mathbb{E}[\varepsilon_{m}(\gamma)]$ is treated as a constant $\bar{\varepsilon}_{m}$. Note that we do \emph{not} assume $\bar{\varepsilon}_{m} \le \varepsilon_{s}(\beta)$ for all $\beta$; mLLM may be better or worse than sLLM depending on the region of the score space. \newline

\noindent Again, rather than restricting ourselves a priori to threshold policies, we now allow
any measurable policy $\pi:[0,1]\to\{0,1\}$, where
\begin{equation}
\pi(\beta) =
\begin{cases}
0 & \text{accept sLLM at score $\beta$},\\
1 & \text{defer to mLLM at score $\beta$}.
\end{cases}
\end{equation}

\paragraph{Cost and error formulations for a general policy.}

Under $\pi$, the expected cost is
\begin{align}
\mathrm{Cost}(\pi)
&= \mathbb{E}\big[c_s + c_m \pi(\beta)\big] \notag\\
&= c_s
 + c_m \int_0^1 \pi(\beta) f(\beta)\,d\beta.
\label{eq:binary-cost-pi}
\end{align}

The expected error is
\begin{align}
\mathrm{Error}(\pi)
&= \mathbb{E}\Big[
     \varepsilon_s(\beta)\,\mathbf{1}[\pi(\beta)=0]
   + \bar\varepsilon_m \,\mathbf{1}[\pi(\beta)=1]
   \Big] \notag\\
&= \int_0^1
    \Big(
      \varepsilon_s(\beta)\mathbf{1}[\pi(\beta)=0]
     +\bar\varepsilon_m \mathbf{1}[\pi(\beta)=1]
    \Big) f(\beta)\,d\beta.
\label{eq:binary-error-pi}
\end{align}

The constrained problem can thus be written as:
\begin{equation}
\min_{\pi} \mathrm{Cost}(\pi)
\quad\text{s.t.}\quad
\mathrm{Err}(\pi)\le b\,.
\end{equation}





\subsection{Lagrangian formulation}

Introducing a Lagrange multiplier $\lambda\ge 0$:
\begin{equation}
\mathcal{L}(\pi,\lambda)
= \mathrm{Cost}(\pi) + \lambda(\mathrm{Error}(\pi)-b).
\end{equation}

For fixed $\lambda$, minimizing $\mathcal{L}(\pi,\lambda)$ over $\pi$ is
equivalent to minimizing:
\begin{equation}
J_\lambda(\pi)
= \mathbb{E}\big[c_s + c_mD+ \lambda L\big]\,
\end{equation}
where $L$ is the final error indicator. That is, $J_\lambda$ is the
expected value of the loss function $\ell_\lambda$:
\begin{equation}
\ell_\lambda(D,{\boldsymbol{a}}, \hat{\boldsymbol{a}}^s, \hat{\boldsymbol{a}}^m)
=
\begin{cases}
c_s + \lambda \,\mathds{1}[\hat{\boldsymbol{a}}^s\neq {\boldsymbol{a}}] & \text{if } D=0
  \ (\text{accept}), \\
c_s + c_m + \lambda \,\mathds{1}[\hat{\boldsymbol{a}}^m\neq {\boldsymbol{a}}] & \text{if } D=1
  \ (\text{defer})\,.
\end{cases}
\end{equation}


\subsection{Bayes-optimal deferral rule in the binary case}

For a fixed score $\beta$, the conditional risks of the two actions are:

\begin{align*}
r_s(\beta) &= c_s + \lambda \varepsilon_s(\beta),
\quad\text{(accept sLLM)}\\
r_m        &= c_s + c_m + \lambda \bar\varepsilon_m,
\quad\text{(defer to mLLM)}.
\end{align*}

\begin{proposition}[Bayes-optimal deferral rule]
  For fixed $\lambda>0$, define
  \[
  \pi^*_\lambda(\beta)
  =
  \begin{cases}
   0 & \text{if } r_s(\beta) \le r_m,\\
   1 & \text{if } r_s(\beta) > r_m.
  \end{cases}
  \]
  Then $\pi^*_\lambda$ is a global minimizer of $J_\lambda(\pi)$ over all
  measurable policies $\pi$.
\end{proposition}

\begin{proof}
Let $\pi'$ be any policy and $\pi^*_\lambda$ the policy defined above. For each
$\beta$, $D'(\beta)=\pi'(\beta)$ and $D^*(\beta)=\pi^*_\lambda(\beta)$. The
associated conditional contribution of a deferral $D$ to the Lagrangian loss is
\begin{equation}
R(\beta;D) = \mathbb{E}[c_s + c_m D + \lambda L \mid \beta, D]
= \begin{cases}
   r_s(\beta) & \text{if } D=0,\\
   r_m       & \text{if } D=1
  \end{cases}
\end{equation}
Thus, the excess conditional loss is defined as: 

\begin{equation}
\Delta(\beta) = R(\beta;D'(\beta)) - R(\beta;D^*(\beta)).
\end{equation}

Since $\mathbb{E}[\ell_\lambda]$ equals $r_s(\beta)$ when $D=0$ and $r_m$ when $D=1$, we have
\begin{equation}
\Delta(\beta)
= \begin{cases}
  0                         & \text{if } D'(\beta)=D^*(\beta),\\
  r_s(\beta)-r_m \ge 0 & \text{if } D'(\beta)=0,\ D^*(\beta)=1,\\
  r_m-r_s(\beta) \ge 0 & \text{if } D'(\beta)=1,\ D^*(\beta)=0
  \end{cases}
\end{equation}
with strict inequality when $D(\beta)\neq D^*(\beta)$ and $r_s(\beta)\neq
r_m$. Thus $\Delta(\beta)\ge 0$ for all $\beta$.

Integrating w.r.t. the distribution of $\beta$ yields
\begin{equation}
J_\lambda(\pi') - J_\lambda(\pi^*_\lambda)
= \mathbb{E}[\Delta(\beta)] \ge 0\,,
\end{equation}
Hence $\pi^*_\lambda$ minimizes
$J_\lambda$ over all policies.
\end{proof}

The decision rule of $\pi^*_\lambda$ can be written explicitly:
\begin{align*}
r_s(\beta) \le r_m
&\iff c_s + \lambda \varepsilon_s(\beta)
     \le c_s + c_m + \lambda \bar\varepsilon_m \\
&\iff \varepsilon_s(\beta) - \bar\varepsilon_m \le \frac{c_m}{\lambda}\,.
\end{align*}


Then, for each $\beta$,
\begin{equation}
  \label{eq:binary-bayes-acceptance}
  \pi^*_\lambda(\beta) =
  \begin{cases}
    0 & \text{if } \varepsilon_s(\beta) - \bar\varepsilon_m \le c_m/\lambda,\\
    1 & \text{if } \varepsilon_s(\beta) - \bar\varepsilon_m > c_m/\lambda
  \end{cases}
\end{equation}

\subsection{Monotone-risk assumption and threshold structure in confidence score}

So far, $\pi^*_\lambda$ is defined as a pointwise decision in the error space
\eqref{eq:binary-bayes-acceptance}. We now identify the condition under which this Bayes rule coincides with a \emph{single-threshold policy} in $\beta$. 

\paragraph{Per-score monotone risk.}

We assume that the small-model error curve is non-increasing in the score:
\begin{equation}
  \label{eq:binary-monotone-risk}
  \beta_1 \le \beta_2 \;\Rightarrow\;
  \varepsilon_s(\beta_1) \ge \varepsilon_s(\beta_2)
  \quad\text{for almost all }\beta_1,\beta_2\in[0,1]\,.
\end{equation}
Equivalently, the correctness probability
$1-\varepsilon_s(\beta)$ is non-decreasing in $\beta$. This
monotone-risk property captures the idea that higher confidence scores should
not be more error-prone than lower ones. 

Under \eqref{eq:binary-monotone-risk}, the excess error
$\varepsilon_s(\beta)-\bar\varepsilon_m$ is also non-increasing
in $\beta$.

\paragraph{Shape of the acceptance set.}

For each $t\in\mathbb{R}$, define the sublevel set
\begin{equation}
S(t) = \{\beta\in[0,1]: \varepsilon_s(\beta) - \bar\varepsilon_m\le t\}\,.
\end{equation}

Since the excess error function is non-increasing (and $\beta$ is a continuous variable), each $S(t)$ is either empty, all of $[0,1]$,
or an interval of the form $[\theta(t),1]$More precisely, letting
\begin{equation}
\theta(t) := \inf\{\beta\in[0,1]: \varepsilon_s(\beta) - \bar\varepsilon_m\le t\}\,,
\end{equation}
we have $S(t)=[\theta(t),1]$ whenever $S(t)$ is non-empty and not all of
$[0,1]$.

For a given $\lambda$, the Bayes rule accepts on the set
\begin{equation}
S\big(c_m/\lambda\big)
= \{\beta : \varepsilon_s(\beta) - \bar\varepsilon_m\le c_m/\lambda\}\,.
\end{equation}
Thus, under monotone-risk, there exists a threshold $\theta(\lambda)$ such that
\begin{equation}
\pi^*_\lambda(\beta) =
\begin{cases}
0 & \text{if } \beta \ge \theta(\lambda),\\
1 & \text{if } \beta <   \theta(\lambda)
\end{cases}
\end{equation}
Thus, under the monotone-risk assumption on the small-model score, the Bayes-optimal policy for each $\lambda$ is a single-threshold policy in $\beta$.

\section{Continuity of the constraint function} \label{app:continuity}

We present a proof about the continuity of the constraint function of our optimization problems.

\begin{lemma}[Absolute continuity of the single-threshold oracle and non-oracle error functions]
\label{lem:absolute_continuity_g}
Let $f:[0,1]\to[0,\infty)$ be a probability density, i.e.\ $f$ is measurable and
$\int_0^1 f(\beta)\,d\beta = 1$. Let $\varepsilon_s:[0,1]\to[0,1]$ be measurable and
let $\bar\varepsilon_m\in[0,1]$ be a constant. Define the oracle and non-oracle error maps
\begin{align}
g(\theta) &= \int_{\theta}^{1}\varepsilon_s(\beta)\,f(\beta)\,d\beta\,,
\label{eq:g_def_ac}
\\
g_{\mathrm{NO}}(\theta) &= \int_{\theta}^{1}\varepsilon_s(\beta)\,f(\beta)\,d\beta
\;+\;\bar\varepsilon_m \int_{0}^{\theta} f(\beta)\,d\beta\,.
\label{eq:gno_def_ac}
\end{align}
Then $g$ and $g_{\mathrm{NO}}$ are absolutely continuous on $[0,1]$ (and hence continuous).
Moreover, both are differentiable almost everywhere and satisfy
\begin{align}
g'(\theta) &= -\,\varepsilon_s(\theta)\,f(\theta) \quad\text{for a.e. }\theta\in[0,1]\,,
\label{eq:g_deriv_ae}
\\
g_{\mathrm{NO}}'(\theta) &= -\big(\varepsilon_s(\theta)-\bar\varepsilon_m\big)\,f(\theta)
\quad\text{for a.e. }\theta\in[0,1]\,.
\label{eq:gno_deriv_ae}
\end{align}
\end{lemma}

\begin{proof}
Let $h(\beta)=\varepsilon_s(\beta)f(\beta)$. Since $0\le \varepsilon_s\le 1$ and $f\in L^1([0,1])$
with $\int_0^1 f=1$, we have $0\le h(\beta)\le f(\beta)$ and therefore $h\in L^1([0,1])$.

\smallskip
\noindent\textbf{Tail integrals of $L^1$ functions are absolutely continuous.}
Let $H(\theta)=\int_\theta^1 h(\beta)\,d\beta$. For any finite collection of pairwise disjoint intervals
$\{(a_i,b_i)\}_{i=1}^n\subset[0,1]$:
\begin{equation}
H(b_i)-H(a_i)=\int_{b_i}^1 h(\beta)\,d\beta-\int_{a_i}^1 h(\beta)\,d\beta
= -\int_{a_i}^{b_i} h(\beta)\,d\beta\,,
\end{equation}
hence
\begin{equation}
|H(b_i)-H(a_i)|\le \int_{a_i}^{b_i}|h(\beta)|\,d\beta\,.
\end{equation}
Summing and using disjointness yields
\begin{equation}
\sum_{i=1}^n |H(b_i)-H(a_i)| \;\le\; \int_{\cup_i (a_i,b_i)} |h(\beta)|\,d\beta.
\end{equation}
Because $h\in L^1([0,1])$, the integral $\int_E |h|$ can be made arbitrarily small whenever
the Lebesgue measure $|E|$ is sufficiently small (absolute continuity of the Lebesgue integral).
Therefore, for every $\eta>0$ there exists $\delta>0$ such that
$\sum_i(b_i-a_i)<\delta$ implies $\sum_i |H(b_i)-H(a_i)|<\eta$.
This is precisely the definition of absolute continuity of $H$.
Thus $H$ is absolutely continuous on $[0,1]$. Applying this with $H=g$ and $h=\varepsilon_s f$ proves that $g$ is absolutely continuous.

\noindent\textbf{Absolute continuity of $g_{\mathrm{NO}}$.}
Let $F(\theta)=\int_0^\theta f(\beta)\,d\beta$. Since $f\in L^1([0,1])$, the same argument as above
(with $h=f$) shows that $F$ is absolutely continuous. Hence $g_{\mathrm{NO}}(\theta)=g(\theta)+\bar\varepsilon_m F(\theta)$
is a sum of absolutely continuous functions and is therefore absolutely continuous on $[0,1]$.
\end{proof}

\noindent\textbf{Absolute continuity of the multi-thresholds oracle and non-oracle error functions.}
Let $K\in\mathbb{N}$. For each class $k\in\{1,\dots,K\}$, let
$f_k:[0,1]\to[0,\infty)$ be measurable with $\int_0^1 f_k(\beta)\,d\beta = 1$,
let $\varepsilon_{s,k}:[0,1]\to[0,1]$ be measurable, and let $\bar{\varepsilon}_{m,k}\in[0,1]$ be a constant.
Let $p_k\ge 0$ with $\sum_{k=1}^K p_k = 1$. 
The oracle and non-oracle class-wise error contributions are:
\begin{align}
g_k(\theta_k)
&:= p_k\int_{\theta_k}^{1}\varepsilon_{s,k}(\beta)\,f_k(\beta)\,d\beta\,,
\label{eq:gk_oracle}
\\
g_{\mathrm{NO},k}(\theta_k)
&:= p_k\left(
\int_{\theta_k}^{1}\varepsilon_{s,k}(\beta)\,f_k(\beta)\,d\beta
+\bar{\varepsilon}_{m,k}\int_{0}^{\theta_k} f_k(\beta)\,d\beta
\right)\,,
\label{eq:gk_no}
\end{align}
and the total multiclass error maps on $[0,1]^K$:
\begin{equation}
g(\boldsymbol{\theta}) = \sum_{k=1}^K g_k(\theta_k),
\qquad
g_{\mathrm{NO}}(\boldsymbol{\theta}) = \sum_{k=1}^K g_{\mathrm{NO},k}(\theta_k),
\quad \boldsymbol{\theta}=(\theta_1,\dots,\theta_K)\,.
\label{eq:g_multisum}
\end{equation}
Then, $g$ and $g_{\mathrm{NO}}$ are absolutely continuous in each coordinate separately; in particular, they are continuous on $[0,1]^K$. 
For a class $k$, Since $0\le \varepsilon_{s,k}\le 1$ and $f_k\in L^1([0,1])$ with $\int_0^1 f_k=1$,
the product $h_k(\beta)=\varepsilon_{s,k}(\beta)f_k(\beta)$ belongs to $L^1([0,1])$.
By Lemma~\ref{lem:absolute_continuity_g} applied to $h_k$, the map
$\theta_k\mapsto \int_{\theta_k}^1 h_k(\beta)\,d\beta$ is absolutely continuous (hence continuous) on $[0,1]$,
and its derivative equals $-h_k(\theta_k)$ for almost every $\theta_k$.
Multiplying by the constant $p_k$ gives the absolute continuity of $g_k$. 
Similarly, $\int_0^{\theta_k} f_k(\beta)\,d\beta$ is absolutely continuous with derivative
$f_k(\theta_k)$ for almost every $\theta_k$. Therefore $g_{\mathrm{NO},k}$ is a sum of absolutely continuous
functions, hence absolutely continuous. 
Finally, $g(\boldsymbol{\theta})$ and $g_{\mathrm{NO}}(\boldsymbol{\theta})$ are finite sums of continuous functions of separate coordinates, hence continuous on $[0,1]^K$. The stated partial derivatives hold coordinate-wise almost everywhere.

\section{Details about the experimental setup}\label{app:experiments_additional}

In this section, we present additional information about the datasets, the algorithms that were used as well as supplementary results.

\subsection{Datasets}\label{app:datasets}

We experimented with five NLP datasets covering both discriminative and generative tasks. For the former, we focused on classification tasks, using \textsc{SST-2} \citep{socher-etal-2013-recursive} and \textsc{FakeNews} \citep{fakenews} for binary classification, and \textsc{AGNews} \citep{agnews} and \textsc{Emotion} \citep{saravia-etal-2018-carer} for multi-class classification. For the latter, we used the \textsc{SQuAD} question-answering dataset \citep{rajpurkar-etal-2016-squad}.
\begin{enumerate}
    \item[(i)] \textsc{SST-2}: The Stanford Sentiment Treebank is a dataset of movie review sentences with human-annotated sentiment labels. The task involves predicting sentence-level sentiment under a binary (positive/negative) classification setting.
    \item[(ii)] \textsc{FakeNews}: Contains a collection of English-language news articles that have been manually labeled according to their veracity. It is designed for binary text classification tasks and includes full article text paired with corresponding ground-truth labels indicating whether the content is misleading (fake) or factual (real). The dataset covers a variety of news topics and is commonly used for research in automated fake news detection.
    \item [(iii)] \textsc{Emotion}: Consists of English-language Twitter messages, each annotated with one of six basic emotion categories: anger, fear, joy, love, sadness, and surprise. It is designed for emotion classification tasks, where the goal is to predict the dominant emotional label of a given tweet based on its text. 
    \item[(iv)] \textsc{AGNews}: Consists of news article texts paired with one of four topic labels, World, Sports, Business, and Science/Technology, representing major categories of news content. Designed for multi-class text classification, the task is to predict the correct topic label given the content of a news article.
    \item[(v)] \textsc{SQuAD}: The Stanford Question Answering Dataset (SQuAD) is a large-scale reading comprehension corpus built from English Wikipedia articles. It consists of questions written by crowdworkers paired with context passages, where the task is to extract the correct answer as a contiguous span of text from the passage. Given our use of LLMs, we adopt a free-form answer format, enabling responses that are not limited to fixed labels or spans, like in traditional QA.
    \item[(vi)] \textsc{WMT}: The Workshop on Machine Translation (WMT) dataset is a large-scale benchmark for machine translation, constructed from parallel corpora collected from diverse sources such as news articles, parliamentary proceedings, and web crawls. It consists of sentence-aligned text pairs in multiple language combinations (e.g., English–German, English–French, English–Russian), where the task is to translate a source-language sentence into the corresponding target language.
\end{enumerate}

\begin{table}[h]
  \caption{Summary of datasets and corresponding tasks.}
  \label{tab:datasets}
  \begin{center}
      \begin{adjustbox}{max width=\textwidth}
        \begin{tabular}{lc}
          \toprule
          \bf Dataset & \bf Task \\
          \midrule
            \textsc{SST-2} & 2-way classification \\
            \textsc{FakeNews} & 2-way classification \\
            \textsc{Emotion} & 6-way classification \\
            \textsc{AGNews} & 4-way classification \\
            \textsc{SQuAD} & Question-Answering \\
            \textsc{WMT} & Machine-Translation \\
          \bottomrule
        \end{tabular}
        \end{adjustbox}
  \end{center}
  \vskip -0.1in
\end{table}

\subsection{FrugalGPT and HybridLLM details}\label{app:frugal_hybrid_details}
We used DistilBERT as the encoder for the auxiliary routing model. The router was trained as a binary classifier, and we constructed its training datasets following the descriptions provided in FrugalGPT \citep{frugal} and Hybrid-LLM \citep{hybridlllm} papers. \\
For the HybridLLM setting, we assigned a label of `1' when sLLM produced an answer that was closer to the ground-truth answer than mLLM. Closeness was measured using BARTScore \citep{bartscore}, computed between the ground-truth answer and each model’s generated answer. Specifically, a label of `1' was assigned if the BARTScore between the ground-truth answer and the sLLM output was higher than that between the ground-truth answer and the mLLM output; otherwise, a label of `0' was assigned. For the FrugalGPT setting, we assigned a label of `1' if the ground-truth answer matched the sLLM’s generated answer, and `0' otherwise.\\
For both methods, the router model was trained for up to 20 epochs with early stopping (patience = 3 epochs) and an initial learning rate of $\eta = 2 \times 10^{-5}$.

\subsection{Additional results}\label{app:additional_results}

We present additional results with different model pairs in Tables \ref{tab:add_res1}, \ref{tab:add_res2}, \ref{tab:add_res3}, \ref{tab:add_res4}, \ref{tab:add_res5}, 
\ref{tab:add_res7}, \ref{tab:add_res8}, \ref{tab:add_res9}, \ref{tab:add_res10}, 
\ref{tab:add_res11}, 
\ref{tab:add_res13}, 
\ref{tab:add_res14}, \ref{tab:add_res15}, and \ref{tab:add_res16}.

\begin{table}[h!]
  \caption{Performance comparison of our policy with other policies on classification tasks using the \textsc{gemma3-1b} (sLLM) and \textsc{gemma3-12b} (mLLM) models pair.}
  \label{tab:add_res1}
  \centering
      \begin{adjustbox}{max width=\columnwidth}
        \begin{tabular}{cccc}
          \toprule
          \bf Dataset & \bf Policy & \bf Performance ($\bf Accuracy$ / $\bf F_1$) & \bf Cost Saved ($\%$) \\
          \midrule 
           \multirow{4}{*}{{\sc SST-2}} & None Deferred & 0.84 / 0.84 & 99.32$\%$ \\
           &  All Deferred & 0.92 / 0.91 & 0$\%$\\
           &  Random & 0.88 / 0.88 & 48.51$\%$\\
           & FrugalGPT & 0.91 / 0.91 & 28.32$\%$\\
           & HybridLLM & 0.89 / 0.89 & 32.64$\%$\\
           &  Ours & 0.92 / 0.91 & 27.05$\%$\\
           \midrule
           \multirow{4}{*}{{\sc AGNews}} & None Deferred & 0.56 / 0.49 & 99.19$\%$ \\
           &  All Deferred & 0.84 / 0.82 & 0$\%$\\
           &  Random & 0.71 / 0.68 & 48.38$\%$\\
           & FrugalGPT & 0.77 / 0.75 & 6.44$\%$\\
           & HybridLLM & 0.81 / 0.81 & 15.94$\%$\\
           &  Ours & 0.84 / 0.82 & 6.73$\%$\\
           \midrule
           \multirow{4}{*}{{\sc Emotion}} & None Deferred & 0.46 / 0.36 & 99.61$\%$ \\
           &  All Deferred & 0.58 / 0.48 & 0$\%$\\
           &  Random & 0.53 / 0.45 & 48.79$\%$\\
           & FrugalGPT & 0.51 / 0.45 & 44.17$\%$\\
           & HybridLLM & 0.51 / 0.40 & 79.17$\%$\\
           &  Ours & 0.59 / 0.49 & 44.23$\%$\\
           \midrule
           \multirow{4}{*}{{\sc FakeNews}} & None Deferred & 0.52 / 0.40 & 97.26$\%$ \\
           &  All Deferred & 0.83 / 0.82 & 0$\%$\\
           &  Random & 0.69 / 0.68 & 46.64$\%$\\
           & FrugalGPT & 0.86 / 0.84 & 41.51$\%$\\
           & HybridLLM & 0.77 / 0.76 & 69.26$\%$\\
           &  Ours & 0.82 / 0.81 & 43.26$\%$\\
          \bottomrule
        \end{tabular}
        \end{adjustbox}
  \vskip -0.1in
\end{table}

\begin{table}[h!]
  \caption{Performance comparison of our policy with other policies on classification tasks using the \textsc{gemma3-4b} (sLLM) and \textsc{gemma3-12b} (mLLM) models pair. In the case of \textsc{fakenews}, \textsc{gemma3-4b} was better than \textsc{gemma3-12b} so the policy did not defer any sample.}
  \label{tab:add_res2}
  \centering
      \begin{adjustbox}{max width=\columnwidth}
        \begin{tabular}{cccc}
          \toprule
          \bf Dataset & \bf Policy & \bf Performance ($\bf Accuracy$ / $\bf F_1$) & \bf Cost Saved ($\bf \%$) \\
          \midrule 
           \multirow{4}{*}{{\sc SST-2}} & None Deferred & 0.90 / 0.90 & 98.69$\%$ \\
           &  All Deferred & 0.92 / 0.92 & 0$\%$\\
           &  Random & 0.90 / 0.90 & 26.20$\%$\\
           &  FrugalGPT & 0.89 / 0.87 & 43.26$\%$\\
           &  HybridLLM & 0.90 / 0.89 & 16.95$\%$\\
           &  Ours & 0.91 / 0.91 & 41.70$\%$\\
           \midrule
           \multirow{4}{*}{{\sc agnews}} & None Deferred & 0.78 / 0.74 & 75.27$\%$ \\
           &  All Deferred & 0.83 / 0.82 & 0$\%$\\
           &  Random & 0.81 / 0.78 & 24.46$\%$\\
           &  FrugalGPT & 0.83 / 0.81 & 36.02$\%$\\
           &  HybridLLM & 0.81 / 0.80 & 38.58$\%$\\
           &  Ours & 0.84 / 0.83 & 33.96$\%$\\
           \midrule
           \multirow{4}{*}{{\sc emotion}} & None Deferred & 0.56 / 0.46 & 80.72$\%$ \\
           &  All Deferred & 0.58 / 0.48 & 0$\%$\\
           &  Random & 0.57 / 0.48 & 29.90$\%$\\
           &  FrugalGPT & 0.56 / 0.46 & 77.65$\%$\\
           &  HybridLLM & 0.56 / 0.46 & 80.72$\%$\\
           &  Ours & 0.58 / 0.48 & 73.09$\%$\\
           \midrule
           \multirow{4}{*}{{\sc fakenews}} & None Deferred & 0.90 / 0.88 & 55.20$\%$ \\
           &  All Deferred & 0.83 / 0.82 & 0$\%$\\
           &  Random & 0.88 / 0.88 & 4.58$\%$\\
           &  FrugalGPT & 0.90 / 0.90 & 55.20$\%$\\
           &  HybridLLM & 0.90 / 0.90 & 55.20$\%$\\
           &  Ours & 0.90 / 0.90 & 55.20$\%$\\
          \bottomrule
        \end{tabular}
        \end{adjustbox}
  \vskip -0.1in
\end{table}

\begin{table}[h!]
  \caption{Performance comparison of our policy with other policies on classification tasks using the \textsc{qwen3-4b} (sLLM) and \textsc{gemma3-12b} (mLLM) models pair.}
  \label{tab:add_res3}
  \centering
      \begin{adjustbox}{max width=\columnwidth}
        \begin{tabular}{cccc}
          \toprule
          \bf Dataset & \bf Policy & \bf Performance ($\bf Accuracy$ / $\bf F_1$) & \bf Cost Saved ($\bf \%$) \\
          \midrule 
           \multirow{4}{*}{{\sc SST-2}} & None Deferred & 0.86 / 0.86 & 74.73$\%$ \\
           &  All Deferred & 0.92 / 0.92 & 0$\%$\\
           &  Random & 0.88 / 0.88 & 23.92$\%$\\
           &  FrugalGPT & 0.89 / 0.89 & 59.80$\%$\\
           &  HybridLLM & 0.91 / 0.91 & 2.17$\%$\\
           &  Ours & 0.92 / 0.92 & 58.67$\%$\\
           \midrule
           \multirow{4}{*}{{\sc agnews}} & None Deferred & 0.82 / 0.81 & 73.32$\%$ \\
           &  All Deferred & 0.83 / 0.82 & 0$\%$\\
           &  Random & 0.83 / 0.82 & 22.50$\%$\\
           &  FrugalGPT & 0.83 / 0.82 & 70.07$\%$\\
           &  HybridLLM & 0.82 / 0.81 & 72.63$\%$\\
           &  Ours &  0.84 / 0.83 & 70.44$\%$\\
           \midrule
           \multirow{4}{*}{{\sc emotion}} & None Deferred & 0.55 / 0.49 & 78.69$\%$ \\
           &  All Deferred & 0.58 / 0.48 & 0$\%$\\
           &  Random & 0.57 / 0.49 & 27.87$\%$\\
           &  FrugalGPT & 0.55 / 0.50 & 72.87$\%$\\
           &  HybridLLM & 0.57 / 0.50 & 49.25$\%$\\
           &  Ours & 0.58 / 0.51 & 74.06$\%$\\
           \midrule
           \multirow{4}{*}{{\sc fakenews}} & None Deferred & 0.83 / 0.82 & 54.49$\%$ \\
           &  All Deferred & 0.83 / 0.82 & 0$\%$\\
           &  Random & 0.83 / 0.83 & 3.87$\%$\\
           &  FrugalGPT & 0.83 / 0.82 & 54.49$\%$\\
           &  HybridLLM & 0.84 / 0.82 & 52.37$\%$\\
           &  Ours & 0.83 / 0.82 & 54.49$\%$\\
          \bottomrule
        \end{tabular}
        \end{adjustbox}
  \vskip -0.1in
\end{table}

\begin{table}[h!]
  \caption{Performance comparison of our policy with other policies on classification tasks using the \textsc{gemma3-1b} (sLLM) and \textsc{qwen3-4b} (mLLM) models pair.}
  \label{tab:add_res4}
  \centering
      \begin{adjustbox}{max width=\columnwidth}
        \begin{tabular}{cccc}
          \toprule
          \bf Dataset & \bf Policy & \bf Performance ($\bf Accuracy$ / $\bf F_1$) & \bf Cost Saved ($\bf \%$) \\
          \midrule 
           \multirow{4}{*}{{\sc SST-2}} & None Deferred & 0.84 / 0.84 & 98.49$\%$ \\
           &  All Deferred & 0.86 / 0.86 & 0$\%$\\
           &  Random & 0.85 / 0.84 &  47.68$\%$\\
           &  FrugalGPT & 0.85 / 0.82 & 93.24$\%$\\
           &  HybridLLM & 0.84 / 0.82 & 98.24$\%$\\
           &  Ours & 0.88 / 0.88 & 89.31$\%$\\
           \midrule
           \multirow{4}{*}{{\sc agnews}} & None Deferred & 0.56 / 0.49 & 98.29$\%$ \\
           &  All Deferred & 0.82 / 0.81 & 0$\%$\\
           &  Random &  0.70 / 0.67 & 47.48$\%$\\
           &  FrugalGPT & 0.63 / 0.61 & 25.92$\%$\\
           &  HybridLLM & 0.81 / 0.80 & 44.23$\%$\\
           &  Ours &  0.83 / 0.81 &  32.79$\%$\\
           \midrule
           \multirow{4}{*}{{\sc emotion}} & None Deferred & 0.46 / 0.36 & 98.96$\%$ \\
           &  All Deferred & 0.55 / 0.49 & 0$\%$\\
           &  Random & 0.51 / 0.44 & 48.15$\%$\\
           &  FrugalGPT & 0.48 / 0.42 & 65.90$\%$\\
           &  HybridLLM & 0.48 / 0.38 & 92.77$\%$\\
           &  Ours &  0.56 / 0.48 & 68.77$\%$\\
           \midrule
           \multirow{4}{*}{{\sc fakenews}} & None Deferred & 0.52 / 0.40 & 96.60$\%$ \\
           &  All Deferred & 0.82 / 0.82 & 0$\%$\\
           &  Random & 0.68 / 0.66 & 45.97$\%$\\
           &  FrugalGPT & 0.85 / 0.83 & 39.85$\%$\\
           &  HybridLLM & 0.75 / 0.74 & 63.60$\%$\\
           &  Ours & 0.78 / 0.77 & 41.67$\%$\\
          \bottomrule
        \end{tabular}
        \end{adjustbox}
  \vskip -0.1in
\end{table}

\begin{table}[h!]
  \caption{Performance comparison of our policy with other policies on classification tasks using the \textsc{gemma3-1b} (sLLM) and \textsc{ministral3-3b} (mLLM) models pair.}
  \label{tab:add_res5}
  \centering
      \begin{adjustbox}{max width=\columnwidth}
        \begin{tabular}{cccc}
          \toprule
          \bf Dataset & \bf Policy & \bf Performance ($\bf Accuracy$ / $\bf F_1$) & \bf Cost Saved ($\bf \%$) \\
          \midrule 
           \multirow{4}{*}{{\sc SST-2}} & None Deferred & 0.84 / 0.84 & 97.88$\%$ \\
           &  All Deferred & 0.84 / 0.84 & 0$\%$\\
           &  Random & 0.84 / 0.84 & 47.07$\%$\\
           &  FrugalGPT & 0.84 / 0.84 & 97.88$\%$\\
           &  HybridLLM & 0.84 / 0.84 & 97.88$\%$\\
           &  Ours & 0.84 / 0.84 & 97.88$\%$\\
           \midrule
           \multirow{4}{*}{{\sc agnews}} & None Deferred & 0.56 / 0.49 & 97.58$\%$ \\
           &  All Deferred & 0.82 / 0.81 & 0$\%$\\
           &  Random &  0.70 / 0.67 & 46.77$\%$\\
           &  FrugalGPT & 0.60 / 0.58 & 56.96$\%$\\
           &  HybridLLM & 0.81 / 0.80 & 45.02$\%$\\
           &  Ours &  0.82 / 0.80 & 59.90$\%$\\
           \midrule
           \multirow{4}{*}{{\sc emotion}} & None Deferred & 0.46 / 0.36 & 98.52$\%$ \\
           &  All Deferred & 0.59 / 0.50 & 0$\%$\\
           &  Random & 0.54 / 0.45 & 47.71$\%$\\
           &  FrugalGPT & 0.51 / 0.44 & 48.96$\%$\\
           &  HybridLLM & 0.54 / 0.40 & 75.33$\%$\\
           &  Ours &  0.60 / 0.49 & 52.77$\%$\\
           \midrule
           \multirow{4}{*}{{\sc fakenews}} & None Deferred & 0.52 / 0.40 & 95.74$\%$ \\
           &  All Deferred & 0.81 / 0.80 & 0$\%$\\
           &  Random & 0.66 / 0.63 & 45.12$\%$\\
           &  FrugalGPT & 0.83 / 0.83 & 38.99$\%$\\
           &  HybridLLM & 0.78 / 0.78 & 62.74$\%$\\
           &  Ours & 0.81 / 0.81 & 40.15$\%$\\
          \bottomrule
        \end{tabular}
        \end{adjustbox}
  \vskip -0.1in
\end{table}

\begin{table}[h!]
  \caption{Performance comparison of our policy with other policies on generation tasks using the \textsc{gemma3-1b} (sLLM) and \textsc{gemma3-12b} (mLLM) models pair. In the case of \textsc{WMT}, \textsc{gemma3-1b} was better than \textsc{gemma3-12b} so the policy did not defer any sample.}
  \label{tab:add_res7}
  \centering
      \begin{adjustbox}{max width=\columnwidth}
        \begin{tabular}{cccc}
          \toprule
          \bf Dataset & \bf Policy & \bf Performance (\textsc{rouge-l} / \textsc{cosine}) & \bf Cost Saved ($\bf \%$) \\
          \midrule 
           \multirow{4}{*}{{\sc SQuAD}} & None Deferred & 0.40 / 0.49 & 99.82$\%$ \\
           &  All Deferred & 0.77 / 0.79 & 0$\%$\\
           &  Random & 0.58 / 0.64 & 49.01$\%$\\
           &  FrugalGPT & 0.75 / 0.76 & 11.51$\%$\\
           &  HybridLLM & 0.77 / 0.79 & 0.70$\%$\\
           &  Ours & 0.75 / 0.77 & 11.95$\%$\\
           \midrule 
           \multirow{4}{*}{{\sc WMT}} & None Deferred & 0.37 / 0.62 & 98.39$\%$ \\
           &  All Deferred & 0.18 / 0.30 & 0$\%$\\
           &  Random & 0.28 / 0.45 & 47.58$\%$\\
           &  FrugalGPT & 0.37 / 0.62 & 98.39$\%$\\
           &  HybridLLM & 0.37 / 0.62 & 98.39$\%$\\
           &  Ours & 0.37 / 0.62 & 98.39$\%$\\
          \bottomrule
        \end{tabular}
        \end{adjustbox}
  \vskip -0.1in
\end{table}

\begin{table}[t!]
  \caption{Performance comparison of our policy with other policies on generation tasks using the \textsc{gemma3-4b} (sLLM) and \textsc{gemma3-12b} (mLLM) models pair. In the case of \textsc{WMT}, \textsc{gemma3-4b} was better than \textsc{gemma3-12b} so the policy did not defer any sample.}
  \label{tab:add_res8}
  \centering
      \begin{adjustbox}{max width=\columnwidth}
        \begin{tabular}{cccc}
          \toprule
          \bf Dataset & \bf Policy & \bf Performance (\textsc{rouge-l} / \textsc{cosine}) & \bf Cost Saved ($\bf \%$) \\
          \midrule 
           \multirow{4}{*}{{\sc SQuAD}} & None Deferred & 0.57 / 0.65 & 80.65$\%$ \\
           &  All Deferred & 0.77 / 0.79 & 0$\%$\\
           &  Random & 0.68 / 0.72 & 29.83$\%$\\
           &  FrugalGPT & 0.75 / 0.78 & 1.15$\%$\\
           &  HybridLLM & 0.71 / 0.74 & 9.77$\%$\\
           &  Ours & 0.75 / 0.77 & 12.02$\%$\\
           \midrule 
           \multirow{4}{*}{{\sc WMT}} & None Deferred & 0.18 / 0.32 & 46.68$\%$ \\
           &  All Deferred & 0.18 / 0.30 & 0$\%$\\
           &  Random & 0.19 / 0.31 & -4.12$\%$\\
           &  FrugalGPT & 0.18 / 0.32 & 46.68$\%$\\
           &  HybridLLM & 0.19 / 0.32 & -0.56$\%$\\
           &  Ours & 0.18 / 0.32 & 46.68$\%$\\
          \bottomrule
        \end{tabular}
        \end{adjustbox}
  \vskip -0.1in
\end{table}

\begin{table}[t!]
  \caption{Performance comparison of our policy with other policies on generation tasks using the \textsc{qwen3-4b} (sLLM) and \textsc{gemma3-12b} (mLLM) models pair. In the case of \textsc{WMT}, \textsc{qwen3-4b} was better than \textsc{gemma3-12b} so the policy did not defer any sample.}
  \label{tab:add_res9}
  \centering
      \begin{adjustbox}{max width=\columnwidth}
        \begin{tabular}{cccc}
          \toprule
          \bf Dataset & \bf Policy & \bf Performance (\textsc{rouge-l} / \textsc{cosine}) & \bf Cost Saved ($\bf \%$) \\
          \midrule 
           \multirow{4}{*}{{\sc SQuAD}} & None Deferred & 0.41 / 0.58 & 82.87$\%$ \\
           &  All Deferred & 0.77 / 0.79 & 0$\%$\\
           &  Random & 0.60 / 0.69 & 32.05$\%$\\
           &  FrugalGPT & 0.69 / 0.76 & 10.18$\%$\\
           &  HybridLLM & 0.53 / 0.65 & 48.99$\%$\\
           &  Ours & 0.67 / 0.75 & 12.87$\%$\\
           \midrule 
           \multirow{4}{*}{{\sc WMT}} & None Deferred & 0.37 / 0.63 & 56.68$\%$ \\
           &  All Deferred & 0.18 / 0.30 & 0$\%$\\
           &  Random & 0.27 / 0.46 & 5.87$\%$\\
           &  FrugalGPT & 0.37 / 0.63 & 56.68$\%$\\
           &  HybridLLM & 0.37 / 0.63 & 56.68$\%$\\
           &  Ours & 0.37 / 0.63 & 56.68$\%$\\
          \bottomrule
        \end{tabular}
        \end{adjustbox}
  \vskip -0.1in
\end{table}

\begin{table}[t!]
  \caption{Performance comparison of our policy with other policies on generation tasks using the \textsc{gemma3-1b} (sLLM) and \textsc{qwen3-4b} (mLLM) models pair. }
  \label{tab:add_res10}
  \centering
      \begin{adjustbox}{max width=\columnwidth}
        \begin{tabular}{cccc}
          \toprule
          \bf Dataset & \bf Policy & \bf Performance (\textsc{rouge-l} / \textsc{cosine}) & \bf Cost Saved ($\bf \%$) \\
          \midrule 
           \multirow{4}{*}{{\sc SQuAD}} & None Deferred & 0.39 / 0.48 & 98.98$\%$ \\
           &  All Deferred & 0.41 / 0.58 & 0$\%$\\
           &  Random &  0.39 / 0.53 & 48.16$\%$\\
           &  FrugalGPT & 0.48 / 0.58 & 40.91$\%$\\
           &  HybridLLM & 0.40 / 0.57 & 3.91$\%$\\
           &  Ours & 0.46 / 0.57 & 43.23$\%$\\
           \midrule 
           \multirow{4}{*}{{\sc WMT}} & None Deferred & 0.37 / 0.62 & 96.29$\%$ \\
           &  All Deferred & 0.37 / 0.63 & 0$\%$\\
           &  Random &  0.37 / 0.62 & 45.48$\%$\\
           &  FrugalGPT & 0.38 / 0.62 & 63.17$\%$\\
           &  HybridLLM & 0.38 / 0.62 & 71.36$\%$\\
           &  Ours & 0.38 / 0.64 & 40.17$\%$\\
          \bottomrule
        \end{tabular}
        \end{adjustbox}
  \vskip -0.1in
\end{table}

\begin{table}[h!]
  \caption{Performance comparison of our policy with other policies on generation tasks using the \textsc{gemma3-1b} (sLLM) and \textsc{ministral3-3b} (mLLM) models pair. In the case of \textsc{WMT}, \textsc{gemma3-1b} was better than \textsc{ministral3-3b} so the policy did not defer any sample.}
  \label{tab:add_res11}
  \centering
      \begin{adjustbox}{max width=\columnwidth}
        \begin{tabular}{cccc}
          \toprule
          \bf Dataset & \bf Policy & \bf Performance (\textsc{rouge-l} / \textsc{cosine}) & \bf Cost Saved ($\bf \%$) \\
          \midrule 
           \multirow{4}{*}{{\sc SQuAD}} & None Deferred & 0.40 / 0.48 & 98.92$\%$ \\
           &  All Deferred & 0.82 / 0.85 & 0$\%$\\
           &  Random &  0.60 / 0.67 & 48.10$\%$\\
           &  FrugalGPT & 0.79 / 0.81 & 10.10$\%$\\
           &  HybridLLM & 0.82 / 0.84 & 0.85$\%$\\
           &  Ours & 0.81 / 0.82 & 11.04$\%$\\
           \midrule 
           \multirow{4}{*}{{\sc WMT}} & None Deferred & 0.37 / 0.62 & 97.93$\%$ \\
           &  All Deferred & 0.09 / 0.16 & 0$\%$\\
           &  Random &  0.23 / 0.39 & 47.12$\%$\\
           &  FrugalGPT & 0.37 / 0.62 & 97.93$\%$\\
           &  HybridLLM & 0.37 / 0.62 & 97.93$\%$\\
           &  Ours & 0.37 / 0.62 & 97.93$\%$\\
          \bottomrule
        \end{tabular}
        \end{adjustbox}
  \vskip -0.1in
\end{table}


\begin{table}[t!]
  \caption{Performance comparison of our policy with other policies on classification tasks using the \textsc{gemma3-4b} (sLLM) and \textsc{gemma3-27b} (mLLM) models pair in the \textsc{sst-2} dataset.}
  \label{tab:add_res13}
  \centering
      \begin{adjustbox}{max width=\columnwidth}
        \begin{tabular}{cccc}
          \toprule
          \bf Dataset & \bf Policy & \bf Performance ($\bf Accuracy$ / $\bf F_1$)  & \bf Cost Saved ($\bf \%$) \\
          \midrule 
           \multirow{4}{*}{{\sc sst-2}} & None Deferred & 0.90 / 0.90 & 88.33$\%$ \\
           &  All Deferred & 0.90 / 0.90 & 0$\%$\\
           &  Random &  0.89 / 0.89 & 37.51$\%$\\
           &  FrugalGPT & 0.90 / 0.90 & 84.08$\%$\\
           &  HybridLLM & 0.90 / 0.90 & 77.95$\%$\\
           &  Ours & 0.91 / 0.91 & 82.89$\%$\\
          \bottomrule
        \end{tabular}
        \end{adjustbox}
  \vskip -0.1in
\end{table}

\begin{table}[t!]
  \caption{Performance comparison of our policy with other policies on classification tasks using the \textsc{gemma3-12b} (sLLM) and \textsc{gemma3-27b} (mLLM) models pair in the \textsc{sst-2} dataset. In this case, \textsc{gemma3-12b} was better than \textsc{gemma3-27b}.}
  \label{tab:add_res14}
  \centering
      \begin{adjustbox}{max width=\columnwidth}
        \begin{tabular}{cccc}
          \toprule
          \bf Dataset & \bf Policy & \bf Performance ($\bf Accuracy$ / $\bf F_1$)  & \bf Cost Saved ($\bf \%$) \\
          \midrule 
           \multirow{4}{*}{{\sc sst-2}} & None Deferred & 0.91 / 0.91 & 61.20$\%$ \\
           &  All Deferred & 0.90 / 0.90 & 0$\%$\\
           &  Random &  0.89 / 0.89 & 10.39$\%$\\
           &  FrugalGPT & 0.91 / 0.91 & 61.20$\%$\\
           &  HybridLLM & 0.91 / 0.91 & 61.20$\%$\\
           &  Ours & 0.91 / 0.91 & 61.20$\%$\\
          \bottomrule
        \end{tabular}
        \end{adjustbox}
  \vskip -0.1in
\end{table}
\begin{table}[t!]
  \caption{Classification results for our policy and baselines in the \textsc{sst-2} dataset, using using the \textsc{qwen3-4b} (sLLM) and \textsc{gemma3-27b} (mLLM) models pair.}
  \label{tab:add_res15}
  \centering
      \begin{adjustbox}{max width=\columnwidth}
        \begin{tabular}{cccc}
          \toprule
          \bf Dataset & \bf Policy & \bf Performance ($\bf Accuracy$ / $\bf F_1$)  & \bf Cost Saved ($\bf \%$) \\
          \midrule 
           \multirow{4}{*}{{\sc sst-2}} & None Deferred & 0.86 / 0.86 & 88.89$\%$ \\
           &  All Deferred & 0.89 / 0.89 & 0$\%$\\
           &  Random &  0.88 / 0.88 & 38.08$\%$\\
           &  FrugalGPT & 0.88 / 0.88 & 83.58$\%$\\
           &  HybridLLM & 0.88 / 0.88 & 61.77$\%$\\
           &  Ours & 0.90 / 0.90 & 82.46$\%$\\
          \bottomrule
        \end{tabular}
        \end{adjustbox}
  \vskip -0.1in
\end{table}

\begin{table}[t!]
  \caption{Classification results for our policy and baselines in the \textsc{sst-2} dataset, using using the \textsc{ministral3-3b} (sLLM) and \textsc{gemma3-27b} (mLLM) models pair.}
  \label{tab:add_res16}
  \centering
      \begin{adjustbox}{max width=\columnwidth}
        \begin{tabular}{cccc}
          \toprule
          \bf Dataset & \bf Policy & \bf Performance ($\bf Accuracy$ / $\bf F_1$)  & \bf Cost Saved ($\bf \%$) \\
          \midrule 
           \multirow{4}{*}{{\sc sst-2}} & None Deferred & 0.84 / 0.84 & 100$\%$ \\
           &  All Deferred & 0.90 / 0.90 & 0$\%$\\
           &  Random &  0.87 / 0.87 & 40.50$\%$\\
           &  Ours & 0.91 / 0.91 & 64.33$\%$\\
          \bottomrule
        \end{tabular}
        \end{adjustbox}
  \vskip -0.1in
\end{table}
\subsection{Monte Carlo sampling}\label{app:mc_sampling_details}
The theoretical cost and error curves can be written as expectations of functions of the random tuple $(\boldsymbol{a}, \hat{\boldsymbol{a}}^s, \beta, \hat{\boldsymbol{a}}^m)$ under the data-generating distribution. We approximate these expectations by sample averages. Concretely, for any policy parameter $\theta$ (scalar or vector), let $g(z;\theta)$ denote the per-sample error indicator (defined below). Then
\begin{equation}
\mathbb{E}[g(Z;\theta)] \;\approx\; \widehat{\mathbb{E}}_N[g(Z;\theta)]
\;:=\; \frac{1}{N}\sum_{i=1}^N g(z_i;\theta),
\label{eq:mc_expectation}
\end{equation}
where $z_i$ denotes the observed tuple on example $i$ and $N$ is the number of available tuples. This direct MC evaluation 
approximates the full expectations appearing in the integrals by sample averages of indicator functions.

\subsubsection{Single-threshold policy}
\label{sec:mc_single_threshold}

In the single-threshold policy, we accept sLLM's prediction if $\beta_i\ge \theta$ and defer otherwise:
\begin{equation}
A_i(\theta) = \mathds{1}\{\beta_i \ge \theta\}, \qquad D_i(\theta) = 1-A_i(\theta).
\label{eq:accept_def_single}
\end{equation}
Let $e_i^{s} = \mathds{1}\{\hat{\boldsymbol{a}}^s_i\neq \boldsymbol{a}_i\}$ and
$e_i^{m} = \mathds{1}\{\hat{\boldsymbol{a}}^m_i\neq \boldsymbol{a}_i\}$ denote per-sample error indicators for sLLM and mLLM, respectively. In practice, we use the single-threshold policy with generative tasks where models provide open-ended answers. Thus, exact equality with a reference output can be difficult to acquire; there may be multiple valid responses, partial correctness, and semantic variation. As a result, the binary indicator functions may require thresholding of the performance metric. To address this, one can relax these binary error indicators into continuous ones, i.e., $
e_i^s = 1 - M(\hat{\boldsymbol{a}}^s_i, \boldsymbol{a}_i)$ and $
e_i^m = 1 - M(\hat{\boldsymbol{a}}^m_i, \boldsymbol{a}_i),
$, where $M(\cdot, \cdot)$ is the performance metric of choice (e.g., cosine similarity between answers' embeddings). We leave this formulation for future work. 

\paragraph{Oracle error (single-threshold).}
In the oracle regime, deferred examples incur no error, hence the MC estimate of the error is
\begin{equation}
\widehat{\mathrm{Err}}_{\mathrm{ST-O}}(\theta)
\;=\;
\frac{1}{N}\sum_{i=1}^N A_i(\theta)\, e_i^{s}.
\label{eq:mc_err_single_oracle}
\end{equation}

\paragraph{Non-oracle error (single-threshold).}
In the non-oracle regime, deferred examples use mLLM that may err:
\begin{equation}
\widehat{\mathrm{Err}}_{\mathrm{ST-NO}}(\theta)
\;=\;
\frac{1}{N}\sum_{i=1}^N\Big(
A_i(\theta)\, e_i^{s} \;+\; D_i(\theta)\, e_i^{m}
\Big).
\label{eq:mc_err_single_nonoracle}
\end{equation}

\paragraph{MC cost (single-threshold).}
We use a cost model with constant sLLM cost $c_s>0$ and additional mLLM cost $c_m>0$ when deferring. Thus, the expected per-sample cost and its MC estimate are:
\begin{equation}
\widehat{\mathrm{Cost}}(\theta)
\;=\;
c_s \;+\; c_m \cdot \widehat{\mathrm{Def}}(\theta),
\qquad
\widehat{\mathrm{Def}}(\theta) := \frac{1}{N}\sum_{i=1}^N D_i(\theta)\,,
\label{eq:mc_cost_single}
\end{equation}
where $\widehat{\mathrm{Def}}(\theta)$ is the empirical deferral rate (and $1-\widehat{\mathrm{Def}}(\theta)$ the empirical coverage).

\subsubsection{Multi-thresholds policy}
\label{sec:mc_multi_threshold}

In the multi-thresholds policy,  we employ a class-conditional threshold vector
$\boldsymbol{\theta}=(\theta_1,\dots,\theta_K)$, where $\theta_k$ is applied to examples whose sLLM predicted label equals $k$. The accept/deferral rule is:
\begin{equation}
A_i(\boldsymbol{\theta}) = \mathds{1}\{\beta_i \ge \theta_{\hat{\boldsymbol{a}}^s_i}\},
\qquad
D_i(\boldsymbol{\theta}) = 1-A_i(\boldsymbol{\theta}).
\label{eq:accept_def_multi}
\end{equation}

\paragraph{Oracle error (multi-thresholds).}
Under the oracle regime,
\begin{equation}
\widehat{\mathrm{Err}}_{\mathrm{MT-O}}(\boldsymbol{\theta})
\;=\;
\frac{1}{N}\sum_{i=1}^N A_i(\boldsymbol{\theta})\, e_i^{s}.
\label{eq:mc_err_multi_oracle}
\end{equation}

\paragraph{Non-oracle error (multi-thresholds).}
Under the non-oracle regime,
\begin{equation}
\widehat{\mathrm{Err}}_{\mathrm{MT-NO}}(\boldsymbol{\theta})
\;=\;
\frac{1}{N}\sum_{i=1}^N\Big(
A_i(\boldsymbol{\theta})\, e_i^{s} \;+\; D_i(\boldsymbol{\theta})\, e_i^{m}
\Big).
\label{eq:mc_err_multi_nonoracle}
\end{equation}

\paragraph{MC cost (multi-thresholds).}
The cost expression is identical in form to the single-threshold case  (\ref{eq:mc_cost_single}).

\subsubsection{Finite-sample effects and empirical feasibility}
\label{sec:mc_discussion}
Finite-sample MC estimates yield stepwise functions of the threshold(s), since the accept/defer sets change only when thresholds cross observed confidence values. Consequently, the optimal empirical solution to
\begin{equation}
\min_{\theta}~\widehat{\mathrm{Cost}}(\theta)
\quad \text{s.t.}\quad
\widehat{\mathrm{Err}}(\theta)\le b
\label{eq:mc_constrained}
\end{equation}
may exhibit \emph{slack} ($\widehat{\mathrm{Err}}(\theta^*) < b$), because there may be no threshold configuration whose empirical error equals $b$ exactly. This is a discretization effect inherent to optimizing over thresholds on a finite dataset and does not contradict the population-level optimality condition that places the optimum on the boundary when the constraint is active and the error curve is continuous. 
Finally, note that the class prior terms $p_k=\Pr(\hat{\boldsymbol{a}}^s=k)$ appearing in the theoretical integral decompositions are automatically accounted for by MC evaluation: summing over all samples weights each predicted class proportionally to its empirical frequency $n_k/N$, yielding the standard sample-average approximation of the population expectation.

\subsection{Details about implemented algorithms}\label{app:algo_details}

We briefly describe the algorithms used to solve the empirical constrained optimization problem (\ref{eq:mc_constrained}, together with their computational complexity. In all cases, we minimize the expected per-sample cost (equivalently, the deferral rate) subject to a global error budget. A key observation is that on a finite dataset the error and deferral rate are \emph{stepwise} functions of the threshold(s): the accept/defer partition changes only when a threshold crosses an observed confidence value (or a tie-group of equal confidences). Hence, the search can be restricted to a finite set of candidates induced by the sample.

\subsubsection{Single-threshold policy}
\label{sec:alg_single_threshold}

Consider the single-threshold rule $A_i(\theta)=\mathds{1}\{\beta_i\ge \theta\}$. Since the empirical cost
$\widehat{\mathrm{Cost}}(\theta)$ is non-decreasing in $\theta$ (larger thresholds defer more), the minimum-cost feasible solution is obtained by selecting the \emph{smallest} threshold that satisfies the constraint $\widehat{\mathrm{Err}}(\theta)\le b$, when such a threshold exists. On finite samples, it suffices to evaluate $\widehat{\mathrm{Err}}(\theta)$ at the distinct observed $\beta$ values (plus the two endpoints corresponding to ``accept all'' and ``defer all'').

\paragraph{Sorting-based scan ($\mathcal{O}(N\log N)$).}
We sort examples by $\beta_i$ and compute cumulative sums of error indicators to evaluate all candidate thresholds in a single pass. 
Grouping ties in $\beta$ yields a tie-aware scan over unique confidence levels. The runtime is dominated by sorting: $\mathcal{O}(N\log N)$ time and $\mathcal{O}(N)$ memory (or $\mathcal{O}(N)$ time after sorting for the cumulative updates). An implementation outline is presented in Algorithm~\ref{alg:single_oracle_compact} (for the oracle case) and Algorithm~\ref{alg:single_nonoracle_compact} (for the non-oracle case).
\begin{algorithm}[t]
\caption{Single-threshold selection (oracle, MC, sorting-based)}
\label{alg:single_oracle_compact}
\begin{algorithmic}[1]
\STATE {\bfseries Input:} Dataset with generated answers $\{(\beta_i,\hat{\boldsymbol{a}}^m_i,\hat{\boldsymbol{a}}^s_i)\}_{i=1}^N$, accuracy lower bound $\xi\in[0,1]$, costs $(c_s,c_m)$
\STATE {\bfseries Output:} Optimal threshold $\theta^*$ minimizing cost w.r.t the error constraint 
\STATE Set error budget $b \gets 1 - \xi$
\STATE Compute $e_i^{s} \gets \mathds{1}\{\hat{\boldsymbol{a}}^s_i\neq \hat{\boldsymbol{a}}^m_i\}$
\STATE Sort samples by $\beta$: $\beta_{1}\le \dots \le \beta_{N}$ and reorder $e^{s}$ accordingly
\STATE Let $u_1<\dots<u_L$ be distinct confidence values; define tie-groups $G_\ell=\{i:\beta_{i}=u_\ell\}$
\STATE For each $\ell$: $N_\ell\gets|G_\ell|$, $E^{s}_\ell\gets\sum_{i\in G_\ell} e^{s}_{i}$
\STATE \textbf{Accept-all check:} $\widehat{\mathrm{Err}}_{\mathrm{ST-O}}\gets \frac{1}{N}\sum_{i=1}^N e^{s}_{i}$
\IF{$\widehat{\mathrm{Err}}_{\mathrm{ST-O}} \le b$}
  \STATE $\theta^*\gets u_1$;\quad $\widehat{\mathrm{Def}}\gets 0$;\quad $\widehat{\mathrm{Cost}}\gets c_s$
  \STATE \textbf{Return} $\theta^*$
\ENDIF
\STATE $\mathrm{AcceptedErrors}\gets \sum_{\ell'=1}^L E^{s}_{\ell'}$ 
\STATE $\mathrm{DeferredCount}\gets 0$
\FOR{$\ell=1$ to $L$}
  \STATE $\widehat{\mathrm{Err}}_{\mathrm{ST-O}}(u_\ell)\gets \mathrm{AcceptedErrors}/N$
  \IF{$\widehat{\mathrm{Err}}_{\mathrm{ST-O}}(u_\ell)\le b$}
    \STATE $\theta^*\gets u_\ell$
    \STATE $\widehat{\mathrm{Def}}\gets \mathrm{DeferredCount}/N$
    \STATE $\widehat{\mathrm{Cost}}\gets c_s + c_m\,\widehat{\mathrm{Def}}$
    \STATE \textbf{Return} $\theta^*$
  \ENDIF
  \COMMENT{Update totals for next threshold: defer level $\ell$ (oracle: no deferred error)}
  \STATE $\mathrm{DeferredCount}\gets \mathrm{DeferredCount}+N_\ell$
  \STATE $\mathrm{AcceptedErrors}\gets \mathrm{AcceptedErrors}-E^{s}_\ell$
\ENDFOR
\STATE $\theta^*\gets u_L$;\quad $\widehat{\mathrm{Def}}\gets 1$;\quad $\widehat{\mathrm{Err}}\gets 0$;\quad $\widehat{\mathrm{Cost}}\gets c_s+c_m$
\STATE \textbf{Return} $\theta^*$
\end{algorithmic}
\end{algorithm}

\begin{algorithm}[t]
\caption{Single-threshold selection (non-oracle, MC, sorting-based)}
\label{alg:single_nonoracle_compact}
\begin{algorithmic}[1]
\STATE {\bfseries Input:} Dataset with generated answers $\{(\beta_i,\hat{\boldsymbol{a}}^m_i,\hat{\boldsymbol{a}}^s_i, \boldsymbol{a}_i)\}_{i=1}^N$, accuracy lower bound $\xi\in[0,1]$, costs $(c_s,c_m)$
\STATE {\bfseries Output:} Optimal threshold $\theta^*$ minimizing cost w.r.t the error constraint 
\STATE Set error budget $b\gets 1-\xi$
\STATE Compute $e^{s}_i \gets \mathds{1}\{\hat{\boldsymbol{a}}^s_i\neq \boldsymbol{a}_i\}$, $e^{m}_i \gets \mathds{1}\{\hat{\boldsymbol{a}}^m_i\neq \boldsymbol{a}_i\}$
\STATE Sort samples by $\beta$: $\beta_{1}\le \dots \le \beta_{N}$ and reorder $e^{s},e^{m}$ accordingly
\STATE Let $u_1<\dots<u_L$ be the distinct values in $\{\beta_{i}\}$; for each $\ell$, let $G_\ell=\{i:\beta_{(i)}=u_\ell\}$
\STATE For each level $\ell$: $N_\ell\gets |G_\ell|$, $E^{s}_\ell\gets\sum_{i\in G_\ell} e^{s}_{i}$, $E^{m}_\ell\gets\sum_{i\in G_\ell} e^{m}_{i}$
\STATE \textbf{Accept-all check:} $\widehat{\mathrm{Err}}_{\mathrm{ST-NO}}\gets \frac{1}{N}\sum_{i=1}^N e^{(s)}_{(i)}$
\IF{$\widehat{\mathrm{Err}}_{\mathrm{ST-NO}} \le b$}
  \STATE $\theta^* \gets u_1$;\quad $\widehat{\mathrm{Def}}\gets 0$;\quad $\widehat{\mathrm{Cost}}\gets c_s$
  \STATE \textbf{Return} $\theta^*$
\ENDIF
\STATE $\mathrm{DeferredErrors}\gets 0$ \COMMENT{$\sum_{\beta<u_\ell} e^{m}$}
\STATE $\mathrm{DeferredCount}\gets 0$ \COMMENT{$\#\{i:\beta<u_\ell\}$}
\STATE $\mathrm{AcceptedErrors}\gets \sum_{\ell'=1}^L E^{s}_{\ell'}$ \COMMENT{$\sum_{\beta\ge u_1} e^{s}$ initially}
\FOR{$\ell=1$ to $L$}
  \STATE $\widehat{\mathrm{Err}}_{\mathrm{ST-NO}}(u_\ell)\gets \frac{\mathrm{DeferredErrors}+\mathrm{AcceptedErrors}}{N}$
  \COMMENT{Candidate threshold $\theta=u_\ell$: defer levels $<\ell$, accept levels $\ge \ell$}
  \IF{$\widehat{\mathrm{Err}}_{\mathrm{ST-NO}}(u_\ell)\le b$}
    \STATE $\theta^*\gets u_\ell$
    \STATE $\widehat{\mathrm{Def}}\gets \mathrm{DeferredCount}/N$
    \STATE $\widehat{\mathrm{Cost}}\gets c_s + c_m\,\widehat{\mathrm{Def}}$
    \STATE \textbf{Return} $\theta^*$
  \ENDIF
  \STATE $\mathrm{DeferredErrors}\gets \mathrm{DeferredErrors}+E^{m}_\ell$
  \STATE $\mathrm{DeferredCount}\gets \mathrm{DeferredCount}+N_\ell$
  \STATE $\mathrm{AcceptedErrors}\gets \mathrm{AcceptedErrors}-E^{s}_\ell$
\ENDFOR

\STATE $\theta^*\gets u_L$;\quad $\widehat{\mathrm{Def}}\gets 1$;\quad $\widehat{\mathrm{Err}}\gets \frac{1}{N}\sum_{i=1}^N e^{m}_{i}$;\quad $\widehat{\mathrm{Cost}}\gets c_s+c_m$ \COMMENT{No feasible threshold: defer all (choose any $\theta>u_L$)}
\STATE \textbf{Return} $\theta^*$
\end{algorithmic}
\end{algorithm}

\paragraph{BFPRT variant ($\mathcal{O}(N)$ worst-case; oracle only).}
In the oracle case, feasibility depends only on the number of small-model mistakes among accepted samples, and the optimal solution corresponds to accepting the largest-confidence subset that satisfies the error budget. This can be implemented without full sorting by using a linear-time order-statistic routine (median-of-medians/BFPRT \citep{BLUM1973448}) to find the relevant confidence quantile (and then a linear pass to count errors/deferrals). This yields $\mathcal{O}(N)$ worst-case time. We stress that this guarantee relies on the oracle structure; in the non-oracle case the constraint depends on \emph{both} accepted and deferred subsets, 
and an exact worst-case $\mathcal{O}(N)$ algorithm is generally not available without additional assumptions. Implementation outline is presented in Algorithm~\ref{alg:single_oracle_linear_compact}.

\begin{algorithm}[t]
\caption{Single-threshold selection (oracle, MC, BFPRT variant)}
\label{alg:single_oracle_linear_compact}
\begin{algorithmic}[1]
\STATE {\bfseries Input:} Dataset with generated answers $\{(\beta_i,\hat{\boldsymbol{a}}^m_i,\hat{\boldsymbol{a}}^s_i)\}_{i=1}^N$, accuracy lower bound $\xi\in[0,1]$, costs $(c_s,c_m)$
\STATE {\bfseries Output:} Optimal threshold $\theta^*$ minimizing cost w.r.t the error constraint 
\STATE Compute $e^{s}_i \gets \mathds{1}\{\hat{\boldsymbol{a}}^s_i\neq \hat{\boldsymbol{a}}^m_i\}$
\IF{$\frac{1}{N}\sum_{i=1}^N e^{s}_i \le b$}
  \STATE \textbf{Return} any $\theta^*$ that accepts all samples
\ENDIF
\STATE Use a median-of-medians \textsc{Select} routine to obtain a pivot confidence $\theta$ \COMMENT{We search over acceptance sets defined by confidence cutoffs without full sorting}
\STATE Partition indices into $A=\{i:\beta_i\ge\theta\}$ and $R=\{i:\beta_i<\theta\}$
\STATE Compute $\widehat{\mathrm{Err}}_{\mathrm{ST-O}}(\theta)=\frac{1}{N}\sum_{i\in A} e^{s}_i$
\IF{$\widehat{\mathrm{Err}}_{\mathrm{ST-O}}(\theta)\le b$}
  \STATE Recurse on $R$ to try a smaller feasible threshold (accept more)
\ELSE
  \STATE Recurse on $A$ to increase the threshold (accept fewer)
\ENDIF
\STATE Return the smallest feasible cutoff encountered as $\theta^*$
\end{algorithmic}
\end{algorithm}

\subsubsection{Multi-thresholds policy}
\label{sec:alg_multi_threshold}
For each predicted class $k$, we sort the samples with $\hat{\boldsymbol{a}}^s=k$ by $\beta$ and consider tie-aware unique confidence levels. This produces a finite set of \emph{options} per class $k$, indexed by $j$, each corresponding to deferring the lowest-confidence prefix within that class. Each option yields a pair:
\[
(d_{k,j},\, e_{k,j}),
\]
where $d_{k,j}$ is the number of deferred samples in class $k$ and $e_{k,j}$ is the number of total errors contributed by class $k$ (oracle: accepted sLLM mistakes; non-oracle: accepted sLLM mistakes plus deferred mLLM mistakes). The global optimization becomes: choose one option per class to minimize total deferrals subject to a total error-count budget $B=\lfloor bN\rfloor$.

\paragraph{Option-table construction.}
Let $n_k$ be the number of samples with $\hat{\boldsymbol{a}}^s=k$ and let $U_k\le n_k$ be the number of unique confidence levels within class $k$ (tie-aware). Constructing per-class option matrices requires sorting within each class:
$\sum_{k=1}^K \mathcal{O}(n_k\log n_k)\le \mathcal{O}(N\log N)$,
followed by linear-time cumulative sums, $\mathcal{O}(N)$.
The resulting total number of options is $U=\sum_k (U_k+1)$ (including the ``defer all'' endpoint), with $U\le N+K$. Algorithms for oracle and non-oracle settings are presented in Algorithm \ref{alg:multi_options_oracle_compact} and Algorithm \ref{alg:multi_options_compact}, respectively.
\begin{algorithm}[t]
\caption{Build per-class option matrices (oracle)}
\label{alg:multi_options_oracle_compact}
\begin{algorithmic}[1]
\STATE {\bfseries Input:} $\{(\beta_i,\hat{\boldsymbol{a}}^s_i,\hat{\boldsymbol{a}}^m_i)\}_{i=1}^N$, number of classes $K$
\STATE {\bfseries Output:} For each class $k$: options $\{(d_{k,j},e_{k,j},\theta_{k,j})\}_{j=1}^{J_k}$
\STATE Compute $e^{s}_i=\mathds{1}\{\hat{\boldsymbol{a}}^s_i\neq \hat{\boldsymbol{a}}^m_i\}$
\FOR{$k=1$ to $K$}
  \STATE $S_k \gets \{i:\hat{\boldsymbol{a}}^s_i=k\}$;\quad $n_k\gets |S_k|$
  \IF{$n_k=0$} \STATE Store one trivial option $(d,e)=(0,0)$ and \textbf{continue} \ENDIF
  \STATE Sort $S_k$ by $\beta_i$ ascending; let $u_1<\dots<u_L$ be distinct values in class $k$
  \STATE For each level $\ell$, let $G_\ell=\{i\in S_k:\beta_i=u_\ell\}$ and compute:
  $
  N_\ell=|G_\ell|,
  \quad
  E^{s}_\ell=\sum_{i\in G_\ell} e^{s}_i
  $
  \STATE $\mathrm{DeferredCount}\gets 0$ \COMMENT{$\#\{i\in S_k:\beta<u_j\}$}
  \STATE $\mathrm{AcceptedErrors}\gets \sum_{\ell'=1}^L E^{(s)}_{\ell'}$ \COMMENT{$\sum_{\beta\ge u_1} e^{(s)}$ in class $k$}
  \FOR{$j=1$ to $L$}
    \STATE $\theta_{k,j}\gets u_j$
    \STATE $d_{k,j}\gets \mathrm{DeferredCount}$
    \STATE $e_{k,j}\gets \mathrm{AcceptedErrors}$ \COMMENT{oracle: only accepted sLLM mistakes matter}
    \STATE $\mathrm{DeferredCount}\gets \mathrm{DeferredCount}+N_{j}$
    \STATE $\mathrm{AcceptedErrors}\gets \mathrm{AcceptedErrors}-E^{s}_{j}$
  \ENDFOR
  \COMMENT{Defer-all option: accept none in class $k$}
  \STATE Add defer-all option: $d_{k,L+1}\gets n_k$, $e_{k,L+1}\gets 0$
\ENDFOR
\STATE \textbf{Return} option matrices
\end{algorithmic}
\end{algorithm}
\begin{algorithm}[t]
\caption{Build per-class option matrices (non-oracle)}
\label{alg:multi_options_compact}
\begin{algorithmic}[1]
\STATE {\bfseries Input:} $\{(\beta_i,\hat{\boldsymbol{a}}^s_i,\hat{\boldsymbol{a}}^m_i, \boldsymbol{a}_i)\}_{i=1}^N$, number of classes $K$
\STATE {\bfseries Output:} For each class $k$: options $\{(d_{k,j},e_{k,j},\theta_{k,j})\}_{j=1}^{J_k}$
\STATE Compute $e^{s}_i=\mathds{1}\{\hat{\boldsymbol{a}}^s_i\neq \boldsymbol{a}_i\}$, $e^{m}_i=\mathds{1}\{\hat{\boldsymbol{a}}^m_i\neq \boldsymbol{a}_i\}$
\FOR{$k=1$ to $K$}
  \STATE $S_k \gets \{i:\hat{\boldsymbol{a}}^s_i=k\}$; \quad $n_k\gets |S_k|$
  \IF{$n_k=0$} \STATE Store one trivial option $(d,e)=(0,0)$ and \textbf{continue} \ENDIF
  \STATE Sort $S_k$ by $\beta_i$ ascending; let $u_1<\ldots<u_L$ be distinct values in class $k$
  \STATE For each level $\ell$, let $G_\ell=\{i\in S_k:\beta_i=u_\ell\}$ and compute:
  $
  N_\ell=|G_\ell|,\quad
  E^{s}_\ell=\sum_{i\in G_\ell} e^{s}_i,\quad
  E^{m}_\ell=\sum_{i\in G_\ell} e^{m}_i
  $
  \STATE $\mathrm{DeferredErrors}\gets 0$ \COMMENT{$\sum_{\beta<u_j} e^{m}$ in class $k$}
  \STATE $\mathrm{DeferredCount}\gets 0$ \COMMENT{$\#\{i\in S_k:\beta<u_j\}$}
  \STATE$\mathrm{AcceptedErrors}\gets \sum_{\ell'=1}^L E^{s}_{\ell'}$ \COMMENT{$\sum_{\beta\ge u_1} e^{s}$ in class $k$}
  \FOR{$j=1$ to $L$}
    \STATE $\theta_{k,j}\gets u_j$
    \STATE $d_{k,j}\gets \mathrm{DeferredCount}$
    \STATE $e_{k,j}\gets \mathrm{DeferredErrors}+\mathrm{AcceptedErrors}$
    \COMMENT{Update totals for next option $j{+}1$}
    \STATE $\mathrm{DeferredErrors}\gets \mathrm{DeferredErrors}+E^{m}_{j}$
    \STATE $\mathrm{DeferredCount}\gets \mathrm{DeferredCount}+N_{j}$
    \STATE $\mathrm{AcceptedErrors}\gets \mathrm{AcceptedErrors}-E^{s}_{j}$
  \ENDFOR
  \STATE Add defer-all option: $d_{k,L+1}\gets n_k$, $e_{k,L+1}\gets \sum_{\ell=1}^L E^{m}_\ell$
\ENDFOR
\STATE \textbf{Return} option matrices
\end{algorithmic}
\end{algorithm}

\paragraph{Dynamic programming (pseudo-polynomial).}
The finite-sample problem is a multi-choice knapsack instance. An exact solution can be obtained via dynamic programming over the integer error budget $B$, storing for each attainable total error the minimum total deferrals achievable after processing a subset of classes. In the worst case, the runtime is
\[
\mathcal{O}\!\Big(\sum_{k=1}^K (U_k+1)\cdot B\Big)=\mathcal{O}(U B),
\]
and memory is $\mathcal{O}(B)$ if one stores only the current DP frontier (plus backpointers if reconstructing the selected options). Since $B=\Theta(N)$ and $U\le \Theta(N)$, the worst-case runtime can be $\mathcal{O}(N^2)$, which can be heavy for large $N$. In practice, we maintain only the Pareto-optimal (non-dominated) frontier of states. 
An implementation sketch is presented in Algorithm \ref{alg:multi_nonoracle_dp_compact}. There is no change in the oracle and the non-oracle setting, apart from the fact that each algorithm takes as input the corresponding `per-class option matrices'.
\begin{algorithm}[t]
\caption{Multi-thresholds selection (oracle \& non-oracle, MC, exact DP)}
\label{alg:multi_nonoracle_dp_compact}
\begin{algorithmic}[1]
\STATE {\bfseries Input:} Option matrices $\{(d_{k,j},e_{k,j},\theta_{k,j})\}_{j=1}^{J_k}$ for each class $k$, sample size $N$, accuracy bound $\xi$
\STATE {\bfseries Output:} Optimal threshold vector $\boldsymbol{\theta}^*$
\STATE Set error budget $b\gets 1-\xi$
\STATE $B\gets \lfloor bN\rfloor$ \COMMENT{integer error budget}
\STATE $\mathrm{DP}\gets \{0\mapsto 0\}$ \COMMENT{error count $\to$ min deferrals}
\FOR{$k=1$ to $K$}
  \STATE $\mathrm{DP'}\gets$ empty map
  \FORALL{$(E \mapsto D)$ in $\mathrm{DP}$}
    \FORALL{options $j$ of class $k$}
      \STATE $E' \gets E + e_{k,j}$
      \IF{$E' \le B$}
        \STATE $\mathrm{DP'}[E'] \gets \min\big(\mathrm{DP'}[E'],\, D + d_{k,j}\big)$
        \STATE Store backpointer for reconstruction
      \ENDIF
    \ENDFOR
  \ENDFOR
  \STATE Pareto-prune $\mathrm{DP'}$ (remove dominated states) and set $\mathrm{DP}\gets \mathrm{DP'}$
\ENDFOR
\STATE Choose $E^* \in \arg\min_{E\le B} \mathrm{DP}[E]$
\STATE Backtrack to recover chosen option $j_k^*$ for each class
\STATE Set $\theta_k^* \gets \theta_{k,j_k^*}$ for all $k$ and return $\boldsymbol{\theta}^*$
\end{algorithmic}
\end{algorithm}

\paragraph{Lagrangian sweep (scalable approximation).}
To improve scalability, we consider a Lagrangian relaxation of the (empirical) error constraint. For a fixed multiplier $\lambda\ge 0$, we select for each class $k$ the option minimizing
$
d_{k,j} + \lambda\, e_{k,j}
$.
This relaxation \emph{separates across classes}, so for any $\lambda$ we can compute a candidate solution by a linear scan through each class option list. Sweeping over a grid of $L$ multipliers produces $L$ candidate threshold vectors, and we retain the feasible solution with minimum deferrals. The runtime is
\[
\mathcal{O}(N\log N) \;+\; \mathcal{O}(L U),
\]
and memory is $\mathcal{O}(U)$ (to store option matrices). While not guaranteed to return the globally optimal solution of the original discrete constrained problem, the Lagrangian sweep is polynomial in the option-matrix size. The implementation outline is presented in Algorithm \ref{alg:multi_nonoracle_lagrangian_compact}. As with the DP variant, there is no change in the oracle and the non-oracle setting, apart from the fact that each algorithm takes as input the corresponding `per-class option matrices'.
\begin{algorithm}[t]
\caption{Multi-threshold selection (non-oracle, MC, Lagrangian sweep)}
\label{alg:multi_nonoracle_lagrangian_compact}
\begin{algorithmic}[1]
\STATE {\bfseries Input:} Option matrices $\{(d_{k,j},e_{k,j},\theta_{k,j})\}_{j=1}^{J_k}$ for each class $k$, sample size $N$, accuracy bound $\xi$, multipliers $\Lambda=\{\lambda_1,\dots,\lambda_L\}$
\STATE {\bfseries Output:} Optimal threshold vector $\boldsymbol{\theta}^*$ (best feasible among candidates)
\STATE Set error budget $b\gets 1-\xi$
\STATE $B\gets \lfloor bN\rfloor$
\STATE $\mathrm{BestDeferrals}$ $\gets +\infty$;\quad $\boldsymbol{\theta}^*$ $\gets [\mathbf{+\infty}]_{K\times1}$ 
\FOR{each $\lambda \in \Lambda$}
  \STATE $\mathrm{TotalDeferrals}$ $\gets 0$;\quad $\mathrm{TotalErrors}$ $\gets 0$
  \FOR{$k=1$ to $K$}
    \STATE $j_k \gets \arg\min_j\big(d_{k,j}+\lambda e_{k,j}\big)$
    \STATE $\mathrm{TotalDeferrals}$ $\gets$ TotalDeferrals $+ d_{k,j_k}$
    \STATE $\mathrm{TotalErrors}$ $\gets$ $\mathrm{TotalErrors}$ $+ e_{k,j_k}$
    \STATE $\theta_k \gets \theta_{k,j_k}$
  \ENDFOR
  \IF{$\mathrm{TotalErrors}$ $\le B$ and $\mathrm{TotalDeferrals}$ $<$ $\mathrm{BestDeferrals}$}
    \STATE $\mathrm{BestDeferrals}$ $\gets$ $\mathrm{TotalDeferrals}$;\quad $\boldsymbol{\theta}^*$ $\gets (\theta_1,\dots,\theta_K)$
  \ENDIF
\ENDFOR
\IF{$\boldsymbol{\theta}^* = [\mathbf{+\infty}]_{K\times1}$} \STATE $\boldsymbol{\theta}^*$ $\gets$ $\mathbf{1}_{K\times1}$ \ENDIF
\STATE \textbf{Return} $\boldsymbol{\theta}^*$
\end{algorithmic}
\end{algorithm}


\subsection{Alternative approximation methods}

In our optimization framework, the decision criterion and associated performance measure depend on expectations. For example, in the single-threshold policy (oracle setting), the expected cost and error are of the form
\begin{equation}
    \mathrm{Cost}_{\mathrm{ST-O}}(\theta) = c_s + c_m\int_{0}^{\theta} f(\beta)\,\mathrm{d}\beta, \quad
    \mathrm{Error}_{\mathrm{ST-O}}(\theta) = \int_{\theta}^{1} \varepsilon_s(\beta)\,f(\beta)\,\mathrm{d}\beta\,, \label{eq-app:error-integral}
\end{equation}
We now discuss practical alternatives to MC sampling for estimating the relevant components of the underlying joint distribution and for evaluating the integrals in \eqref{eq-app:error-integral}. 


\paragraph{Density Estimation: Marginal and Conditional Functions.}
An alternative approach is to explicitly estimate the components of the underlying distribution, and then compute the integrals based on those estimates. To this end, the marginal density $f(\beta)$ can be estimated nonparametrically using kernel density estimation (KDE) \citep{kde_est} or histogram binning. For example, a kernel density estimator with a smoothing bandwidth $h$ is given by
\begin{equation}
    \hat f(\beta) = \frac{1}{Nh}\sum_{i=1}^N K\!\left(\frac{\beta-\beta_i}{h}\right)\,,
\end{equation}
where $K$ is a nonnegative kernel function. Similarly, histogram binning partitions the domain $[0,1]$ into disjoint intervals $\{B_j\}$ and estimates the density in bin $B_j$ as the fraction of samples in that bin, normalized by bin width.

Because the optimization also depends on the conditional error rate, an estimate of this function is also required. A nonparametric estimate can be obtained by binning or smoothing regression. For binning, one first partitions the range of $\beta$ into bins $\{B_j\}$, and then computes
\begin{equation}
    \hat \varepsilon_s(B_j) = \frac{1}{n_j}\sum_{i:\beta_i\in B_j}\mathbf{1}\{\boldsymbol{a}_i^m\neq \boldsymbol{a}_i^s\},
\end{equation}
where $n_j=\#\{i:\beta_i\in B_j\}$. Alternatively, a smooth conditional function $\hat\varepsilon_s(\beta)$ can be obtained via e.g., kernel regression \citep{Bierens_1994},
\begin{equation}
    \hat\varepsilon_s(\beta)=\frac{\sum_{i=1}^N \mathbf{1}\{\boldsymbol{a}_i^m\neq \boldsymbol{a}_i^s\}\,K\big((\beta-\beta_i)/h\big)}
    {\sum_{i=1}^N K\big((\beta-\beta_i)/h\big)}\,.
\end{equation}
Given these estimates, one can approximate the cost and error integrals by numerical quadrature of products of $\hat f(\beta)$ and $\hat\varepsilon_s(\beta)$.

\paragraph{Parametric Modeling and Goodness‑of‑Fit Tests}
When the sample size is moderate to large and the marginal distribution exhibits a recognizable shape, fitting a parametric family to $f(\beta)$ can yield a parsimonious functional representation that simplifies subsequent computation. 
Parameters of the chosen family can be estimated by maximum likelihood, and the fit can be evaluated via goodness‑of‑fit tests such as the Kolmogorov–Smirnov \citep{ks_test} or Anderson–Darling \citep{Anderson01121954} tests. 
If the test results indicate a satisfactory fit, the fitted density $f(\beta;\hat\theta)$ may be used in place of $f(\beta)$ for computing integrals.

A similar strategy applies to the conditional error function $\varepsilon_s(\beta)$. Instead of nonparametric smoothing, one can posit a parametric regression model for the conditional error rate, such as a logistic regression or generalized linear model,
\begin{equation}
    \hat \varepsilon_s(\beta)=\sigma(\alpha+\gamma\beta),
\end{equation}
where $\sigma(\cdot)$ is the logistic function. Parameters are estimated from the data, and model adequacy is assessed through standard regression diagnostics. If the parametric form is appropriate, the fitted function $\hat\varepsilon_s(\beta)$ can be used in subsequent integration.

\paragraph{Closed‑Form vs.\ Numerical Integration}
If the underlying distribution and conditional functions admit simple analytic forms, the integrals \eqref{eq-app:error-integral} may have closed‑form expressions. For example, when $\beta\sim\mathrm{Uniform}(0,1)$ and $\varepsilon_s(\beta)=1-\beta$, the integrals reduce to elementary polynomials in $\theta$, yielding analytic expressions for $\mathrm{Cost}_{\mathrm{ST-O}}$ and $\mathrm{Error}_{\mathrm{ST-O}}$. In contrast, when $f(\beta)$ follows a more complex parametric form (e.g., Beta with arbitrary parameters) or when $\varepsilon_s(\beta)$ is non‑linear (e.g., logistic), closed‑form antiderivatives typically do not exist. In such cases, numerical integration techniques can be used to evaluate the integrals.


\section{Sufficiency of the confidence score}\label{app:sufficiency}
The confidence of a generated answer is a natural metric to rely on when designing decision rules in various settings. For example, in classification tasks, it is common to use the confidence of a model’s predicted class (which approximates the true posterior) in threshold-based rules \citep{chow1970}. In generative settings, where answers are more open-ended, recent works have attempted to measure the confidence of the entire generated answer using uncertainty quantification methods \citep{lm-polygraph, DBLP:conf/iclr/KuhnGF23, gupta_uncertainty, murong}. It is therefore rational to explore the decision sufficiency of the confidence score, i.e., whether it is a sufficient statistic for our deferral decision or extra information is needed. We provide clear definitions about sufficiency below.

Let $J$ be the extra information available at gating time (e.g., the input text embeddings, full logits, etc.). We again have two actions: $D=0$ (accept) and $D=1$ (defer). The true answer is denoted by $\boldsymbol{a}$. The loss for action $D$ is defined as $L(D, \boldsymbol{a}, J)$ so the Lagrangian is of the form \[ L(0, \boldsymbol{a}, J) = c_s + \lambda\mathds{1}[\boldsymbol{a} \neq \hat{\boldsymbol{a}}^s], \quad L(1, \boldsymbol{a}, J) = c_s + c_m + \lambda\mathds{1}[\boldsymbol{a} \neq \hat{\boldsymbol{a}}^m] \,.\]
The conditional risk of each action given $J=j$ is defined as:
\[r_0(j) = \mathbb{E}[L(0, \boldsymbol{a}, J) | J=j], \quad r_1(j) = \mathbb{E}[L(1, \boldsymbol{a}, J) | J=j] \,, \]
and the risk difference as $\Delta(j)=r_0(j)-r_1(j)$.

So, among all policies $\pi(J) \in \{0,1\}$, the optimal one is
\begin{equation}
    \pi^*(j) = \begin{cases}
        0, & \Delta(j) \leq 0\\
        1, & \Delta(j) > 0\,.
    \end{cases}
\end{equation}
This is already a “threshold” on a scalar function of $J$, $\Delta(J)$. 

We now introduce the scalar confidence score $\beta = \beta(J)\in\mathbb{R}$. 
\begin{definition}
    A scalar score is sufficient (relative to $J$) for the deferral problem if there exists a function $\phi: \mathbb{R} \to \mathbb{R}$ such that: \begin{equation}\Delta(j) = \phi(\beta(j)) \quad\forall j\,. \end{equation}
    Equivalently, for two $j_1$, $j_2$ with $\beta(j_1)=\beta(j_2)$: $\Delta(j_1)=\Delta(j_2)$. This means that if two queries have the same score $\beta$, then the expected advantage of using sLLM over mLLM (or vice-versa) is the same; the extra details in $J$ don’t change that trade-off.
\end{definition} 
For any policy $\pi$, we can write 
\begin{equation}
    R(\pi) = \mathbb{E}\left[  r_{\pi(J)}(J) \right] = \mathbb{E} \left[  r_1(J) + \mathds{1}[\pi(J) =0](r_0(J) - r_1(J))  \right]  \,.
\end{equation}
For any policy $\pi'(J)$, $R(\pi')\geq R(\pi^*)$. So $\pi^*$ is the globally optimal policy over all measurable policies $\pi(J)$. Then, assuming a sufficient confidence score $\beta(j)$, the optimal policy becomes:\begin{equation}
    \pi^*(j) = \begin{cases}
        0, & \phi(\beta(j)) \leq 0\\
        1, & \phi(\beta(j))  > 0\,.
    \end{cases}
\end{equation}

So, $\pi^*$ depends only on the `summary' $\beta$ of $J$. If we define the scalar policy \begin{equation}
    \delta^*(\beta) = \begin{cases}
        0, & \phi(\beta) \leq 0\\
        1, & \phi(\beta)  > 0\, ,
    \end{cases}
\end{equation}
then $\pi^*(j) = \delta^*(\beta(j))$. For any other policy $\pi'(J)$, its risk satisfies $R(\pi')\geq R(\pi^*) = R(\delta^*(\beta(J)))$. No policy that uses $J$ can beat the best policy that uses only the scalar score 
$\beta(J)$, provided that it is sufficient. If we now assume that $\Delta(J)$ (the excess error of sLLM over mLLM) is a non-increasing function of the confidence score, there exists a non-increasing $\phi(\cdot)$ such that $\Delta(J) = \phi(\beta(J))$. Then, as shown in Appendix~\ref{app:multi_general_policies}, the acceptance set $S = \{\beta: \phi(\beta) \leq 0\}$ is either empty (never accept), all of $\mathbb{R}$ (always accept), or a half-line $[\beta^*, \infty)$ (accept if $\beta \geq \beta^*$) with $\beta^* = \inf S$.\footnote{Usually, $\beta \in [0,1]$ so the intervals become $[0, 1]$ (always accept) and $[\beta^*, 1]$ (accept if $\beta \geq \beta^*$).}  

Theoretically, sufficiency is a property of the true joint distribution. In practice, we approximated it using the finite-sample datasets in the following way: we fit models that predicted the difference of correctness of the two LLMs, i.e., ``sLLM better than mLLM?" using \emph{(i) only $\beta$} and \emph{(ii) $\beta$ plus other components of $J$}. For these extra components, we considered either the whole input text embedded using Sentence Transformers \citep{DBLP:conf/emnlp/ReimersG19} or different statistics of the probability outputs of the sLLM (e.g., entropy, absolute difference etc.). For the predictive model, we used a Gradient Boosting Regressor \citep{friedman2000greedy}. Using a cross-fitting approach (Algorithm~\ref{alg:sufficiency_test}), we compared the mean squared error of the two models via a permutation test to assess whether including $J$ yields significant predictive improvement (more criteria about the sufficiency of confidence scores can be found in \citep{DBLP:conf/nips/JitkrittumGMNRK23}). In preliminary experiments that we conducted, the difference between the errors was not statistically significant, so we opted to use $\beta$ as our deferral signal. For example, Table~\ref{tab:sufficiency} shows results for two different pairs of models.

\begin{table}[t!]
  \caption{Results of initial sufficiency experiments. MSE$_{\beta}$ and MSE$_{\beta+J}$ are the cross-validated mean squared errors of the null and full models, respectively. $\Delta_{\mathrm{MSE}}$ = MSE$_{\beta}$ - MSE$_{\beta+J}$. Permutation $p$-value measures whether the observed $\Delta_{\mathrm{MSE}}$ is significantly larger than expected under the null hypothesis of no added predictive information from $J$ and is based on $T=10$ permutations of $J$ while keeping $\beta$ fixed.}
  \label{tab:sufficiency}
  \begin{center}
      \begin{adjustbox}{max width=\textwidth}
        \begin{tabular}{lcccccc}
          \toprule
          \bf Dataset & \bf sLLM / mLLM & \bf Extra information $J$ & \bf MSE$_{\beta}$ & \bf MSE$_{\beta+J}$      & \bf $\Delta_{\mathrm{MSE}}$  & \bf $p$-value \\
          \midrule
           \multirow{4}{*}{{\sc SST-2}} & \sc gemma-1b / gemma-4b & Input Text Embeddings & 0.1200$\pm$0.004 & 0.1009$\pm$0.019 & 0.0191 & 0.0574$\pm$0.0003 \\
           & \sc gemma-1b / gemma-4b & Logits Statistics & 0.1200$\pm$0.004 & 0.1208$\pm$0.004 & -0.0008 & 0.0909$\pm$0.0001 \\
           & \sc gemma-1b / gemma-12b & Input Text Embeddings & 0.1403$\pm$0.002 & 0.1037$\pm$0.002 & 0.0366 & 0.0853$\pm$0.0004 \\
           & \sc gemma-1b / gemma-12b & Logits Statistics & 0.1403$\pm$0.002 & 0.1415$\pm$0.001 & -0.0011 & 0.0702$\pm$0.0003 \\
           \midrule
           \multirow{4}{*}{{\sc emotion}} & \sc gemma-1b / gemma-4b & Input Text Embeddings & 0.2291$\pm$0.014 & 0.2415$\pm$0.018 & -0.015 & 0.0967$\pm$0.0004 \\
           & \sc gemma-1b / gemma-4b & Logits Statistics & 0.2291$\pm$0.014 & 0.2408$\pm$0.018 & -0.011 & 0.0724$\pm$0.0009 \\
           & \sc gemma-1b / gemma-12b & Input Text Embeddings & 0.2400$\pm$0.012 & 0.2534$\pm$0.009 & -0.013 & 0.0914$\pm$0.0007 \\
           & \sc gemma-1b / gemma-12b & Logits Statistics & 0.2400$\pm$0.012 & 0.2490$\pm$0.014 & -0.009 & 0.0815$\pm$0.001 \\
           \midrule
           \multirow{4}{*}{{\sc agnews}} & \sc gemma-1b / gemma-4b & Input Text Embeddings & 0.2044$\pm$0.004 & 0.1248$\pm$0.005 & 0.079 & 0.0978$\pm$0.002 \\
           & \sc gemma-1b / gemma-4b & Logits Statistics & 0.2044$\pm$0.004 & 0.1583$\pm$0.006 & 0.046 & 0.0589$\pm$0.0008 \\
           & \sc gemma-1b / gemma-12b & Input Text Embeddings & 0.2219$\pm$0.004 & 0.1191$\pm$0.005 & 0.1027 & 0.0814$\pm$0.0005 \\
           & \sc gemma-1b / gemma-12b & Logits Statistics & 0.2219$\pm$0.004 & 0.1616$\pm$0.011 & 0.0603 & 0.0678$\pm$0.0003 \\
           \midrule
           \multirow{4}{*}{{\sc fakenews}} & \sc gemma-1b / gemma-4b & Input Text Embeddings & 0.3181$\pm$0.015 & 0.2196$\pm$0.011 & 0.0984 & 0.0878$\pm$0.004 \\
           & \sc gemma-1b / gemma-4b & Logits Statistics & 0.3181$\pm$0.015 & 0.3521$\pm$0.014 & -0.034 & 0.0889$\pm$0.002 \\
           & \sc gemma-1b / gemma-12b & Input Text Embeddings & 0.4707$\pm$0.016 & 0.3202$\pm$0.011 & 0.1504 & 0.0793$\pm$0.002 \\
           & \sc gemma-1b / gemma-12b & Logits Statistics & 0.4707$\pm$0.016 & 0.5191$\pm$0.024 & -0.0483 & 0.0923$\pm$0.007 \\
           \midrule
           \multirow{4}{*}{{\sc SQuAD}} & \sc gemma-1b / gemma-4b & Input Text Embeddings & 0.3243$\pm$0.008 & 0.3620$\pm$0.0106 & -0.0377 & 0.0905$\pm$0.001 \\
           & \sc gemma-1b / gemma-4b & Logits Statistics & 0.3243$\pm$0.008 & 0.3334$\pm$0.009 & -0.009 & 0.0689$\pm$0.001 \\
           & \sc gemma-1b / gemma-12b & Input Text Embeddings & 0.2916$\pm$0.002 & 0.3205$\pm$0.005 & -0.0288 & 0.0993$\pm$0.0004 \\
           & \sc gemma-1b / gemma-12b & Logits Statistics & 0.2916$\pm$0.002 & 0.2987$\pm$0.003 & -0.007 & 0.0721$\pm$0.0006 \\
           \midrule
           \multirow{4}{*}{{\sc WMT}} & \sc gemma-1b / gemma-4b & Input Text Embeddings & 0.1941$\pm$0.003 & 0.2060$\pm$0.002 & -0.0119 & 0.0675$\pm$0.0005 \\
           & \sc gemma-1b / gemma-4b & Logits Statistics & 0.1941$\pm$0.003 & 0.1999$\pm$0.003 & -0.0239 & 0.0734$\pm$0.0007 \\
           & \sc gemma-1b / gemma-12b & Input Text Embeddings & 0.2012$\pm$0.003 & 0.2136$\pm$0.008 & -0.0123 & 0.0789$\pm$0.0012 \\
           & \sc gemma-1b / gemma-12b & Logits Statistics & 0.2012$\pm$0.003 & 0.2082$\pm$0.004 & -0.0069 & 0.0821$\pm$0.0004 \\
          \bottomrule
        \end{tabular}
        \end{adjustbox}
  \end{center}
  \vskip -0.1in
\end{table}

\begin{algorithm}[t]
  \caption{Score Sufficiency Testing}
  \label{alg:sufficiency_test}
  \begin{algorithmic}[1]
    \STATE {\bfseries Input:} Dataset $\{(J_i, \boldsymbol{a}_i, \hat{\boldsymbol{a}}^{s}_i, \hat{\boldsymbol{a}}^{m}_i, \beta_i)\}_{i=1}^N$
    \FOR{$i=1$ {\bfseries to} $N$}
    \STATE {\bfseries Output:} Statistical test of whether score $s$ is sufficient for predicting correctness difference $\Delta$
    \STATE Compute binary error indicators:
    $E^{s}_i \gets \mathds{1}\{\hat{\boldsymbol{a}}^{s}_i\neq \boldsymbol{a}_i\},\quad E^{m}_i \gets \mathds{1}\{\hat{\boldsymbol{a}}^{m}_i\neq \boldsymbol{a}_i\} $
    \STATE Define signed error difference: $\Delta_i \gets E^{s}_i - E^{m}_i$
    \ENDFOR
    \STATE Partition data into $K$ folds
    \FOR{each fold $k$}
        \STATE Train `null' learner $\hat{f}_{\mathrm{null}}$ on $k-1$ training folds to predict $\Delta$ from $s$
        \STATE Train `full' learner $\hat{f}_{\mathrm{full}}$ on $k-1$ training folds to predict $\Delta$ from $(s,J)$
        \STATE Evaluate predictions on held-out fold:
        $\hat{\Delta}_{\mathrm{null}}^{k},\ \hat{\Delta}_{\mathrm{full}}^{k}$
    \ENDFOR
    \STATE Aggregate mean squared errors:
    $
    \text{MSE}_{\mathrm{null}}=\frac{1}{N}\sum_i(\Delta_i-\hat{\Delta}_{{\mathrm{null}},i})^2,
    \quad
    \text{MSE}_{\mathrm{full}}=\frac{1}{N}\sum_i(\Delta_i-\hat{\Delta}_{{\mathrm{full}},i})^2
    $
    \STATE Compute observed test statistic:
    $
    T_{\mathrm{obs}}=\text{MSE}_{\mathrm{null}}-\text{MSE}_{\mathrm{full}}
    $
    \STATE Obtain permutation distribution:
    \FOR{$p=1\ldots n_{\mathrm{perm}}$}
        \STATE Permute feature set $J_i^{p}\gets$ shuffle of $J$
        \STATE Repeat train/predict to obtain $\text{MSE}_{\mathrm{full}}^{p}$
        \STATE $T^{p}\gets \text{MSE}_{\mathrm{null}}-\text{MSE}_{\mathrm{full}}^{p}$
    \ENDFOR
    \STATE Approximate p-value:
    \[
    \frac{1+\sum_{p=1}^{n_{\mathrm{perm}}}\{T^{p}\geq T_{\text{obs}}\}}{n_{\mathrm{perm}}+1}
    \]
    \IF{p-value $<0.05$}
        \STATE Reject $H_0$: $J$ adds predictive value beyond $\beta$
    \ELSE
        \STATE Fail to reject $H_0$
    \ENDIF
  \end{algorithmic}
\end{algorithm}

\end{document}